\documentclass[9pt,twocolumn]{article}
\everydisplay{\small}

\usepackage[T1]{fontenc}
\usepackage[utf8]{inputenc}
\usepackage{authblk}
\usepackage{arxiv}

\newcommand{\bs}[1]{\boldsymbol{#1}}

\newcommand{\dpar}[2]{\frac{\partial #1}{\partial #2}}

\newcommand{\myappref}[1]{Appendix~\ref{#1}}
\newcommand{\myeqref}[1]{Eq.~\eqref{#1}}
\newcommand{\myfigref}[1]{Fig.~\ref{#1}}

\newcommand{\mytabref}[1]{Table~\ref{#1}}

\newcommand{\mypropref}[1]{Proposition~\ref{#1}}
\newcommand{\mylemmaref}[1]{Lemma~\ref{#1}}

\usepackage{url}
\usepackage{pgfplots}
\usepackage{algorithm,algpseudocode}
\usepackage{algorithmicx}
\usepackage{amsmath}
\usepackage{amsfonts}
\usepackage{gensymb}
\usepackage{siunitx}
\usepackage[colorlinks,bookmarksnumbered,bookmarksopen,citecolor=blue,urlcolor=blue,linkcolor=blue,]{hyperref}
\usepackage{amsthm}
\usepackage{booktabs}

\usepackage{enumitem}
\setlist[itemize]{leftmargin=2em}
\setlist[enumerate]{leftmargin=2em}

\newtheorem{theorem}{Theorem}[section]
\newtheorem{proposition}[theorem]{Proposition}
\newtheorem{lemma}[theorem]{Lemma}
\newtheorem*{remark}{Remark} 

\DeclareMathOperator{\tr}{tr}

\usepackage{mathtools}
\mathtoolsset{showonlyrefs=true}
\usepackage{autonum}

\allowdisplaybreaks

\usepackage{filecontents}

\definecolor{tabblue}{HTML}{1f77b4}
\definecolor{taborange}{HTML}{ff7f0e}
\definecolor{tabgreen}{HTML}{2ca02c}
\definecolor{tabred}{HTML}{d62728}
\definecolor{tabpurple}{HTML}{9467bd}

\usepackage[scaled]{helvet}

\newcommand{\commentquercus}[1]{\textcolor{black}{#1}}
\newcommand{\commentnat}[1]{\textcolor{black}{#1}}

\newcommand{\PEA}[1]{\textcolor{black}{#1}}

\title{Data-driven particle dynamics: Structure-preserving coarse-graining for emergent behavior in non-equilibrium systems}

\author[1]{Quercus Hern\'andez}
\author[1]{Max Win}
\author[2]{Thomas O'Connor}
\author[1]{Paulo E. Arratia}
\author[1]{Nathaniel Trask}
\affil[1]{{\small School of Engineering and Applied Science, University of Pennsylvania, 220 South 33rd Street, Philadelphia,
PA 19104, USA.}}
\affil[2]{{\small Department of Materials Science and Engineering, Carnegie Mellon University, Pittsburgh, PA, USA.}}

\begin{document}

\twocolumn[
\maketitle

\begin{abstract}
Multiscale systems are ubiquitous in science and technology, but are notoriously challenging to simulate as short spatiotemporal scales must be appropriately linked to emergent bulk physics. When expensive high-dimensional dynamical systems are coarse-grained into low-dimensional models, the entropic loss of information leads to emergent physics which are dissipative, history-dependent, and stochastic. To machine learn coarse-grained dynamics from time-series observations of particle trajectories, we propose a framework using the metriplectic bracket formalism that preserves these properties by construction; most notably, the framework guarantees discrete notions of the first and second laws of thermodynamics, conservation of momentum, and a discrete fluctuation-dissipation balance crucial for capturing non-equilibrium statistics. We introduce the mathematical framework abstractly before specializing to a particle discretization. As labels are generally unavailable for entropic state variables, we introduce a novel self-supervised learning strategy to identify emergent structural variables. We validate the method on benchmark systems and demonstrate its utility on two challenging examples: (1) coarse-graining star polymers at challenging levels of coarse-graining while preserving non-equilibrium statistics, and (2) learning models from high-speed video of colloidal suspensions that capture coupling between local rearrangement events and emergent stochastic dynamics. We provide open-source implementations in both PyTorch and LAMMPS, enabling large-scale inference and extensibility to diverse particle-based systems.
\end{abstract}
\hfill \break]

\section{Introduction}

Multiscale \commentnat{phenomena governing complex systems, from quantum materials to geophysical flows, exhibit emergent behaviors in which small-scale, potentially stochastic physics influences widely separated spatiotemporal scales. For example, quantum effects govern ferromagnetic ordering in magnetic materials \cite{lim2018improper}, while topological entanglement leads to anomalous diffusion in polymer melts \cite{o2020topological}. In geophysics, the bulk evolution of soil landscapes can be tied to microstructural events such as jamming transitions in dense suspensions \cite{jerolmack2019viewing,kostynick2022rheology}. Bridging these disparate scales is an outstanding challenge for multiscale simulation, as computationally tractable strategies must avoid resolving prohibitively short spatiotemporal scales while preserving their influence on resolved ones. This process is typically referred to as ``coarse-graining" and spans multiple applications. For example, in molecular physics, the resolution of electron motion by density functional theory (DFT) or vibrational motion by molecular dynamics (MD)  would mandate $10^{12}-10^{15}$ timestep simulations to resolve $O(1) s$  timescales, and many works aim to identify coarse-grained representations, either empirically by ansatz \cite{groot1997dissipative}, through formal procedures \cite{noid2008multiscale,izvekov2005multiscale}, or with machine learning (ML) to supervise pair potentials from DFT data \cite{behler2007generalized}. In turbulence, full resolution of the energy cascade requires prohibitive resolution of the Kolmogorov scale; in practice, coarse-graining to either the Taylor microscale (in large eddy simulation/LES) or integral length scale (in unsteady Reynolds averaged simulation/URANS) provides tractable simulation; turbulence literature similarly reflects a trend from empirical ansatz \cite{smagorinsky1963general}, to formal coarse-graining procedures \cite{germano1991dynamic,parish2017jcp,li2014construction}, to machine-learned LES/RANS models \cite{ling2016reynolds,duraisamy2021perspectives,kochkov2021machine}. Our aim is to build a generally applicable framework that can machine learn a stochastic dynamical system governing large scale, easily measurable, dynamics while preserving the essential statistical character of unresolved physics.}

\begin{figure*}[t!]
\centerline{\includegraphics[width=\linewidth]{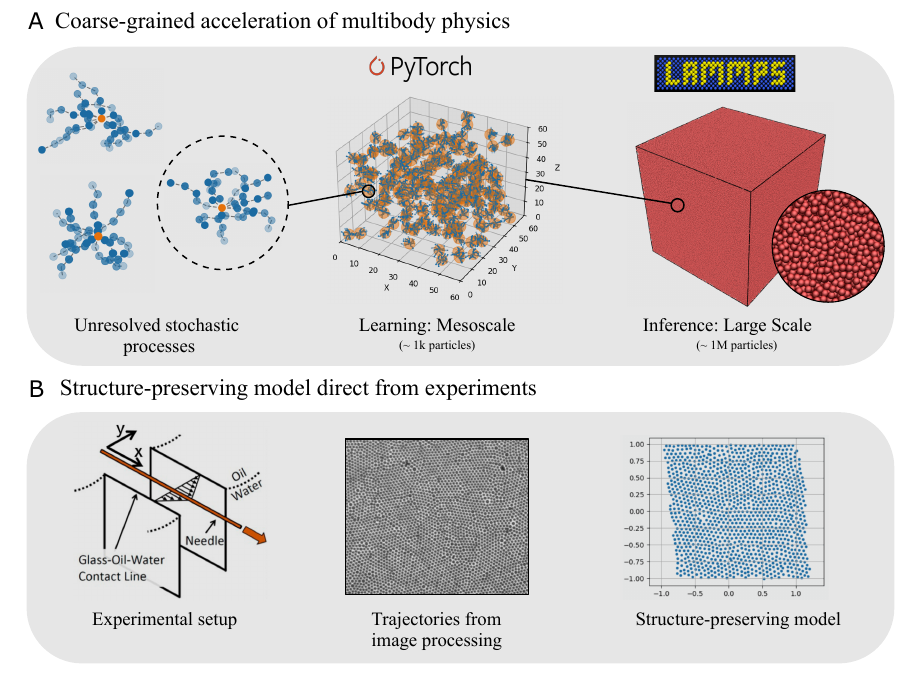}}
\caption{Exemplar applications using data-driven particle dynamics to bridge scales in simulated (top) and experimental (bottom) datasets. (A) Detailed simulation of polymers fully resolve stochastic fluctuations (left), in-silico experiments of a small domain supervise model discovery (middle), yielding a data-driven model for non-equilibrium bulk response solved in a massive parallel simulator (right). (B) A 2D suspension in an oscillatory shear flow provides a rich multiphysics/multiscale exemplar incorporating lubrication, electrostatics, contact, and stochasticity (left). Trajectories of particle data are extracted using computer vision techniques, providing complete top-down (but noisy) data to supervise model collection (center), ultimately yielding a model which predicts linkages between structure and emergent dynamics (right). \label{fig:overview}}
\end{figure*}

\commentnat{We address the challenge by building a unified ``top-down" coarse-graining framework which can machine learn dynamics from partial observations of coarse-grained variables without labels for unobservable states, using bracket structure to build thermodynamically consistent representations of unobservable fluctuations. The current machine learning paradigm is largely ``bottom-up", where data and simulations must be collected at the most fundamental level to understand emergent dynamics \cite{jin2022bottom}. This ab initio strategy naturally yields a potentially unwieldy hierarchy of models, linking descriptions of physics that bridge submolecular, molecular, mesoscopic, continuum, and systems/network scales. This is particularly relevant for experimental systems in which velocimetry or digital image correlation may yield observable fields but unobserved fields are either intractable or impossible to measure (e.g. material structure or entropy, respectively); extracting predictive models direct from experiment is the primary goal of this work. Following the trend of self-supervised learning which has recently emerged in the machine learning community, we demonstrate how geometric structure in the unresolved physics at lower levels can be exploited to infer their effect on higher levels of the hierarchy without the need for labeled data. This enables one to e.g. learn a continuum scale model without DFT/MD labels, subverting the prevailing wisdom that top-down approaches cannot capture unresolved statistical features \cite{jin2022bottom}. }

\commentnat{For a machine learned model to credibly describe coarse-grained representations of systems, emergent thermodynamic and stochastic structure must be preserved at a discrete level, independent of data availability, optimization error, or model form error. Several frameworks offer candidate mathematical formalisms to identify this requisite structure. The Mori-Zwanzig (MZ) formalism \cite{zwanzig1960ensemble,mori1965transport,grabert2006projection,zwanzig2001nonequilibrium} interprets coarse-graining as a projection of full-order dynamics onto coarse-grained variables. Through projections, coarse-graining induces three terms: a mean-field coarsened dynamics; a non-Markovian/history-dependent dissipative memory kernel; and a stochastic term accounting for work done by unresolved physics. Although the full-order system may be reversible, MZ establishes the need to consistently treat non-Markovian stochasticity and dissipation; the work done on the system by fluctuations must exactly balance dissipation to ensure energy is conserved, providing so-called \textit{detailed fluctuation-dissipation balance} \cite{callen1951irreversibility,kubo1966fluctuation}. In non-equilibrium systems, this provides a formalism to model the transfer of energy from unresolved non-equilibrium fluctuations and resolved state variables via dissipative mechanisms. Through fluctuation-dissipation balance, we aim to infer the statistical character of fluctuations through labels of observable dissipative processes.}

\commentnat{While powerful, the MZ formalism often requires restrictive Markovian assumptions to obtain tractable models \cite{chorin2000optimal}. The metriplectic/GENERIC formalism \cite{morrison1984bracket,morris1997modeling,grmela1997dynamics,ottinger1997dynamics,morrison2024inclusive} provides an algebraic alternative where processes are described via reversible and irreversible brackets whose matrix representations admit natural parameterizations in tensor-based ML frameworks like PyTorch. A stochastic metriplectic model takes the form:}

\begin{equation}\label{eq:GENERIC_stochastic}
d\bs{x} = \left[\bs{L}\dpar{E}{\bs{x}} + \bs{M}\dpar{S}{\bs{x}} + k_B\dpar{}{\bs{x}}\cdot\bs{M}\right]dt + d\tilde{\bs{x}}
\end{equation}
\commentnat{for: state $\bs{x}$, energy and entropy functionals $E(\bs{x})$ and $S(\bs{x})$, Poisson matrix $\bs{L}(\bs{x})=-\bs{L}^T$, friction matrix $\bs{M}(\bs{x})=\bs{M}^T$, Boltzmann constant $k_B$ and fluctuation $d\tilde{\bs{x}}$. Under the degeneracy constraint $\bs{L}\bs{\nabla}S=\bs{M}\bs{\nabla}E=\bs{0}$, one achieves constant energy and non-decreasing entropy in the $k_B = 0$ case. If fluctuations are posed which satisfy}
\begin{equation}
\mathbb{E}\left[d\tilde{\bs{x}}d\tilde{\bs{x}}^T\right] = 2k_B\bs{M}dt,
\end{equation}
\commentnat{detailed fluctuation-dissipation balance is achieved so that energy is preserved despite stochastic forcing. In the conventional literature, an analytically derived model is identified as metriplectic and inherits theoretical benefits. The current work instead poses a parameterized metriplectic bracket so that, \textit{by construction}, any fit model preserves detailed balance and preserves statistical thermodynamical structure. This provides a linkage between observable dissipative processes and unobserved statistical structure.}

\commentnat{In this manuscript, we prescribe a general construction of parametric $E,S, \bs{L}, \bs{M}$ as perturbations of non-canonical Hamiltonian systems which serves as a broadly applicable template, before specializing to a particle Hamiltonian appropriate for learning material/Lagrangian descriptions suitable for experimental particle tracking, data-driven MD, and large deformation solid mechanics. The particle bracket is adapted from classical particle methods \cite{monaghan1992smoothed,espanol2003smoothed} and provides interpretable predictions of volume, internal energy, pressure, and temperature via an equation of state (EOS) which discretely preserves Gibbs' relations. Our approach provides a new capability to recover a model from time-series observations of particle trajectories which extrapolates beyond equilibrium training data with the capacity to recover a conventional EOS. From a computational perspective, we provide an open-source LAMMPS implementation that exhibits weak scaling, with the potential to evaluate these models at the exascale. }

\commentnat{This development enables the design of new experimental campaigns in which predictive models are rigorously extracted from particle-tracking data. Such a capability is particularly promising for systems which defy conventional first-principles modeling. Here, we exemplify this approach using an experimental dense colloidal suspension composed of nearly $5\times10^4$ particles undergoing shear deformation (see Fig. \ref{fig:overview}B).  For such systems, predictive models have proven elusive \cite{galloway2022relationships,galloway2020scaling, charbonneau2017glass, manning2011vibrational}, with limited progress using supervised learning to identify configurational features that correlate with emergent suspension dynamics \cite{cubuk2015identifying,cubuk2015identifying,richard2020predicting}. This challenging but tractable modeling task could be done for a few particles, but cannot scale up to millions of particles with conventional simulation. We demonstrate with our top-down approach a dynamical model that predicts deviations from linear shear profiles associated with rare dislocation events driven by force dipoles, establishing a nontrivial linkage between structure and emergent dynamics. }

\section{Prior literature} 

\subsection{Candidate coarse-graining theories}

\commentnat{A number of theoretical frameworks provide a principled foundation for coarse-graining. Mori and Zwanzig's initial works \cite{mori1965transport,zwanzig1960ensemble} were generalized to optimal prediction extensions \cite{chorin2000optimal} and the metriplectic/GENERIC \cite{morrison1984bracket,ottinger1997dynamics} theories; Ottinger showed \cite{ottinger1998general} that GENERIC can be derived as a special case of MZ.  The renormalization group (RNG) theory \cite{yakhot1992rg}, as well as information theoretic frameworks like relative entropy, iterative Boltzmann,  and MaxEnt \cite{shell2016coarse,mashayak2015relative,chaimovich2011coarse, jin2022bottom}, and more recent stochastic  \cite{esposito2012stochastic} provide alternative formalisms to coarse-grain. In the mechanics community, homogenization theory and the more finite element-centric variational multiscale (VMS) method offer related projection-based variational techniques but lack an explicit tie to non-equilibrium thermodynamical/statistical behavior \cite{hughes1998vms,hughes2004multiscale}. }

\commentnat{In fluids, these frameworks have provided a theory-informed basis to develop turbulence and constitutive models, including representative earlier works in RNG \cite{yakhot1986renormalization,smith1998rgreview} and more recent MZ \cite{parish2017jcp,parish2017non,li2014construction,stinis2007higher} and VMS \cite{john2005finite,hughes2002variational,pradhan2020variational} works. In molecular dynamics, coarse-graining is arguably more mature, with excellent review articles \cite{jin2022bottom, schilling2022coarse} elucidating the bottom-up/top-down divide, summarizing recent trends, and providing MD-specific ties to information theory. While some coarse-grained MD methods conjecture effective potentials in an ad hoc manner \cite{groot1997dissipative}, several works have used MZ or variational/information-theoretic frameworks to provide rigorous justification of their functional form \cite{li2015incorporation,izvekov2021mori,noid2008multiscale}.}

\commentnat{The present work requires a theory able to provide a machine-learnable ansatz for dynamics which can: prescribe the reversible, irreversible, and stochastic dynamics that emerge in coarse-grained descriptions; guarantee detailed fluctuation-dissipation balance; provide discrete notions of the fist and second laws of thermodynamics; support a tractable computation amenable to end-to-end automatic differentiation; and identify internal variables describing unobservable physics. Interestingly, \cite{jin2022bottom,dama2013theory} advocate ``ultra-coarse-grained" state dynamics, where the state is enriched with an additional field which encodes unresolved structure; the self-supervised entropic variable identified in the present framework plays an analogous role.}

\subsection{Supervised learning: force matching, missing physics and closure modeling}

\commentnat{In data-driven modeling, the inference of parameterized dynamics $\dot{X} = F(X,t;\theta)$ from either timeseries $\left\{X_i(\cdot)\right\}$ or labeled state-force pairs $\left\{X(t_i),F(t_i)\right\}$ is sometimes referred to as \textit{force matching}. When a force is partially known but missing terms due to a coarse-graining procedure, different communities refer to this as \textit{closure modeling}, \textit{missing physics}, or \textit{residual correction}. A typical configuration is to discretize in time and infer parameters $\theta$ by minimizing the residual associated with a given timestep (e.g. $\min_\theta || X^{n+1} - X^n - F(X^n,t^n;\theta)||^2$), to integrate the parameterized ODE and train over rolled-out trajectories, or to employ equality constrained optimization. For all configurations, it is assumed that labels for a fully observable state $X$ are available. }

\commentnat{In the MD community, while conventional approximants \cite{ercolessi1994interatomic,izvekov2005multiscale} or atomic cluster expansions \cite{drautz2019atomic} have provided reliable bases for regressing $F$, neural architectures have surged in popularity due to their ability to extract information from high-dimensional inputs. While early works used neural networks \cite{behler2007generalized} or Gaussian processes \cite{bartok2010gaussian} to regress univariate force-distance relationships, contemporary graph architectures can process many-body geometric configurations of either particles or continuum meshes \cite{gilmer2017neural,sriram2022towards,zhang2018deep,thomas2018tensor,weiler20183d,kondor2018clebsch,brandstetter2021geometric,schutt2017schnet,gasteiger2020directional,liu2022spherical}. Autoencoders can identify latent representations from bottleneck structure, providing candidate representations for reduced dynamics \cite{brunton2021modern,wang2019coarse,wang2019learning,ilnicka2023designing}. In most instances, training data takes the form of ab initio DFT calculations that give a complete prescription of $F$ from molecular configuration in the bottom-up configuration.}

\commentnat{Recent work has shown that low force-matching residuals are insufficient benchmarks for coarse-grained models, while structural descriptors (e.g., radial distribution functions) and dynamical quantities (e.g., velocity autocorrelation functions, diffusion coefficients) are necessary benchmarks to qualify thermodynamic and stability performance \cite{fu2022forces}. Without physically consistent coupling between dynamics and fluctuations, many ML methods rely on \textit{backmapping} \cite{jin2022bottom}, employing generative architectures trained on ab initio data to reproduce stochastic behavior in both continuum \cite{li2024synthetic, du2024conditional, sardar2024spectrally} and molecular \cite{gebauer2019symmetry,noe2019boltzmann,xu2022geodiff} systems. Here, we demonstrate that metriplectic brackets prescribe thermodynamically valid fluctuations without ab initio training data, whereas popular physics-agnostic GNNs fail to capture correct statistical benchmarks. Jin et al \cite{jin2022bottom} highlight a particular challenge in bottom-up coarse graining: the choice of coarse-grained variables and the identification of a model-form compatible with target ensemble statistics require careful constructions, often prompting overly simplistic Langevin descriptions of missing friction. Our approach provides a rich class of automatically compatible stochastic/dissipative processes.}

\subsection{Structure-preserving algorithms}
\commentnat{While the incorporation of physical principles has broadly been recognized to improve scientific machine learning, recent works have distinguished between \textit{physics-informed} techniques \cite{lagaris1998artificial,karniadakis2021physics} which impose physics by introducing penalties, compared to \textit{structure-preserving} which integrate physics either by construction or via equality constraints \cite{sanderse2024scientific}. Among other pathologies, the former loosely endows properties only to within optimization error, and therefore cannot provide detailed fluctuation dissipation balance \cite{wang2021understanding,wang2023expert}. The latter has particularly been considered in computer vision and molecular dynamics, where substantial effort focuses on equivariant/invariant architectures which provide uniform predictions as input images/graphs are rotated or translated \cite{liao2022equiformer,liao2023equiformerv2,klein2023equivariant,passaro2023reducing,wang2024enhancing,fu2022forces,vignac2023midi}. This is distinct from the notion of geometry in variational/geometric mechanics, where symmetries of governing brackets, rather than geometric transformations, guarantee notions of conservation or dissipation. For this work, we consider the latter when we define a data-driven coarse-graining approach to be structure preserving. In Fu's assessment of ML-MD benchmarks \cite{fu2022forces}, the former is shown to meet more stringent benchmarks but with millions of parameters; we achieve the same performance but with tens of thousands of parameters (a $10-1000\times$ improvement in parameter efficiency, and therefore inference time).}

\commentnat{Enforcing these symmetries are commonplace for reversible systems, where Hamiltonian networks \cite{greydanus2019hamiltonian,david2023symplectic,chen2019symplectic}, Lagrangian networks \cite{cranmer2020lagrangian}, and similar constructions yield architectures which conserve energy by design. Irreversible systems are only recently considered in the scientific ML literature; while early works considered physics-informed soft constraints \cite{hernandez2021structure,hernandez2022thermodynamics}, strong imposition of dissiptive structure has been considered only more recently with variational or port-Hamiltonian descriptions \cite{zhong2020dissipative,sosanya2022dissipative,xiao2024generalized,huang2022variational,yu2021onsagernet, desai2021port}. Some consider metriplectic brackets \cite{lee2021machine,zhang2022gfinns,gruber2023reversible}, but with a construction that scales superlinearly in the degrees of freedom so that structure-preserving dynamics have been shown only for thousands of degrees of freedom. The present work prescribes metriplectic brackets in linear $O(N)$ complexity, allowing us to demonstrate weak parallel scaling at inference time in LAMMPS \cite{plimpton1995fast}. Open source code is provided  (\url{https://github.com/PIMILab}) for both learning models in PyTorch and performing inference at massive scales in LAMMPS.  The current work thus amounts to the first demonstration of learned metriplectic dynamics with O(N) scaling enabling massively parallel inference.}

\section{Results}
 \commentnat{We consider a diverse set of physics to highlight important consequences of structure preservation in our approach:}
 \begin{itemize}
 \item \textbf{Ideal gas.} Recovery of a continuum fluid system with a conventional equation of state allows us to demonstrate out-of-distribution performance for non-equilibrium flow predicted from equilibrium training data and viceversa.
 \item \textbf{Star polymer.} A conventional coarse-grained MD demonstrates performance superior to a well-calibrated empirical coarse-grained model and recovery of structural/dynamic statistics where conventional GNNs fail.
 \item \textbf{Viscoelastic solid.} The generality of the framework is highlighted by predicting a solid system, proving generalization to other energy functionals (i.e. strain energy) to treat continuum systems.
 \item \textbf{Jammed colloidal system.} Illustration of full capability of framework to bridge structure-dynamics to predict rare events in a complex multiphysics system from experimental data; lubrication/contact/electrostatic interactions governing defect dynamics are unresolved, yet learned stochastic dynamics capture their impact on statistics of defect dynamics.
 \end{itemize}

\commentnat{In light of \cite{fu2022forces}, which stresses the importance of structural/dynamical statistics benchmarks, we include comparisons to popular GNN architectures when predicting the Radial Distribution Function (RDF), Velocity Autocorrelation Function (VACF), and Mean Squared Displacement (MSD).}

\begin{itemize}
    \item \textbf{Graph Network-based Simulator (GNS).} We take Graph Network-based Simulator  \cite{sanchez2020learning} as a representative state of the art ``black-box" GNN model commonly used for scientific machine learning tasks.
    \item \textbf{Stochastic Graph Network-based Simulator (GNS-SDE).} We add an additional message passing network which imposes multiplicative Gaussian noise and uses a maximum likelihood loss  \cite{dridi2021learning}  to ensure improved performance is not due to maximum likelihood training.
    \item \textbf{Conventional coarse-grained model (DPD).} We use the same maximum likelihood training to calibrate the parameters of a standard DPD model \cite{groot1997dissipative}, providing a respresentative baseline for accuracy accessible to a conventional ad hoc coarse graining method. While accurate for an ideal gas, for the more complicated systems the calibrated DPD provides an overly "soft" potential, with lower amplitude peaks in RDF compared to our learned models.
\end{itemize}
Details for all benchmarks and comparison methods may be found in \myappref{app:hyperparams}.

\subsection{Ideal gas}

We first consider an ideal gas in a 3D periodic box over three different forcings to demonstrate out-of-distribution generalization, sketched in \myfigref{fig:ideal_gas}. Training is performed for vortex forcing (\textsc{Taylor-Green}, \cite{taylor1937mechanism}), while inference is tested under configurations unseen during training for an unforced system (\textsc{Self-diffusion}) and shear forcing via Lees-Edwards boundary conditions (\textsc{Shear flow}, \cite{lees1972computer}).

\begin{figure}[h]
\centerline{\includegraphics[width=\linewidth]{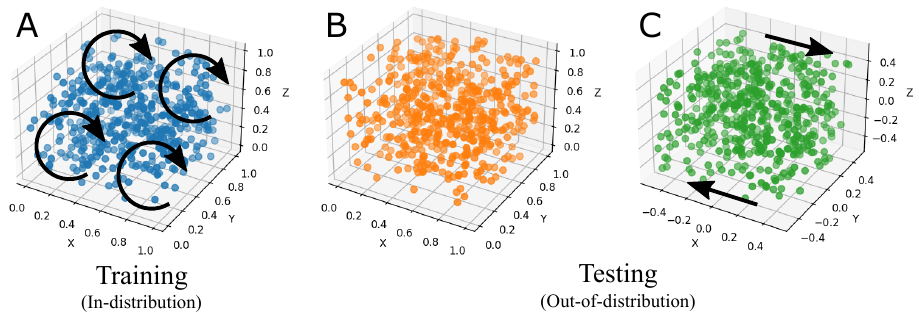}}
\caption{\textsc{Ideal Gas} datasets. Flow configurations for (A) \textsc{Taylor-Green} training dataset, and (B) \textsc{Self Diffusion} and (C) \textsc{Shear Flow} testing datasets. We confirm generalization of non-equilibrium statistics in (B) and (C) despite distinct flow conditions held out during training.\label{fig:ideal_gas}}
\end{figure}

\myfigref{fig:taylor_green_app} demonstrates good agreement between the proposed methodology and spatial statistics for long extrapolation rollouts. On the other hand, both GNS and GNS-SDE fail to capture the statistics of the flow, and are outperformed by classical DPD technique. We also present in \mytabref{tab:results_app} the relative error results of the same model tested under two different periodic boundary conditions without any additional training. We can observe that our method maintains its performance under extrapolation to unseen flows while the traditional GNN methods diverge in longer rollouts.

We further include two additional tests to provide visual intuition of our algorithm's performance in inference; these cases have not been held out during training.   \myfigref{fig:shear_flow_app} shows how the proposed method achieves the correct average velocity shear profile with respect to the reference solution for the \textsc{Shear Flow} dataset. \myfigref{fig:shear_flow_app} shows the inferred \textsc{Self-Diffusion} mixing of a domain scaled up to $L=2$. 

\commentquercus{To demonstrate the capability to predict non-equilibrium dynamics when trained under equilibrium conditions, we also trained the model under quiescent equilibrium conditions and evaluated its rollout performance under externally driven flows. As shown in \myfigref{fig:ideal_gas_app} the equilibrium-trained model reproduces the key dynamical statistics, confirming that the metriplectic structure generalizes beyond its training regime. However, its performance slightly deteriorates when trained on both equilibrium and non-equilibrium data due to the model not having encountered the broader velocity distributions present in non-equilibrium states.}

\subsection{Challenging coarse graining experiment: Star polymer}

A star polymer melt in a 3D periodic box provides a conventional coarse-grained MD exemplar, for which bottom-up ML is natural, but the feasibility to capture statistics in a ``top-down" configuration is unclear. Each polymer is composed of a core atom and 10 arms bonded by a \commentquercus{finite extensible nonlinear elastic (FENE)} potential, whereas the interaction between polymers is modeled by a Lennard-Jones potential. Two different internal configurations are considered: 1 and 5 beads per arm, which results in less and higher coarse-graining levels (\myfigref{fig:star_polymer}). The fully resolved datasets are generated using LAMMPS \cite{plimpton1995fast} and then coarse-grained by computing the center of mass position and velocity of each polymer. 

\begin{figure}[h]
\centerline{\includegraphics[width=\linewidth]{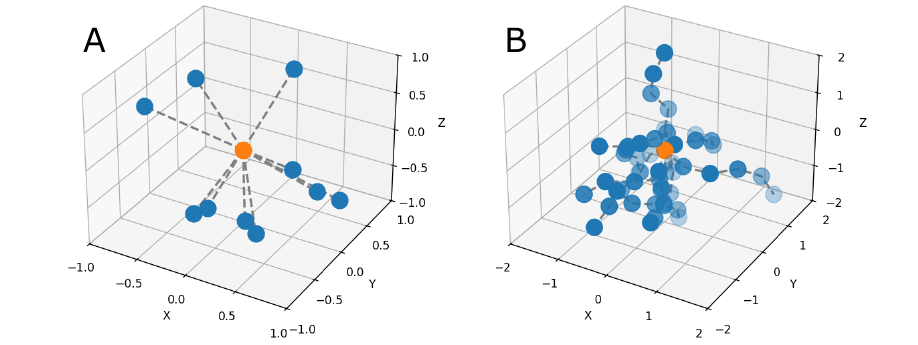}}
\caption{Representative star polymers coarse grained into a single data-driven particle in \textsc{Star Polymer} datasets. (A) \textsc{Star Polymer 11} with 1 core, 10 arms and 1 bead per arm. (B) \textsc{Star Polymer 51} with 1 core, 10 arms and 5 beads per arm.\label{fig:star_polymer}}
\end{figure}

The results are shown in \myfigref{fig:star_polymers} and \mytabref{tab:results_app}. We observe that, similar to the previous dataset, the correct statistics are obtained while GNN baselines fail to predict statistical structure. \commentquercus{Both GNS and GNS-SDE predict a nearly flat RDF, meaning that the prediction lacks spatial structure and correlations. Both the VACF and MSD deviate from the ground truth by orders of magnitude, which suggests that they overestimate the velocity and dissipation. Similarly, the deviation in the MSD shows an overestimation of the diffusion regime, totally misrepresenting the stochastic dynamics of the system. This highlights the major qualitative impact physical structure preservation has on even simple systems, and that ``black-box" techniques are poor candidates for coarse-graining, as the well-calibrated DPD model outperforms these methods with drastically fewer parameters. In contrast, our method outperforms the well-calibrated model by a factor of $2\times-10\times$ in relative error.}

\begin{figure*}
\centering
\begin{tikzpicture}
\pgfplotsset{width=\textwidth, height=8cm}

\begin{axis}[
    width=0.33\textwidth,
    height=4cm,
    ylabel=VACF,
    grid=major, 
    grid style={dashed,gray!30}, 
    tick label style={font=\tiny}, 
    font=\small,
    ylabel style={align=center, yshift=-1.5em}, 
    restrict y to domain*=0:0.3,
    legend style={font=\tiny,},       
]
\addplot [color=black, thick, dashed] table[x=t_vec, y=VACF_gt] {plots/star_polymer2/star_polymer_11.txt};
\addplot [color=tabblue, thick] table[x=t_vec, y=VACF_net] {plots/star_polymer2/star_polymer_11.txt};
\addplot [color=taborange, thick] table[x=t_vec, y expr={ ( \thisrow{VACF_net} <= 0 || \thisrow{VACF_net} >= 0.3 ) ? nan : \thisrow{VACF_net} },] {plots/star_polymer2/star_polymer_11_gns.txt};
\addplot [color=tabgreen, thick] table[x=t_vec, y expr={ ( \thisrow{VACF_net} <= 0 || \thisrow{VACF_net} >= 0.3 ) ? nan : \thisrow{VACF_net} },] {plots/star_polymer2/star_polymer_11_gns_sde.txt};
\addplot [color=tabred, thick] table[x=t_vec, y=VACF_net] {plots/star_polymer2/star_polymer_11_dpd.txt};
\end{axis}

\begin{axis}[xshift=0.33\textwidth,
    width=0.33\textwidth,
    height=4cm,
    ylabel=RDF,
    grid=major, 
    grid style={dashed,gray!30}, 
    tick label style={font=\tiny}, 
    font=\small,
    ylabel style={align=center, yshift=-1.5em}, 
    legend columns=5,              
    transpose legend=false,
    legend style={
    at={(0.5,1.05)},  
    anchor=south,
    legend cell align=left,
    font=\scriptsize,
    draw=none,                 
    /tikz/every even column/.append style={column sep=0.3cm}
    }
]
\addplot [color=black, thick, dashed] table[x=r_RDF, y=RDF_gt] {plots/star_polymer2/star_polymer_11.txt};
\addplot [color=tabblue, thick] table[x=r_RDF, y=RDF_net] {plots/star_polymer2/star_polymer_11.txt};
\addplot [color=taborange, thick] table[x=r_RDF, y=RDF_net] {plots/star_polymer2/star_polymer_11_gns.txt};
\addplot [color=tabgreen, thick] table[x=r_RDF, y=RDF_net] {plots/star_polymer2/star_polymer_11_gns_sde.txt};
\addplot [color=tabred, thick] table[x=r_RDF, y=RDF_net] {plots/star_polymer2/star_polymer_11_dpd.txt};

\addlegendimage{black, thick, dashed}
\addlegendentry{GT}
\addlegendimage{tabblue, thick}
\addlegendentry{Ours}
\addlegendimage{taborange, thick}
\addlegendentry{GNS}
\addlegendimage{tabgreen, thick}
\addlegendentry{GNS-SDE}
\addlegendimage{tabred, thick}
\addlegendentry{DPD}

\end{axis}

\begin{axis}[xshift=0.66\textwidth,
    width=0.33\textwidth,
    height=4cm,
    ylabel=MSD,
    grid=major, 
    grid style={dashed,gray!30}, 
    tick label style={font=\tiny}, 
    font=\small,
    ylabel style={align=center, yshift=-1.5em}, 
    restrict y to domain*=0:6,
]
\addplot [color=black, thick, dashed] table[x=t_vec, y=MSD_gt] {plots/star_polymer2/star_polymer_11.txt};
\addplot [color=tabblue, thick] table[x=t_vec, y=MSD_net] {plots/star_polymer2/star_polymer_11.txt};
\addplot [color=taborange, thick] table[x=t_vec, y expr={ ( \thisrow{MSD_net} <= 0 || \thisrow{MSD_net} >= 6 ) ? nan : \thisrow{MSD_net} },] {plots/star_polymer2/star_polymer_11_gns.txt};
\addplot [color=tabgreen, thick] table[x=t_vec, y expr={ ( \thisrow{MSD_net} <= 0 || \thisrow{MSD_net} >= 6 ) ? nan : \thisrow{MSD_net} },] {plots/star_polymer2/star_polymer_11_gns_sde.txt};
\addplot [color=tabred, thick] table[x=t_vec, y expr={ ( \thisrow{MSD_net} <= 0 || \thisrow{MSD_net} >= 6 ) ? nan : \thisrow{MSD_net} },] {plots/star_polymer2/star_polymer_11_dpd.txt};
\end{axis}

\node[anchor=north west, font=\sffamily\bfseries\Large] at (-1,2.5) {A};
\node[anchor=north west, font=\sffamily\bfseries\Large] at (-1,-0.5) {B};

\begin{axis}[
    xlabel=$t$,
    width=0.33\textwidth,
    	yshift=-3.2cm,
    height=4cm,
    ylabel=VACF,
    grid=major, 
    grid style={dashed,gray!30}, 
    tick label style={font=\tiny}, 
    font=\small,
    ylabel style={align=center, yshift=-1.5em}, 
    xlabel style={yshift=0.5em},
    restrict y to domain*=0:0.3,
]
\addplot [color=black, thick, dashed] table[x=t_vec, y=VACF_gt] {plots/star_polymer2/star_polymer_51.txt};
\addplot [color=tabblue, thick] table[x=t_vec, y=VACF_net] {plots/star_polymer2/star_polymer_51.txt};
\addplot [color=taborange, thick] table[x=t_vec, y expr={ ( \thisrow{VACF_net} <= 0 || \thisrow{VACF_net} >= 0.3 ) ? nan : \thisrow{VACF_net} },] {plots/star_polymer2/star_polymer_51_gns.txt};
\addplot [color=tabgreen, thick] table[x=t_vec, y expr={ ( \thisrow{VACF_net} <= 0 || \thisrow{VACF_net} >= 0.3 ) ? nan : \thisrow{VACF_net} },] {plots/star_polymer2/star_polymer_51_gns_sde.txt};
\addplot [color=tabred, thick] table[x=t_vec, y=VACF_net] {plots/star_polymer2/star_polymer_51_dpd.txt};
\end{axis}

\begin{axis}[xshift=0.33\textwidth,
	yshift=-3.2cm,
    width=0.33\textwidth,
    height=4cm,
    xlabel=$r$,
    ylabel=RDF,
    grid=major, 
    grid style={dashed,gray!30}, 
    tick label style={font=\tiny}, 
    font=\small,
    ylabel style={align=center, yshift=-1.5em}, 
    xlabel style={yshift=0.5em},
]
\addplot [color=black, thick, dashed] table[x=r_RDF, y=RDF_gt] {plots/star_polymer2/star_polymer_51.txt};
\addplot [color=tabblue, thick] table[x=r_RDF, y=RDF_net] {plots/star_polymer2/star_polymer_51.txt};
\addplot [color=taborange, thick] table[x=r_RDF, y=RDF_net] {plots/star_polymer2/star_polymer_51_gns.txt};
\addplot [color=tabgreen, thick] table[x=r_RDF, y=RDF_net] {plots/star_polymer2/star_polymer_51_gns_sde.txt};
\addplot [color=tabred, thick] table[x=r_RDF, y=RDF_net] {plots/star_polymer2/star_polymer_51_dpd.txt};
\end{axis}

\begin{axis}[xshift=0.66\textwidth,
	yshift=-3.2cm,
    width=0.33\textwidth,
    height=4cm,
    xlabel=$t$,
    ylabel=MSD,
    grid=major, 
    grid style={dashed,gray!30}, 
    tick label style={font=\tiny}, 
    font=\small,
    ylabel style={align=center, yshift=-1.5em}, 
    restrict y to domain*=0:6,
    xlabel style={yshift=0.5em},
]
\addplot [color=black, thick, dashed] table[x=t_vec, y=MSD_gt] {plots/star_polymer2/star_polymer_51.txt};
\addplot [color=tabblue, thick] table[x=t_vec, y=MSD_net] {plots/star_polymer2/star_polymer_51.txt};
\addplot [color=taborange, thick] table[x=t_vec, y expr={ ( \thisrow{MSD_net} <= 0 || \thisrow{MSD_net} >= 6 ) ? nan : \thisrow{MSD_net} },] {plots/star_polymer2/star_polymer_51_gns.txt};
\addplot [color=tabgreen, thick] table[x=t_vec, y expr={ ( \thisrow{MSD_net} <= 0 || \thisrow{MSD_net} >= 6 ) ? nan : \thisrow{MSD_net} },] {plots/star_polymer2/star_polymer_51_gns_sde.txt};
\addplot [color=tabred, thick] table[x=t_vec, y expr={ ( \thisrow{MSD_net} <= 0 || \thisrow{MSD_net} >= 6 ) ? nan : \thisrow{MSD_net} },] {plots/star_polymer2/star_polymer_51_dpd.txt};
\end{axis}

\end{tikzpicture}
\caption{Correlation metrics for the (A) \textsc{Star polymer 11} and (B) \textsc{Star polymer 51} datasets, establishing recovery of structure and dissipative response. \commentquercus{Other deep learning baselines (GNS, GNS-SDE) fail to recover structural or dynamical statistics and produce unstable results (outside the plotting range). In contrast, a classical model such as DPD captures some dynamical information but inaccurately.}}
\label{fig:star_polymers}
\end{figure*}
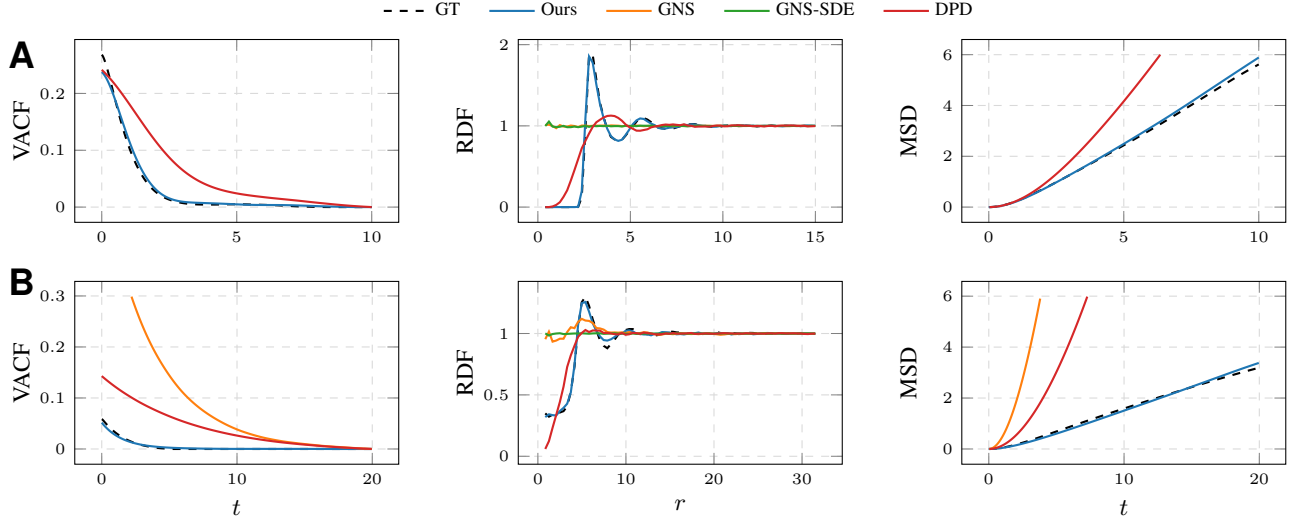

\subsection{Weak scaling study}

\mytabref{tab:time} shows the per timestep performance of the tested models in a serial implementation. The models have been optimized with the TorchScript compiler and run on a single GPU over a $T=2000$ timestep run. As expected, the conventional DPD model exhibits the highest computational efficiency among the models evaluated, primarily due to its minimal architecture with only four trainable parameters. Our method is within a factor of two of GNS/GNS-SDE, as we employ multiple internal neural networks; illustrating that the current framework is of comparable computational complexity to conventional ML pair potentials. Importantly, all coarse-grained models achieve significantly faster runtimes than the fully resolved dynamics, demonstrating the computational benefits of coarse-graining.

Given a trained model, we provide code to perform inference in LAMMPS \cite{plimpton1995fast} as a custom pair style. LAMMPS is an exascale code, and thus provides a means of taking models trained on small systems and performing inference on arbitrarily large numbers of particles. In \myfigref{fig:scaling} we demonstrate weak scaling up to 8 million coarse-grained particles over 64 processors. 

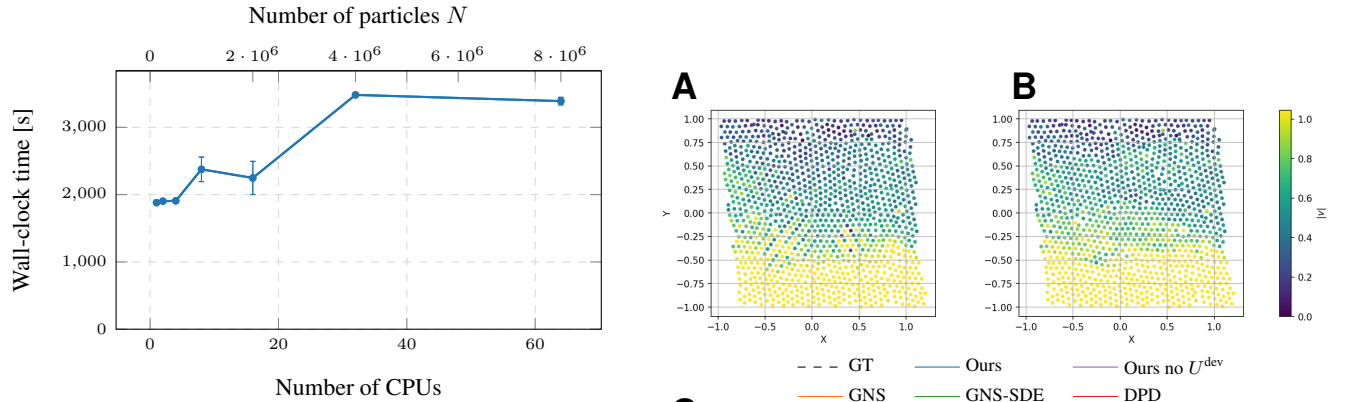
\begin{figure}
    \centering
    \begin{tikzpicture}
\pgfplotsset{width=8cm, height=5cm}
    \begin{axis}[xlabel={Number of CPUs}, 
ylabel={Wall-clock time [s]}, 
                 grid=major, 
                 grid style={dashed,gray!30}, 
                 font=\small, 
                 tick label style={font=\tiny},
                 ymin=0]
        \addplot+[tabblue,thick,mark=*,mark size=1pt,mark options={fill=tabblue}, error bars/.cd,
            y dir=both,
            y explicit] table [col sep=space, 
            x=X1, 
            y=Y, 
            y error=std, 
            header=true] {plots/scaling/data.txt};
            
    \end{axis}
    
    \pgfplotsset{scaled x ticks=false}
    \begin{axis}[axis x line*=top, 
                 xlabel={Number of particles $N$},
                 axis y line= none,
                 font=\small, 
                 tick label style= {font=\tiny}, 
                 xlabel near ticks,
                 ymin=0]
                 
        \addplot+[tabblue,thick,mark=*,mark size=1pt,mark options={fill=tabblue}, error bars/.cd, 
        y dir=both,
        y explicit] table [
        col sep=space, 
        x=X2, 
        y=Y, 
        y error=std,
        header=true] {plots/scaling/data.txt};
        
    \end{axis}
\end{tikzpicture}    
    \caption{Computation time required for a fixed number of CPUs per processor without GPU acceleration. A flat curve denotes weak scaling, demonstrating that systems of arbitrary size may be considered by scaling CPUs proportional to numbers of particles.}
    \label{fig:scaling}
\end{figure}

\section{Extension to general physics problems}

\commentquercus{The previous examples assumed a specific choice of internal energy, }appropriate for a fluid-like system which stores energy only via dilatation.  We now demonstrate how broader classes of physics may be considered by extending the definition of potential energy functional to handle additional inputs by adding a tensorial strain energy dependence appropriate for solid systems. This allows the model to account for reversible effects stored elastically as shear stresses, which are not present in fluids. The derivation of dissipative and noise terms is independent of this addition and remains unchanged; this case serves as an example for how to add additional reversible physics in the framework.

\subsection{Viscoelastic solid}

We consider next a viscoelastic material under shear to demonstrate the capture of strain energy \commentquercus{(\myfigref{fig:results_viscoelastic_app}). This task is equivalent to discovering the underlying constitutive law of the system, encoded in the learned strain energy potential and the dissipative forces.} The predictions of the model\commentquercus{, also reported in \mytabref{tab:results_app},} are in good agreement with the reference finite element solution, up to a large 50\% strain. The model is also able to complete a full cycle starting from rest and returning to its original configuration after a relaxation period. \commentquercus{On the other hand, thermodynamically inconsistent baseline models fail to capture the constitutive behaviour characteristic of solid-like dynamics, resulting in unphysical spatial correlations (\myfigref{fig:results_viscoelastic_app}, denoted as ``Ours no $U^{\text{dev}}$''). Moreover, we have perfomed an ablation study, illustrating that neglecting the new strain energy term precludes prediction of spatial statistics. This is due to the lack of restoring elastic forces which drive the conservative behaviour of deviatoric stress terms.}

\begin{figure}
\centering
\input{plots/results_needle/results_needle}
\caption{\commentquercus{Predicting dynamics and structure of experimental colloidal systems undergoing shear deformation. Dynamics: instantaneous snapshot of jammed colloidal system with particles color-coded by velocity magnitude $|V|$ for (A) experimental dataset and (B) model prediction showing good agreement. (C) Structure: Model can predict the experimental sample radial distribution function (RDF). Neither conventional GNNs or DPD are able to predict structure.}}
\label{fig:results_needle}
\end{figure}

\subsection{Experimental dataset: jammed colloidal system}

Finally, we consider an experimental dataset \PEA{consisting of a dense colloidal particle suspension undergoing cyclic shear deformation \cite{keim2014mechanical,keim2013yielding,galloway2020scaling}. Here, particles are confined to a two-dimensional (2D) monolayer, which allows for visualization of the evolving microstructure; the suspension is a bidisperse (50\% by number) mixture of 4:1 and 5:6 $\mu$m-diameter sulfate latex particles (Invitrogen) adsorbed at the interface between water and oil (decane). This material is subjected to a linear shear deformation in a custom-made interfacial stress rheometer (ISR) \cite{brooks1999interfacial,keim2013yielding,keim2014mechanical}, which provides bulk rheological data. In short, a thin magnetized needle is embedded in the monolayer inside an open channel formed by two walls (\myfigref{fig:overview}B); a static Helmholtz field keeps the needle centered in the channel, while additional electromagnets move the needle back and forth, uniformly shearing the suspension monolayer in the channel. This setup allows for the visualization/characterization of particle position/trajectories using high-resolution optical microscopy. Images are processed using an in-house particle tracking algorithm and velocities are obtained with finite differences. For more details on the experimental setup, please see \cite{keim2013yielding,keim2014mechanical,galloway2020scaling}. This experimental dataset integrates fluid- and solid-like behavior of jammed polycrystaline suspensions, as well as electrostatics, lubrication and contact physics, highlighting the broad applicability of the framework in multiphysics settings.}

\begin{figure}[h]
\centerline{\includegraphics[width=\linewidth]{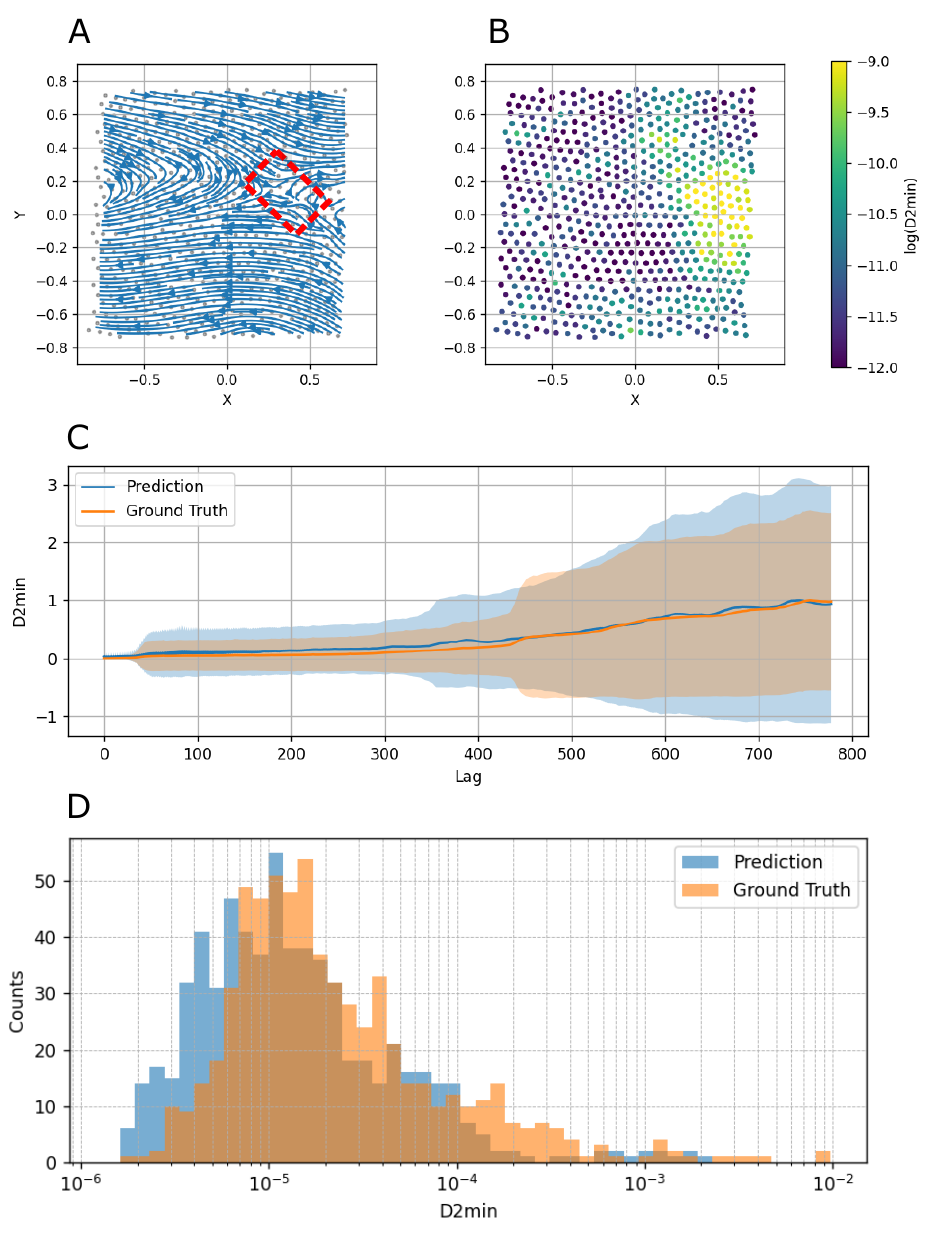}}
\caption{Dynamical properties of the model predictions on the experimental dataset for a domain crop. (A) Streamlines computed from displacements for a lag of 200 snapshots of the first cycle. The red square identifies a hyperbolic motion of a dislocation event fully predicted by the neural network. Predicting the force dipole driving the displacement hopping has been a major challenge.(B) Non-affine squared displacement field ($D^2_{\text{min}}$), where the hotspots represent non-affine deformation clusters, highlighting that the learned model is able to identify local dislocation events. (C) Normalized mean and standard deviation of $D^2_{\text{min}}$ for different lag times on a peak-to-peak cycle. While the magnitude of the force dipole driving the displacement is underestimated, after normalization we observe qualitative agreement with experiment. (D) Histogram of non-affine displacements for lag 150. \label{fig:needle}}
\end{figure}

\PEA{Figure \ref{fig:results_needle} shows snapshots of the experimental (A) and model-predicted (B) sheared colloidal suspension color-coded by their local speed $|V|$. The model successfully reproduces key flow features, including regions of relatively high particle speed near the moving needle (bottom of image) and low particle speed near the stationary wall (top). Figure \ref{fig:results_needle}C shows strong agreement between the experimental and model-predicted radial distribution function (RDF). The RDF indicates that the suspension is disordered at long range but exhibits short-range order, consistent with crystalline domains only a few hundred particles in size. These domains are separated by extensive grain boundaries composed of disordered configurations. Notably, our model outperforms several commonly used methods in predicting the experimental RDF (\myfigref{fig:results_needle}).} 

\PEA{Next, we explore the sample dynamical properties using the proposed framework.  \myfigref{fig:needle}A shows streamlines associated with particle motions during shear.  We find the expected hyperbolic motion of dislocations associated with the force quadrupole, consistent with the geometry of a single plastic event found in incompressible elastic media \cite{picard2004elastic} and in the experimental dataset \cite{keim2014mechanical}. The learned model also predicts local non-affine deformations at the particle scale, as shown in Fig. \ref{fig:needle}B.  Here, we quantify non-affine (particle) rearrangements using the quantity $D^2_{\text{min}}$ \cite{falk_dynamics_1998,keim2014mechanical}. This quantity, computed between 2 instants, measures how much each particle and its nearest ``shells" of neighbors move unlike a continuous elastic solid; it is the mean squared residual displacement after subtracting the best affine transformation \cite{falk_dynamics_1998,keim2014mechanical}. Measures of $D^2_{\text{min}}$ are used to identify discrete, local plastic rearrangements, which are a key feature of the shear transformation zone (STZ) picture of plasticity. Figure \ref{fig:needle}C shows predictions of the mean and standard deviation trends of $D^2_{\text{min}}$ (non-affine deformation) as a function of time for one full deformation cycle, while Fig. \ref{fig:needle}D shows the distribution of non-affine rearrangements at a particular phase (lag 150) in the deformation cycle. Both analyses reveal the good agreement between the experimental dataset and the model. In addition, we compare the statistics of the same dislocation metric histograms of both the reference and the predicted simulations (\myfigref{fig:needle_app}) and observe that, up to a constant factor of $20\%$, the model can reproduce the statistics of local non-affine rearrangements. Overall, the model is mechanistically resolving the force dipole driving displacement events but underestimating the magnitude of the interaction. Nevertheless, the model successfully captures the underlying mechanisms governing dislocation dynamics, which is a notoriously difficult task \cite{galloway2020scaling,manning2011vibrational,richard2020predicting}. Importantly, though these rare events initiate through mechanisms below the coarse-graining length-scale, we obtain correct emergent statistics documented through the $D^2_{min}$ field.}

\section{Discussion}

We have presented a top-down deep learning methodology to construct physical models of conservative, dissipative, and stochastic effects which are consistent with the laws of thermodynamics and detailed fluctuation-dissipation balance. The model is supervised by time-series data for position and velocity only, with a self-supervised technique identifying emergent entropic variables that govern dissipative physics. The technique successfully captures the coarse-grained response of molecular systems using only top-down information, providing a new alternative to traditional bottom-up models which require complete characterization (e.g. via ab initio MD for particles, or other characterization of driving forces). We have further demonstrated that diverse physics may be handled by modifying the internal energy functional, allowing us to capture viscoelasticity and electrokinetic suspensions in a common framework.

The presented methodology is mostly agnostic to the architectural choice of the machine-learnable quantities. The algorithm could be improved with more sophisticated architectures such as cross-attention \cite{waswani2017attention} or conditioning over parametric variables such as differing molecule types. One might also use, e.g. graph autoencoders to directly map from microstructure configuration to the emergent entropy variable rather than using self-supervised learning \cite{lee2019self,zhang2019hierarchical}. We have chosen to demonstrate the framework with simple dense network architectures to highlight the significance of structure-preservation without relying on large datasets or transformers which may not be tractable to curate in data-sparse experimental settings.

Let $\bs{x}$ be a set of a dynamical system state variables. In metriplectic dynamics, the deterministic evolution of $\bs{x}$ is given by
\begin{equation}\label{eq:GENERIC_stochastic}
d\bs{x} = \left[\bs{L}\dpar{E}{\bs{x}} + \bs{M}\dpar{S}{\bs{x}} + k_B\dpar{}{\bs{x}}\cdot\bs{M}\right]dt + d\tilde{\bs{x}}.
\end{equation}
where $E(\bs{x})$ and $S(\bs{x})$ are functionals prescribing generalized energy and entropy of the system, $\bs{L}(\bs{x})=-\bs{L}^T$ is a skew-symmetric Poisson matrix and $\bs{M}(\bs{x})=\bs{M}^T$ is a symmetric positive semi-definite friction matrix. These two terms prescribe the reversible and irreversible dynamics of the system, respectively. By imposing the degeneracy conditions,
\begin{equation}\label{eq:degeneracy}
\bs{L}\bs{\nabla}S=\bs{M}\bs{\nabla}E=\bs{0},
\end{equation}
the dynamics satisfy the following discrete versions of the first and second laws of thermodynamics independent of the choice of state variables; these degeneracy conditions are challenging to impose in machine learning contexts. The stochastic effects $d\tilde{\bs{x}}$ are directly related to the dissipative effects via the fluctuation dissipation theorem \cite{callen1951irreversibility,kubo1966fluctuation}:
\begin{equation}\label{eq:FDT}
d\tilde{\bs{x}}d\tilde{\bs{x}}^T = 2k_B\bs{M}dt,
\end{equation}
where $k_B$ is the Boltzmann constant, meaning that $d\tilde{\bs{x}}$ is an infinitesimal of order $1/2$. Intuitively, the extra mechanical work induced by random fluctuation should be balanced with the dissipation forces to maintain thermodynamical equilibrium. We prove in \myappref{app:proofs} conservation properties and fluctuation-dissipation principles for dynamics of this class.
\begin{proposition}\label{prop:energy_stochastic}
\eqref{eq:GENERIC_stochastic} satisfies energy conservation $dE=0$ if $\bs{M}\bs{\nabla} E=\bs{0}$ and $d\tilde{\bs{x}}\cdot\bs{\nabla}E=0$.
\end{proposition}
\begin{proof}
See proof in \mypropref{prop:energy_stochastic}.
\end{proof}

We next demonstrate a general construction for metriplectic dynamics as a dissipative/stochastic perturbation of an arbitrary non-canonical Hamiltonian system. While we later specialize this to particles, the technique can be applied to any discrete system (e.g. finite element simulations or circuit models). Assume that the state $\bs{x}$ can be partitioned into reversible and irreversible variables $\bs{x} = (\bs{x}_\text{rev},\bs{x}_\text{irr}) $ such that in the absence of irreversibility
\begin{equation}
    d\bs{x}_{\text{rev}} = \bs{L}_{\text{rev}}\dpar{E}{\bs{x}_{\text{rev}}} dt,
\end{equation}
or in other words, $\bs{L}_{\text{rev}}=-\bs{L}_{\text{rev}}^T$ prescribes non-canonical Hamiltonian dynamics for $\bs{x}_{\text{rev}}$. Under the further assumption that the reversible state can be written in terms of generalized position-momentum coordinates such that $\bs{x}_{\text{rev}} = (\bs{q},\bs{p})$ and $\bs{x} = (\bs{q},\bs{p}, \bs{x}_{\text{irr}})$, we define the Poisson matrix as a block system
\begin{equation}
    \bs{L} =
\begin{bmatrix}
\bs{L}_\text{rev} & \bs{0} \\
\bs{0} & \bs{0} \\
\end{bmatrix}
\end{equation}
If we assume that the entropic variable depends only on the irreversible state $S(\bs{x}_{\text{irr}})$ then the degeneracy condition $\bs{L}\bs{\nabla}S=\bs{0}$ holds trivially. To construct noise and dissipation compatible with the other degeneracy conditions $\bs{M}\bs{\nabla}E=\bs{0}$ and $d\tilde{\bs{x}}\cdot\bs{\nabla}E=0$, we express the noise based on the Cholesky decomposition of $\bs{M}=\bs{Q}\bs{Q}^T$ as $d\tilde{\bs{x}} = \bs{Q} d\bs{W}$. We enforce the degeneracy condition via the ansatz,
\begin{equation}
Q_{ji}=C_{ijk}\dpar{E}{\bs{x}_k},
\end{equation}
 where $C_{ijk}$ is a sparse 3-tensor with the skew-symmetry $C_{ijk}=-C_{ikj}$. See \mypropref{prop:3tensor} for the detailed derivation specifying a precise functional form.

With the degeneracy conditions satisfied, we have thus arrived at a construction for a metriplectic system. We construct metriplectic ML architectures via the following ``metriplectic cookbook":  
\begin{enumerate}
	\item Select reversible $\bs{x}_{\text{rev}}$ and irreversible variables $\bs{x}_\text{irr}$ such that $\bs{x}=(\bs{x}_{\text{rev}},\bs{x}_\text{irr})$.
    \item Develop parameterized energies and entropies, assuming functional dependencies $E(\bs{x}_{\text{rev}},\bs{x}_\text{irr})$ and $S(\bs{x}_\text{irr}).$
    \item Develop a parametrization of $d\tilde{\bs{x}}$ (and consequently $\bs{M}$) consistent with the fluctuation-dissipation theorem.
    \item Use the metriplectic \myeqref{eq:GENERIC_stochastic} to determine the dynamical equations.
    \item Use maximum likelihood to fit the model to data, using a marginal distribution over subsets of degrees of freedom to minibatch and maintain $O(1)$/particle complexity training.
    \item Introduce a self-supervised ``teacher network" to identify unobservable $\bs{x}_{irr}$ labels during training.
\end{enumerate}

\subsection{Data-driven particle dynamics}

Consider now a collection of $N$ particles with positions $\bs{r}_i$ and velocities $\bs{v}_i = \frac{d\bs{r}_i}{dt}$ for $i\in V=\{1,...,N\}$. We introduce energy and entropy functionals
\begin{equation}\label{eq:EandS}
E=\sum_i \left[\frac{1}{2}m_i \bs{v}_i^2 + U_i \right],\quad S=\sum_i S_i,
\end{equation}
where: $m_i=m$ is a particle mass, assumed to be constant; $S_i$ is a per particle entropy; $U_i=U(\mathcal{V}_i,S_i)$ is a function prescribing per particle internal energy; and $\mathcal{V}_i(\bs{r}_1,...,\bs{r}_N)$ is a per particle volume which will be prescribed as a function of neighboring particle positions. We identify this with the general framework of the previous section by selecting generalized position and momentum  $(\bs{q}_i,\bs{p}_i)$ as $\bs{x}_{\text{rev},i} =(\bs{r}_i,\bs{v}_i)$ and selecting entropy as the irreversible variable $\bs{x}_{\text{irr},i} = S_i$. 

Recall the first law of thermodynamics $dU=TdS-Pd\mathcal{V}$, from which we may define per particle temperature and pressure as partials of $U$
\begin{equation}\label{eq:EoS}
T_i=\dpar{U_i}{S_i}, \quad P_i=-\dpar{U_i}{\mathcal{V}_i}.
\end{equation}

For these choices, the gradients of energy and entropy are given by
\begin{equation}\label{eq:gradE_gradS}
\dpar{E}{\bs{x}_i} = 
\begin{bmatrix}
\dpar{U}{\bs{r}_i} \\
m_i\bs{v}_i \\
T_i 
\end{bmatrix},\quad
\dpar{S}{\bs{x}_i} =
\begin{bmatrix} 
\bs{0} \\
\bs{0} \\
1 
\end{bmatrix}.
\end{equation}

Substituting these expressions into \myeqref{eq:GENERIC_stochastic} we obtain per particle dynamics
\begin{equation}\label{eq:GENERIC_discrete_part}
\begin{split}
\underbrace{\begin{bmatrix}
d\bs{r}_i \\
d\bs{v}_i \\
dS_i 
\end{bmatrix}}_{d\bs{x}_i}
&= \sum_{j}
\underbrace{\frac{1}{m}
\begin{bmatrix}
\bs{0} & \bs{I}\delta_{ij} & \bs{0}\\
-\bs{I}\delta_{ij} & \bs{0} & \bs{0}\\
\bs{0} & \bs{0} & \bs{0}
\end{bmatrix}}_{\bs{L}_{ij}}
\underbrace{\begin{bmatrix}
\dpar{U}{\bs{r}_j} \\
m\bs{v}_j \\
T_j 
\end{bmatrix}}_{\dpar{E}{\bs{x}_j}} dt \\
&+
\sum_{j}
\bs{M}_{ij}
\underbrace{\begin{bmatrix}
\bs{0} \\
\bs{0} \\
1
\end{bmatrix}}_{\dpar{S}{\bs{x}_j}}dt
+
k_B\sum_{j} \dpar{}{\bs{x}_j}\cdot\bs{M}_{ij}dt
+
\underbrace{\begin{bmatrix}
d\tilde{\bs{r}}_i \\
d\tilde{\bs{v}}_i \\
d\tilde{S}_i 
\end{bmatrix}}_{d\tilde{\bs{x}}_i}.
\end{split}
\end{equation}

\subsection{Machine learnable parameterizations}\label{subsec:states}

To close this system of equations, we will introduce small neural networks shared across all particles to define:
\begin{itemize}
    \item Particle volume $\mathcal{V}_i(\bs{r}_1,...,\bs{r}_N)$.
    \item Per particle internal energy  $U_i(\mathcal{V}_i,S_i)$
    \item Thermal fluctuations $d\tilde{\bs{x}} = \bs{Q} d\bs{W}$.
\end{itemize}
After specifying these terms, $\bs{M}$ may be calculated from $\bs{Q}$, providing a closed system of equations in \myeqref{eq:GENERIC_discrete_part}. For simplicity all neural networks are taken as shallow multilayer perceptrons, unless otherwise noted. There is substantial room for improvement to adopt more sophisticated e.g. graph attention networks or transformers. We have been able to achieve state-of-the-art results with simple architectures and so choose to present this work with minimal complexity in the ML architecture; the simple MLPs adopted here have the added benefit of being amenable to finite difference calculation of gradients at inference time, allowing a fast implementation in LAMMPS.

We stress that in contrast to traditional particle methods, the operators are not constructed to discretize a known PDE; rather, we impose algebraic symmetries which impose structure preservation to build a ``model skeleton", and then use observational data to ``flesh out the skeleton".

\subsubsection{Volume definition}

To obtain a method which scales as $O(N)$ in the number of particles, we prescribe $\mathcal{V}$ to include only contributions from particles within a neighborhood of radius $h$, denoted as $\mathcal{N}_i=\{j\in V,|\bs{r}_{ij}|<h\}$ and parameterize the volume
\begin{equation}\label{eq:volume}
\frac{1}{\mathcal{V}_i}=\sum_{j\in\mathcal{N}_i}W_{ij}=\sum_{j\in\mathcal{N}_i} \exp \left[\text{MLP}_\mathcal{V}\left(\frac{|\bs{r}_{ij}|}{h};\theta_\mathcal{V}\right)\right]\left(1-\frac{|\bs{r}_{ij}|^2}{h^2}\right)_+,
\end{equation}
where for convenience we define the inverse volume $d_i = \mathcal{V}_i^{-1}$, $\bs{r}_{ij}=\bs{r}_i-\bs{r}_j$ denotes the relative positions between particle $i$ and neighboring particles $j\in\mathcal{N}_i$, and $|\cdot|$ denotes the Euclidean norm.

This amounts to a message passing/aggregation network with no attention mechanism. We require only that $W$ be positive to ensure a positive volume for each particle. In light of Noethers theorem, we require that $W$ be invariant under shifts in coordinate to ensure momentum conservation, which we impose by using a radial kernel. More sophisticated graph attention mechanisms, deep message passing networks, and non-radial shift invariant architectures could be incorporated without compromising structure preservation.

We note that this architecture is inspired by classical particle methods such as SPH, where particle volume is defined via a kernel density estimate, see \cite{monaghan1992smoothed,morris1997modeling}, and other works which assign a ``virtual" notion of volume to particles \cite{kwan2012conservative,trask2020conservative}.

\subsubsection{Convexity and monotonicity of internal energy}

In classical variational mechanics, the Gibbs relations mandate that internal energy must respect certain convexity and monotonicity constraints to preserve hyperbolicity in the resulting equations. Equivalently, from a simple physical perspective, pressure and temperature must be positive, requiring $\partial_\mathcal{V} U \leq 0$ and $\partial_S U \geq 0$. The second derivatives define the isothermal compressibility and the specific heat at constant volume, respectively, and must also be positive, i.e. $U$ is convex in both arguments. In this work, we emply the recently proposed Constrained Monotonic Neural Network \cite{runje2023constrained}:
\begin{equation}\label{eq:internal}
U_i=\text{CMNN}_U(S_{i},\mathcal{V}_i;\theta_U).
\end{equation} 
guaranteeing by construction that the first and second arguments be monotonically decreasing and increasing, respectively. This is achieved by enforcing the weight signs, and ensure the convexity by using softplus activation functions so that no constrained optimizers need to be used to obtain a thermodynamically consistent internal energy. For more details, see the \myappref{app:cmnn}.

\subsubsection{Stochastic fluctuations}

When specializing the general ansatz (\myeqref{eq:GENERIC_stochastic}) to a particle setting, additional care must be taken to ensure that forces satisfy conservation of linear and angular momentum. Motivated by \cite{espanol2003smoothed}, we assume a functional form for which random forcing in $d\widetilde{\bs{v}}$ consists only of pairwise interactions, from which the degeneracy conditions enforce as a consequence:

\begin{equation}\label{eq:dx_tilde}
\begin{split}
md\tilde{\bs{v}}_i&=\sqrt{2k_B}\sum_{j\in\mathcal{N}_i}\left[A_{ij}d\bs{\overline{W}}_{ij}+B_{ij}\frac{\bs{I}}{D}\tr(d\bs{W}_{ij})\right]\bs{e}_{ij}, \\
T_id\tilde{S}_i&=-\frac{\sqrt{2k_B}}{2}\sum_{j\in\mathcal{N}_i}\left[A_{ij}d\bs{\overline{W}}_{ij}+B_{ij}\frac{\bs{I}}{D}\tr(d\bs{W}_{ij})\right]:\bs{e}_{ij}\bs{v}_{ij} \\
+&\sqrt{2k_B}\sum_{j\in\mathcal{N}_i} C_{ij}dV_{ij},
\end{split}
\end{equation}
where $A_{ij}=A_{ji}$, $B_{ij}=B_{ji}$ and $C_{ij}=C_{ji}$ are $ij$-symmetric functions of the position and entropy of the particles, $D$ is the dimension of the problem and $\bs{e}_{ij}=\bs{r}_{ij}/|\bs{r}_{ij}|$ is the interdistance unitary vector. The Wiener differentials $d\bs{W}_{ij}=d\bs{W}_{ji}$ and $dV_{ij}=-dV_{ji}$ are symmetric and skew-symmetric, respectively, with regard to particle labels $i,j$. Their components are independent zero mean and unit variance Gaussian increments which are uncorrelated in components and particle index except if they belong to the same interaction (more details in \myappref{app:equations}). 

Unlike in \cite{espanol2003smoothed}, where $A$, $B$, and $C$ were chosen to arrive at an accurate discretization of the fluctuating Navier-Stokes equations, we instead define shallow MLPs to identify state-dependent nonlinear scalings of the fluctuations
\begin{equation}\label{eq:coefs}
\begin{split}
A_{ij}&=\text{MLP}_A\left(\frac{|\bs{r}_{ij}|}{h},T_i;\theta_A\right)\cdot\text{MLP}_A\left(\frac{|\bs{r}_{ij}|}{h},T_j;\theta_A\right),\\
B_{ij}&=\text{MLP}_B\left(\frac{|\bs{r}_{ij}|}{h},T_i;\theta_B\right)\cdot\text{MLP}_B\left(\frac{|\bs{r}_{ij}|}{h},T_j;\theta_B\right),\\
C_{ij}&=\text{MLP}_C\left(\frac{|\bs{r}_{ij}|}{h},T_i;\theta_C\right)\cdot\text{MLP}_C\left(\frac{|\bs{r}_{ij}|}{h},T_j;\theta_C\right),
\end{split}
\end{equation}
parametrized by $\theta_A,\theta_B,\theta_C$, which ensures the symmetry over particle pair indices and the functional dependency of viscosity with temperature. Additionally, we include the Boltzmann constant $k_B$ and the particle mass $m$ as trainable parameters to tune the relative importance of stochastic fluctuations relative to conservative dynamics. This is optional and can be replaced with their actual value, consistent with the units of the studied problem.

\subsection{Dynamic equations}\label{subsec:dynamics}

With the specified parameterization and using the metriplectic equation in \myeqref{eq:GENERIC_discrete_part}, we can finally summarize the dynamic equations of the method:
\begin{equation}
\begin{split}
d\bs{r}_i&= \bs{v}_idt\\
md\bs{v}_i&=-\dpar{U}{\bs{r}_i}dt + \bs{F}_{\text{diss}}^{\bs{v}} + \bs{F}_{\text{div}}^{\bs{v}} + md\tilde{\bs{v}}_i \\
T_idS_i&= F_{\text{diss}}^{S} + F_{\text{div}}^{S} + T_id\tilde{S}_i
\end{split}
\end{equation}
where $F_{\text{diss}}$ are the dissipative terms and $F_{\text{div}}$ the divergence terms of the dynamics, balanced out by the fluctuation-dissipation theorem. The developed formulation ensures the following conservation laws: 

\begin{proposition}
The metriplectic equation in \myeqref{eq:GENERIC_discrete_part} satisfies momentum conservation $d\bs{P}=\bs{0}$.
\end{proposition}
\begin{proof}
See proof in \mypropref{prop:momentum}.
\end{proof}

\begin{proposition}
The metriplectic equation in \myeqref{eq:GENERIC_discrete_part} satisfies energy conservation $dE=0$.
\end{proposition}
\begin{proof}
See proof in \mypropref{prop:energy_stochastic}.
\end{proof}

\subsection{Self-supervision of entropy}\label{subsec:entropy}

While position and velocity are easily measurable either in experimental or by post-processed synthetic data (e.g. fully resolved MD data), calculating entropy would require an enumeration of microstates which is intractable for the vast majority of cases. To fully supervise the system, we therefore need to generate time-series labels of $\bs{x}_{\text{irr}}$. This has been a challenge for several works which have aimed to machine learn metriplectic dynamics \cite{gruber2023reversible,hernandez2021structure}, limiting their applicability to small degree of freedom systems. In this work we adopt a self-supervised approach which allows scalability to massive numbers of degrees of freedom.

To generate labels for $S_i$, we introduce an auxiliary edge convolutional operator \cite{wang2019dynamic} which constructs a closure for the entropy in terms of position and velocity
\begin{equation}\label{eq:S_estimation}
S_i=\frac{1}{|\mathcal{N}_i|}\sum_{j\in\mathcal{N}_i}\text{MLP}_S\left(\frac{|\bs{r}_{ij}|}{h}, \bs{v}_{ij};\theta_S\right).
\end{equation}

In principle, this architecture could incorporate position-velocity information from a finite time history. For simplicity however, we present results using only information from the current timestep. In light of the fluctuation-dissipation theorem, we conjecture that knowledge of the energy dissipated (present in the position and velocity fields) is sufficient to identify the magnitude of the corresponding entropy increment.

We highlight that this architecture is \emph{not} a closure in the dynamics—at the conclusion of training $\bs{M}$ is prescribed, the self-supervision network is discarded, and at inference time entropy is instead evolved by solution of Equation \ref{eq:GENERIC_stochastic}.

\subsection{Training via log-likelihood}\label{subsec:training}

Let $\mathcal{D}=\{(\bs{x}_{\text{rev}}^t,\bs{x}_{\text{rev}}^{t+1})\}_{t=0}^{N_\text{train}}$ be a dataset consisting of position-velocity pairs. To train the system, we construct a joint probability distribution $p(\bs{x}_{\text{rev}}^{t+1},\bs{x}_{\text{irr}}^{t+1}|\bs{x}_{\text{rev}}^{t},\bs{x}_{\text{irr}}^{t})$ corresponding to a numerical integration of Equation \ref{eq:GENERIC_stochastic} and train the architectures with maximum likelihood. We denote all trainable parameters via
$$\Theta=\{\theta_\mathcal{V},\theta_U,\theta_A,\theta_B,\theta_C,\theta_S,k_B,m\}$$ and define the maximum log-likelihood objective
\begin{equation}
\Theta^*=\arg\min_{\Theta} \text{NLL}(x^1,...,x^T).
\end{equation}

Through a marginalization procedure detailed in \myappref{app:NLL}, we can split the problem as per particle-wise loss corresponding to the multivariate Gaussian distribution:
\begin{equation}\label{eq:mu_sigma}
\begin{split}
\bs{\mu}^{t+1}_i&=
\begin{bmatrix}
\bs{v}_i^t+\frac{d\bs{v}_i^t}{dt}\Delta t \\
S_i^t+\frac{dS_i^t}{dt}\Delta t
\end{bmatrix}, \\
\bs{\Sigma}^{t+1}_{ii}&=\begin{bmatrix}
\Delta\tilde{\bs{v}}_i \Delta\tilde{\bs{v}}_i^T &  \Delta\tilde{\bs{v}}_i \Delta\tilde{S}_i\\
\Delta\tilde{S}_i \Delta\tilde{\bs{v}}^T_i &  \Delta\tilde{S}_i\Delta\tilde{S}_i
\end{bmatrix}.
\end{split}
\end{equation}
Thus the final loss function is defined as
\begin{equation}\label{eq:loss}
\begin{split}
\mathcal{L}=&\frac{1}{N_\text{train}}\sum_{t=0}^{N_\text{train}}\frac{1}{N}\sum_{i=0}^{N}\Big[\frac{1}{2}\ln|\bs{\Sigma}_{ii}^{t+1}| \\
&+\frac{1}{2}(\bs{x}_i^{t+1}-\bs{\mu}_i^{t+1})^T\left(\bs{\Sigma}_{ii}^{t+1}\right)^{-1}(\bs{x}_i^{t+1}-\bs{\mu}_i^{t+1})\Big],
\end{split}
\end{equation}
which can be minimized using standard batch gradient descent.

\subsection{Test metrics}\label{subsec:testing}

The benchmark algorithms are trained in fair conditions using the same training procedure and integrator as the presented method (see \myappref{app:gnns} for more details). All methods are tested over long rollouts of temporal length 25 times the training set to ensure steady-state statistics are achieved. Three metrics are considered: velocity autocorrelation function, radial distribution function and mean squared displacement, defined as the following ensemble averages:
\begin{equation}
\begin{split}
\text{VACF}(t)&=\langle\bs{v}(\tau)\cdot\bs{v}(\tau+t)\rangle, \\
\text{RDF}(r)&=\frac{\langle\rho(r)\rangle}{\rho}, \\
\text{MSD}(t)&=\langle|\bs{r}(\tau)-\bs{r}(\tau+t)|^2\rangle,
\end{split}
\end{equation}
where $\rho(r)$ is the local number density of particles at distance $r$ and $\tau$ is a dummy variable representing the initial reference time over which ensemble averages are computed. Reproducing all three correlation metrics is necessary to correctly capture the dynamic behavior of the molecular dynamic system \cite{chandler1988introduction}:
\begin{itemize}
    \item VACF: This metric evaluates how the velocity of a particle at one time is correlated with its velocity at a later time, and provides insight into the momentum relaxation of the system. The integral of the function is the self-diffusion coefficient of the system.
    \item RDF: This metric measures how the particle density varies as a function of distance from a reference particle. It is related to the structural organization and relative arrangements of particles, and is strongly determined by the conservative forces of the system.
    \item MSD: This metric quantifies the average squared distance that particles travel over time. It is used to characterize the diffusive behavior of particles and distinguish between different transport regimes. It is related to the VACF via the Green-Kubo relations \cite{green1954markoff,kubo1957statistical}.
\end{itemize}

The $\text{VACF}, \text{RDF},$ and $\text{MSD}$ curves are compared with the ground truth statistics through the point-wise L2 relative error:
\begin{equation}
\varepsilon_{\bs{y}}=\frac{|\bs{y}_\text{GT}-\bs{y}_\text{pred}|}{|\bs{y}_\text{GT}|}.
\end{equation}
\noindent The code was implemented using Pytorch Geometric \cite{fey2019fast}, and is publicly available online at \url{https://github.com/PIMILab}. All the experiments were performed using Adam optimizer \cite{kingma2014adam} and run on a single NVIDIA H200 GPU. The dataset generation and hyperparameter details can be found in \myappref{app:hyperparams}.

\subsection{Extension to new physics}\label{subsec:elasticity}

To extend the method to viscoelastic, we first need to recall the law of thermodynamics for deformed bodies \cite{landau2012theory} as $dU=TdS+\bs{\sigma}:d\bs{\varepsilon}$. By applying a volumetric and deviatoric decomposition of the stress tensor $\sigma^{\alpha\beta}=-P\delta^{\alpha\beta}+\tau^{\alpha\beta}$, we obtain $dU=TdS-Pd\mathcal{V}+\bs{\tau}:d\bar{\bs{\varepsilon}}$ where $\mathcal{V}$ is the volume and $\bs{\tau}$ and $\bar{\bs{\varepsilon}}$ are the traceless stress and strain tensors. This is the same as the previous equation of state in \myeqref{eq:internal} with an additional shear elastic energy component. Now we may define per particle temperature, pressure and deviatoric stress as partials of $U$
\begin{equation}\label{eq:EoS}
T_i=\dpar{U_i}{S_i}, \quad P_i=-\dpar{U_i}{\mathcal{V}_i} ,\quad \bs{\tau}_i=\dpar{U_i}{\bar{\bs{\varepsilon}}_i}.
\end{equation}

To model the traceless strain tensor, we propose a similar parametrization to \myeqref{eq:volume} but enforce symmetries via
\begin{equation}\label{eq:strain1}
\begin{split}
\bs{\varepsilon}_i&=\sum_{j\in\mathcal{N}_i}\bs{W}_{ij}=\sum_{j\in\mathcal{N}_i}\frac{1}{2}(\bs{l}_{ij}+\bs{l}_{ij}^T),\\
\bs{l}_{ij}&=\text{MLP}_{\bs{\varepsilon}}\left(\frac{\bs{u}_{ij}}{h},\frac{\bs{r}^0_{ij}}{h};\theta_{\bs{\varepsilon}}\right)-\text{MLP}_{\bs{\varepsilon}}\left(\bs{0},\frac{\bs{r}^0_{ij}}{h};\theta_{\bs{\varepsilon}}\right),
\end{split}
\end{equation}
where $\bs{l}_{ij}$ is a lower triangular matrix whose entries are computed using a multilayer perceptron parametrized by $\theta_{\bs{\varepsilon}}$, $\bs{r}_{ij}^0=\bs{r}^0_i-\bs{r}^0_j$ is the undeformed configuration relative positions and $\bs{u}_{ij}=\bs{r}_{ij}-\bs{r}_{ij}^0$ is the displacement with respect to the reference configuration. This parametrization ensures zero strain under zero displacements and is also invariant under coordinate transformations. Finally, we subtract the trace to obtain a traceless strain tensor which will give rise to shear stresses,
\begin{equation}\label{eq:strain}
\bar{\bs{\varepsilon}}_i=\bs{\varepsilon}_i-\frac{\bs{I}}{D}\tr(\bs{\varepsilon}_i)=\sum_{j\in\mathcal{N}_i}\overline{\bs{W}}_{ij},
\end{equation}
where $D$ is the dimension of the problem. This formulation is inspired in large strain elasticity, but many other options could be considered based on classic SPH deformation gradient tensor or equivariant architectures. In regard to the internal energy model, we assume additive internal energies with similar monotonicity and convexity contraints as in the fluid model case (convex in the strain argument, as its second derivative represent the elasticity tensor):
\begin{equation}\label{eq:internal_solid}
\begin{split}
U_i=U^{\text{vol}}_i+U^{\text{dev}}_i=& \text{CMNN}_{U^{\text{vol}}}(S_{i},\mathcal{V}_i;\theta_{U^{\text{vol}}}) \\
&+\text{CMNN}_{U^{\text{dev}}}(S_{i},I(\bar{\bs{\varepsilon}}_i);\theta_{U^{\text{dev}}}),
\end{split}
\end{equation} 
where $I(\bar{\bs{\varepsilon}}_i)$ are the $D-1$ non-zero invariants of the traceless strain tensor. The additional terms in the conservative dynamics equation can be found in the \myappref{app:cons}.

The results are evaluated with L2 relative error in position and RDF function to evaluate the structural fidelity of the predictions. As for the dynamics, we use the non-affine displacement ($D_{\text{min}}^2$) which evaluates how much the local neighbourhood of a particle deviates from an affine transformation between two time snapshots. 
\begin{equation}
D^2_{\text{min}}(t_1,t_2)=\min_{\bs{J}}\langle|\bs{u}_i(t_1)-\bs{J}\bs{u}_j(t_2))|^2\rangle,
\end{equation}
where $\bs{J}$ is the best least squares fit local deformation gradient.

\section*{Acknowledgements}
We thank Pep Español and George Karniadakis for their insightful feedback and valuable discussions throughout the development of this work. We also acknowledge Joel T. Clemmer for his assistance with the LAMMPS code implementation and Aleksandra Vojvodic for the computational resources. The authors acknowledge support from the Department of Energy under the MMICCs program and the Department of Defense under the "Complexity, Nonlocality, and Uncertainty in Heterogeneous Solids" MURI program.

\bibliography{sDPD_ML_bib2}
\bibliographystyle{unsrt}

\clearpage
\onecolumn          
\appendix           

\section{Dynamical equations}\label{app:equations}

This section provides derivations of the complete dynamical equations, from the starting point of the metriplectic problem to the particle discretization used in this work. By recalling the formulation of the main text, consider now a collection of $N$ particles with positions $\bs{r}_i$ and velocities $\bs{v}_i = \frac{d\bs{r}_i}{dt}$ for $i\in V=\{1,...,N\}$. We introduce energy and entropy functionals
\begin{equation}\label{eq:EandS}
E=\sum_i \left[\frac{1}{2}m_i \bs{v}_i^2 + U_i \right],\quad S=\sum_i S_i,
\end{equation}
where: $m_i=m$ is a particle mass, assumed to be constant; $S_i$ is a per particle entropy; $U_i=U(\mathcal{V}_i,S_i,\bar{\bs{\varepsilon}}_i)$ is a function prescribing per particle internal energy; and $\mathcal{V}_i(\bs{r}_1,...,\bs{r}_N)$ is a per particle volume which will be prescribed as a function of neighboring particle positions. The dynamic equations of the system is described by the metriplectic equation:

\begin{equation}\label{eq:GENERIC_discrete_part_app}
\underbrace{\begin{bmatrix}
d\bs{r}_i \\
d\bs{v}_i \\
dS_i 
\end{bmatrix}}_{d\bs{x}_i}
= \sum_{j}
\underbrace{\frac{1}{m}
\begin{bmatrix}
\bs{0} & \bs{I}\delta_{ij} & \bs{0}\\
-\bs{I}\delta_{ij} & \bs{0} & \bs{0}\\
\bs{0} & \bs{0} & \bs{0}
\end{bmatrix}}_{\bs{L}_{ij}}
\underbrace{\begin{bmatrix}
\dpar{U}{\bs{r}_j} \\
m\bs{v}_j \\
T_j 
\end{bmatrix}}_{\dpar{E}{\bs{x}_j}} dt
+
\sum_{j}
\bs{M}_{ij}
\underbrace{\begin{bmatrix}
\bs{0} \\
\bs{0} \\
1
\end{bmatrix}}_{\dpar{S}{\bs{x}_j}}dt
+
k_B\sum_{j} \dpar{}{\bs{x}_j}\cdot\bs{M}_{ij}dt
+
\underbrace{\begin{bmatrix}
d\tilde{\bs{r}}_i \\
d\tilde{\bs{v}}_i \\
d\tilde{S}_i 
\end{bmatrix}}_{d\tilde{\bs{x}}_i},
\end{equation}
subject to the fluctuation-dissipation theorem:
\begin{equation}\label{eq:FDT}
d\tilde{\bs{x}}d\tilde{\bs{x}}^T=2k_B\bs{M}dt.
\end{equation}

The stochastic differentials $d\tilde{\bs{x}}_i=(d\tilde{\bs{r}}_i, d\tilde{\bs{v}}_i,d\tilde{S}_i)$ are defined as: 

\begin{equation}\label{eq:dx_tilde}
\begin{split}
d\tilde{\bs{r}}_i&=\bs{0} \\
md\tilde{\bs{v}}_i&=\sqrt{2k_B}\sum_{j\in\mathcal{N}_i}\left[A_{ij}d\bs{\overline{W}}_{ij}+B_{ij}\frac{\bs{I}}{D}\tr(d\bs{W}_{ij})\right]\bs{e}_{ij}, \\
T_id\tilde{S}_i&=-\frac{\sqrt{2k_B}}{2}\sum_{j\in\mathcal{N}_i}\left[A_{ij}d\bs{\overline{W}}_{ij}+B_{ij}\frac{\bs{I}}{D}\tr(d\bs{W}_{ij})\right]:\bs{e}_{ij}\bs{v}_{ij}
+\sqrt{2k_B}\sum_{j\in\mathcal{N}_i} C_{ij}dV_{ij},
\end{split}
\end{equation}
where $A_{ij}=A_{ji}$, $B_{ij}=B_{ji}$ and $C_{ij}=C_{ji}$ are symmetric functions of the position and entropy of the particles, $D$ is the dimension of the problem and $\bs{e}_{ij}=\bs{r}_{ij}/|\bs{r}_{ij}|$ is the interdistance unitary vector. The Wiener differentials $d\bs{W}_{ij}=d\bs{W}_{ji}$ and $dV_{ij}=-dV_{ji}$ are symmetric and skew-symmetric with respect to particle labels $i,j$. Their components are independent zero mean and unit variance Gaussian increments which are uncorrelated in components and particle index except if they belong to the same interaction. This logic can be compactly expressed by using Kronecker deltas with the following differential rules:
\begin{equation}\label{eq:Wiener}
\begin{split}
d\bs{W}_{ii'}^{\alpha\alpha'}d\bs{W}_{jj'}^{\beta\beta'}&=(\delta_{ij}\delta_{i'j'} + \delta_{i'j}\delta_{ij'})\delta^{\alpha\beta}\delta^{\alpha'\beta'}dt, \\
d V_{ii'}dV_{jj'}&=(\delta_{ij}\delta_{i'j'} - \delta_{i'j}\delta_{ij'})dt, \\
d\bs{W}_{ii'}^{\alpha\alpha'}dV_{jj'}&=0,
\end{split}
\end{equation}
where indices $i,j$ and $i',j'$ denote pairs of particles (latin subscripts) and $\alpha,\beta$ and $\alpha',\beta'$ denote the matrix component indices (greek superscripts). Last, we define the traceless symmetric matrix $d\bs{\overline{W}}_{ij}$ as
\begin{equation}\label{eq:dWbar}
d\bs{\overline{W}}_{ij}=\frac{1}{2}(d\bs{W}_{ij}+d\bs{W}_{ij}^T)-\frac{\bs{I}}{D}\tr(d\bs{W}_{ij}).
\end{equation}

\subsection{Computation of conservative terms}\label{app:cons}

This section aims to derive the conservative terms of the dynamics. Isolating the reversible terms in \myeqref{eq:GENERIC_discrete_part_app}, we obtain the usual Hamilton's equations in our chosen coordinates:
\begin{equation}
\begin{bmatrix}
d\bs{r}_i \\
d\bs{v}_i \\
dS_i 
\end{bmatrix}_\text{rev}
= 
\begin{bmatrix}
\bs{v}_i \\
-\frac{1}{m}\dpar{U}{\bs{r}_i} \\
0
\end{bmatrix}dt.
\end{equation}

We now focus on the non-trivial term, which is the momentum equation. By expanding the gradient and using the chain rule:
\begin{equation}\label{eq:gradU}
-\dpar{U}{\bs{r}_i}=-\sum_k\dpar{U_k}{\bs{r}_i}=-\sum_k\left[\dpar{U_k}{\mathcal{V}_k}\dpar{\mathcal{V}_k}{\bs{r}_i}+\dpar{U_k}{\bar{\bs{\varepsilon}}_k}\dpar{\bar{\bs{\varepsilon}}_k}{\bs{r}_i}\right].
\end{equation}

First, the volumetric term gives rise to a hydrostatic pressure force. By using the equation of state relationships and the volume definition, we get:
\begin{equation}
\begin{split}
-\sum_k\dpar{U_k}{\mathcal{V}_k}\dpar{\mathcal{V}_k}{\bs{r}_i}&=-\sum_k(-P_k)\left(-\frac{1}{d^2_k}\right)\dpar{d_i}{\bs{r}_i}=-\sum_k \frac{P_k}{d^2_k}\sum_j\dpar{W_{kj}}{\bs{r}_i}=-\sum_{k,j}\frac{P_k}{d^2_k}\dpar{W_{kj}}{\bs{r}_{kj}}\dpar{\bs{r}_{kj}}{\bs{r}_i}\\
&= -\sum_{k,j}\frac{P_k}{d^2_k}\dpar{W_{kj}}{\bs{r}_{kj}}(\delta_{ik}-\delta_{ij})=-\sum_j\frac{P_i}{d^2_i}\dpar{W_{ij}}{\bs{r}_{ij}}+\sum_k\frac{P_k}{d^2_k}\dpar{W_{ki}}{\bs{r}_{ki}}=\sum_j\left[\frac{P_j}{d^2_j}\dpar{W_{ji}}{\bs{r}_{ji}}-\frac{P_i}{d^2_i}\dpar{W_{ij}}{\bs{r}_{ij}}\right]\\
\end{split}
\end{equation}

For the elasticity extension, the deviatoric term contributes to the dynamics with a shear elastic restoring force. By using the chain rule, we obtain:
\begin{equation}
\begin{split}
-\sum_k\dpar{U_k}{\bar{\bs{\varepsilon}}^{\alpha\beta}_k}\dpar{\bar{\bs{\varepsilon}}^{\alpha\beta}_k}{\bs{r}^{\gamma}_i}&= -\sum_k\bs{\tau}^{\alpha\beta}_k\sum_j\dpar{\overline{\bs{W}}^{\alpha\beta}_{kj}}{\bs{r}_i^\gamma}=-\sum_k\bs{\tau}^{\alpha\beta}_k\sum_j\dpar{\overline{\bs{W}}^{\alpha\beta}_{kj}}{\bs{u}_{kj}^\delta}\dpar{\bs{u}_{kj}^\delta}{\bs{r}_i^\gamma}=-\sum_{k,j}\bs{\tau}^{\alpha\beta}_k\dpar{\overline{\bs{W}}^{\alpha\beta}_{kj}}{\bs{u}_{kj}^\delta}\delta_{\delta\gamma}(\delta_{ik}-\delta_{ij})\\
&=-\sum_{j}\bs{\tau}^{\alpha\beta}_i\dpar{\overline{\bs{W}}^{\alpha\beta}_{ij}}{\bs{u}_{ij}^\gamma}+\sum_{k}\bs{\tau}^{\alpha\beta}_k\dpar{\overline{\bs{W}}^{\alpha\beta}_{ki}}{\bs{u}_{ki}^\gamma}=\sum_j\left[\bs{\tau}_j^{\alpha\beta}\dpar{\overline{\bs{W}}_{ji}^{\alpha\beta}}{\bs{u}_{ji}^\gamma} - \bs{\tau}_i^{\alpha\beta}\dpar{\overline{\bs{W}}_{ij}^{\alpha\beta}}{\bs{u}_{ij}^\gamma}\right]
\end{split}
\end{equation}

\subsection{Computation of dissipative terms}\label{app:diss}

Once established $d\tilde{\bs{x}}$ by the ansatz in \myeqref{eq:dx_tilde}, the rest of the viscous dynamics can be straightforwardly computed using the fluctuation-dissipation theorem \myeqref{eq:FDT} and substituting the components of the friction matrix $\bs{M}$ and $\bs{\nabla}\cdot\bs{M}$:
\begin{equation}
\begin{split}
\begin{bmatrix}
d\bs{r}_i \\
d\bs{v}_i \\
dS_i 
\end{bmatrix}_\text{irr}
&= 
\sum_{j\in\mathcal{N}_i}
\frac{ d\tilde{\bs{x}}d\tilde{\bs{x}}^T}{2k_B}
\begin{bmatrix}
\bs{0} \\
\bs{0} \\
1
\end{bmatrix}
+
k_B\sum_{j\in\mathcal{N}_i} \dpar{}{\bs{x}_j}\cdot \frac{ d\tilde{\bs{x}}d\tilde{\bs{x}}^T}{2k_B} +d\tilde{\bs{x}}_i \\
&= 
\frac{1}{2k_B}\sum_{j\in\mathcal{N}_i}
\begin{bmatrix}
\bs{0} & \bs{0} & \bs{0}\\
\bs{0} &  d\tilde{\bs{v}}_i d\tilde{\bs{v}}_j^T &  d\tilde{\bs{v}}_i d\tilde{S}_j\\
\bs{0} &  d\tilde{S}_i d\tilde{\bs{v}}_j^T &  d\tilde{S}_i d\tilde{S}_j
\end{bmatrix}
\begin{bmatrix}
\bs{0} \\
\bs{0} \\
1
\end{bmatrix}
+
\frac{1}{2}\sum_{j\in\mathcal{N}_i} \dpar{}{\bs{x}_j}\cdot\begin{bmatrix}
\bs{0} & \bs{0} & \bs{0}\\
\bs{0} &  d\tilde{\bs{v}}_i d\tilde{\bs{v}}_j^T &  d\tilde{\bs{v}}_i d\tilde{S}_j\\
\bs{0} &  d\tilde{S}_i d\tilde{\bs{v}}_j^T &  d\tilde{S}_i d\tilde{S}_j
\end{bmatrix} + d\tilde{\bs{x}}_i\\
&=\frac{1}{2k_B}\sum_{j\in\mathcal{N}_i}
\begin{bmatrix}
\bs{0} \\
 d\tilde{\bs{v}}_i d\tilde{S}_j \\
 d\tilde{S}_i d\tilde{S}_j
\end{bmatrix} + \frac{1}{2}\sum_{j\in\mathcal{N}_i}
\begin{bmatrix}
\bs{0} \\
\dpar{}{\bs{v}_j}\cdot  d\tilde{\bs{v}}_i d\tilde{\bs{v}}_j^T +\dpar{}{S_j}\cdot  d\tilde{\bs{v}}_i d\tilde{S}_j\\
\dpar{}{\bs{v}_j}\cdot  d\tilde{S}_i d\tilde{\bs{v}}_j^T
+\dpar{}{S_j}\cdot  d\tilde{S}_i d\tilde{S}_j
\end{bmatrix}+
\begin{bmatrix}
\bs{0} \\
d\tilde{\bs{v}}_i \\
d\tilde{S}_i
\end{bmatrix}.
\end{split}
\end{equation}

By operating each term, we obtain the following expressions for the irreversible part of the dynamics:
\begin{equation}\label{eq:v_S_irr}
\begin{split}
d\bs{v}_i|_\text{irr}&=\frac{1}{2k_B}\sum_{j\in\mathcal{N}_i} d\tilde{\bs{v}}_i d\tilde{S}_j + \frac{1}{2}\sum_{j\in\mathcal{N}_i}\left(\dpar{}{\bs{v}_j}\cdot  d\tilde{\bs{v}}_i d\tilde{\bs{v}}_j^T +\dpar{}{S_j}\cdot  d\tilde{\bs{v}}_i d\tilde{S}_j \right)+d\tilde{\bs{v}}_i, \\
dS_i|_\text{irr}&=\frac{1}{2k_B}\sum_{j\in\mathcal{N}_i} d\tilde{S}_i d\tilde{S}_j + \frac{1}{2}\sum_{j\in\mathcal{N}_i}\left(\dpar{}{\bs{v}_j}\cdot  d\tilde{S}_i d\tilde{\bs{v}}_j^T +\dpar{}{S_j}\cdot  d\tilde{S}_i d\tilde{S}_j\right)+d\tilde{S}_i.
\end{split}
\end{equation}

In this section we focus in the derivation of each term. First, we need to establish some important identities.

\begin{lemma}\label{lemma:identities}
Given $d\bs{W}_{ii'}$ and $d\bs{W}_{jj'}$ two matrix-valued Wiener processes with the properties stated in \myeqref{eq:Wiener}, the following identities hold:
\begin{enumerate}
\item $\tr(d\bs{W}_{ii'})\tr(d\bs{W}_{jj'})=D(\delta_{ij}\delta_{i'j'} + \delta_{i'j}\delta_{ij'})dt.$
\item $\tr(d\bs{W}_{ii'})d\bs{\overline{W}}_{jj'}=0.$
\item $\displaystyle d\bs{\overline{W}}_{ii'}^{\alpha\alpha'}d\bs{\overline{W}}_{jj'}^{\beta\beta'}=(\delta_{ij}\delta_{i'j'} + \delta_{i'j}\delta_{ij'})\left[\frac{1}{2}(\delta^{\alpha\beta}\delta^{\alpha'\beta'}+\delta^{\alpha'\beta}\delta^{\alpha\beta'})-\frac{1}{D}\delta^{\alpha\alpha'}\delta^{\beta\beta'}\right]dt.$
\item $\displaystyle \left[A_{ii'}d\bs{\overline{W}}_{ii'}^{\alpha\alpha'}+B_{ii'}\frac{\delta^{\alpha\alpha'}}{D}\tr(d\bs{W}_{ii'})\right]\left[A_{jj'}d\bs{\overline{W}}_{jj'}^{\beta\beta'}+B_{jj'}\frac{\delta^{\beta\beta'}}{D}\tr(d\bs{W}_{jj'})\right]\\=(\delta_{ij}\delta_{i'j'} + \delta_{i'j}\delta_{ij'})\left[\frac{1}{2}A_{ii'}A_{jj'}(\delta^{\alpha\beta}\delta^{\alpha'\beta'}+\delta^{\alpha'\beta}\delta^{\alpha\beta'})+(B_{ii'}B_{jj'}-A_{ii'}A_{jj'})\frac{1}{D}\delta^{\alpha\alpha'}\delta^{\beta\beta'}\right]dt.$
\end{enumerate}
\end{lemma}

\begin{proof} First, it is important to remember some of the properties of Kronecker delta operations, which are intimately related to the properties of identity matrices:

\begin{itemize}
\item Symmetric property: $\delta^{\alpha\beta}=\delta^{\beta\alpha}$
\item Summation property: $\sum_\sigma\delta^{\alpha\sigma}\delta^{\sigma\beta}=\delta^{\alpha\beta}$
\item Cumulative summation: $\sum_{\alpha,\beta}\delta^{\alpha\beta}=D$ where $D$ is the dimension of the problem.
\item Conmutativity: $\delta^{\alpha\alpha'}\delta^{\beta\beta'}=\delta^{\beta\beta'}\delta^{\alpha\alpha'}$
\end{itemize}

By recalling the definition of $d\bs{\overline{W}}_{ij}$ in \myeqref{eq:dWbar} and the correlation rules of the Wiener processes defined in \myeqref{eq:Wiener}, the proofs are completed with straightforward algebraic manipulations:

\begin{flalign}
\tr(d\bs{W}_{ii'})\tr(d\bs{W}_{jj'})&=\sum_{\alpha}d\bs{W}_{ii'}^{\alpha\alpha}\sum_{\beta}d\bs{W}_{jj'}^{\beta\beta}=\sum_{\alpha,\beta} d\bs{W}_{ii'}^{\alpha\alpha}d\bs{W}_{jj'}^{\beta\beta} &\\
&=\sum_{\alpha,\beta}(\delta_{ij}\delta_{i'j'} + \delta_{i'j}\delta_{ij'})\delta^{\alpha\beta}\delta^{\alpha\beta}dt&\\
&=(\delta_{ij}\delta_{i'j'} + \delta_{i'j}\delta_{ij'})\sum_{\alpha,\beta}\delta^{\alpha\beta}dt&\\
&=D(\delta_{ij}\delta_{i'j'} + \delta_{i'j}\delta_{ij'})dt.
\end{flalign}

\begin{flalign}
\tr(d\bs{W}_{ii'})d\bs{\overline{W}}_{jj'} &=\tr(d\bs{W}_{ii'})\left[\frac{1}{2}(d\bs{W}_{jj'}^{\alpha\beta}+d\bs{W}_{jj'}^{\beta\alpha})-\frac{\delta^{\alpha\beta}}{D}\tr(d\bs{W}_{jj'})\right]&\\
&=\tr(d\bs{W}_{ii'})\frac{1}{2}(d\bs{W}_{jj'}^{\alpha\beta}+d\bs{W}_{jj'}^{\beta\alpha})-\frac{\delta^{\alpha\beta}}{D}\tr(d\bs{W}_{ii'})\tr(d\bs{W}_{jj'}) &\\
&=\sum_{\sigma}\left[d\bs{W}_{ii'}^{\sigma\sigma}\frac{1}{2}(d\bs{W}_{jj'}^{\alpha\beta}+d\bs{W}_{jj'}^{\beta\alpha})\right]-\frac{\delta^{\alpha\beta}}{D}\tr(d\bs{W}_{ii'})\tr(d\bs{W}_{jj'})&\\
&\stackrel{(1)}{=}\frac{1}{2}\sum_\sigma\left[ d\bs{W}_{ii'}^{\sigma\sigma}d\bs{W}_{jj'}^{\alpha\beta} + d\bs{W}_{ii'}^{\sigma\sigma}d\bs{W}_{jj'}^{\beta\alpha}\right]-\frac{\delta^{\alpha\beta}}{D}D(\delta_{ij}\delta_{i'j'} + \delta_{i'j}\delta_{ij'})dt&\\
&=\frac{1}{2}\sum_\sigma\left[(\delta_{ij}\delta_{i'j'} + \delta_{i'j}\delta_{ij'})(\delta^{\sigma\alpha}\delta^{\sigma\beta} + \delta^{\sigma\beta}\delta^{\sigma\alpha})\right]dt-\delta^{\alpha\beta}(\delta_{ij}\delta_{i'j'} + \delta_{i'j}\delta_{ij'})dt&\\
&=(\delta_{ij}\delta_{i'j'} + \delta_{i'j}\delta_{ij'})\left[\frac{1}{2}\sum_\sigma 2\delta^{\sigma\alpha}\delta^{\sigma\beta}-\delta^{\alpha\beta}\right]dt&\\
&=(\delta_{ij}\delta_{i'j'} + \delta_{i'j}\delta_{ij'})(\delta^{\alpha\beta}-\delta^{\alpha\beta})dt=0.
\end{flalign}

\begin{flalign}
 d\bs{\overline{W}}_{ii'}^{\alpha\alpha'}d\bs{\overline{W}}_{jj'}^{\beta\beta'}&= \left[\frac{1}{2}(d\bs{W}_{ii'}^{\alpha\alpha'}+d\bs{W}_{ii'}^{\alpha'\alpha})-\frac{\delta^{\alpha\alpha'}}{D}\tr(d\bs{W}_{ii'})\right]d\bs{\overline{W}}_{jj'}^{\beta\beta'} &\\
&\stackrel{(2)}{=}\frac{1}{2}(d\bs{W}_{ii'}^{\alpha\alpha'}+d\bs{W}_{ii'}^{\alpha'\alpha})\left[\frac{1}{2}(d\bs{W}_{jj'}^{\beta\beta'}+d\bs{W}_{jj'}^{\beta'\beta})-\frac{\delta^{\beta\beta'}}{D}\tr(d\bs{W}_{jj'})\right] &\\
&=\frac{1}{4}( d\bs{W}_{ii'}^{\alpha\alpha'}d\bs{W}_{jj'}^{\beta\beta'}+ d\bs{W}_{ii'}^{\alpha\alpha'}d\bs{W}_{jj'}^{\beta'\beta}+ d\bs{W}_{ii'}^{\alpha'\alpha}d\bs{W}_{jj'}^{\beta\beta'}+ d\bs{W}_{ii'}^{\alpha'\alpha}d\bs{W}_{jj'}^{\beta'\beta})&\\
&\quad-\frac{\delta^{\beta\beta'}}{2D} (d\bs{W}_{ii'}^{\alpha\alpha'}+d\bs{W}_{ii'}^{\alpha'\alpha})\tr(d\bs{W}_{jj'}) &\\
&\stackrel{(1)}{=}(\delta_{ij}\delta_{i'j'} + \delta_{i'j}\delta_{ij'})\left[\frac{1}{4}(\delta^{\alpha\beta}\delta^{\alpha'\beta'}+\delta^{\alpha\beta'}\delta^{\alpha'\beta}+\delta^{\alpha'\beta}\delta^{\alpha\beta'}+\delta^{\alpha'\beta'}\delta^{\alpha\beta})\right]dt&\\
&\quad-\frac{\delta^{\beta\beta'}}{2D}\sum_\sigma( d\bs{W}_{ii'}^{\alpha\alpha'}d\bs{W}_{jj'}^{\sigma\sigma} + d\bs{W}_{ii'}^{\alpha'\alpha} d\bs{W}_{jj'}^{\sigma\sigma}) &\\
&=(\delta_{ij}\delta_{i'j'} + \delta_{i'j}\delta_{ij'})\left[\frac{1}{4}(2\delta^{\alpha\beta}\delta^{\alpha'\beta'}+2\delta^{\alpha'\beta}\delta^{\alpha\beta'})-\frac{\delta^{\beta\beta'}}{2D}\sum_\sigma(\delta^{\alpha\sigma}\delta^{\alpha'\sigma}+\delta^{\alpha'\sigma}\delta^{\alpha\sigma})\right]dt&\\
&=(\delta_{ij}\delta_{i'j'} + \delta_{i'j}\delta_{ij'})\left[\frac{1}{2}(\delta^{\alpha\beta}\delta^{\alpha'\beta'}+\delta^{\alpha'\beta}\delta^{\alpha\beta'})-\frac{\delta^{\beta\beta'}}{2D}\sum_\sigma 2\delta^{\alpha\sigma}\delta^{\alpha'\sigma}\right]dt&\\
&=(\delta_{ij}\delta_{i'j'} + \delta_{i'j}\delta_{ij'})\left[\frac{1}{2}(\delta^{\alpha\beta}\delta^{\alpha'\beta'}+\delta^{\alpha'\beta}\delta^{\alpha\beta'})-\frac{1}{D}\delta^{\alpha\alpha'}\delta^{\beta\beta'}\right]dt.
\end{flalign}

\begin{flalign}
\Bigg[A_{ii'}d\bs{\overline{W}}_{ii'}^{\alpha\alpha'}&+B_{ii'}\frac{\delta^{\alpha\alpha'}}{D}\tr(d\bs{W}_{ii'})\Bigg]\left[A_{jj'}d\bs{\overline{W}}_{jj'}^{\beta\beta'}+B_{jj'}\frac{\delta^{\beta\beta'}}{D}\tr(d\bs{W}_{jj'})\right]\\
&=A_{ii'}A_{jj'} d\bs{\overline{W}}_{ii'}^{\alpha\alpha'}d\bs{\overline{W}}_{jj'}^{\beta\beta'} + A_{ii'}B_{jj'}\frac{\delta^{\beta\beta'}}{D} d\bs{\overline{W}}_{ii'}^{\alpha\alpha'}\tr(d\bs{W}_{jj'}) &\\
&\quad + B_{ii'}A_{jj'}\frac{\delta^{\alpha\alpha'}}{D}\tr(d\bs{W}_{ii'})d\bs{\overline{W}}_{jj'}^{\beta\beta'}+ B_{ii'}B_{jj'}\frac{\delta^{\alpha\alpha'}\delta^{\beta\beta'}}{D^2} \tr(d\bs{W}_{ii'})\tr(d\bs{W}_{jj'}) &\\
&\stackrel{(2)}{=} A_{ii'}A_{jj'} d\bs{\overline{W}}_{ii'}^{\alpha\alpha'}d\bs{\overline{W}}_{jj'}^{\beta\beta'} + B_{ii'}B_{jj'}\frac{\delta^{\alpha\alpha'}\delta^{\beta\beta'}}{D^2} \tr(d\bs{W}_{ii'})\tr(d\bs{W}_{jj'}) &\\
&\stackrel{(1,3)}{=}(\delta_{ij}\delta_{i'j'} + \delta_{i'j}\delta_{ij'})\left[\frac{A_{ii'}A_{jj'}}{2}(\delta^{\alpha\beta}\delta^{\alpha'\beta'}+\delta^{\alpha'\beta}\delta^{\alpha\beta'})-\frac{A_{ii'}A_{jj'}}{D}\delta^{\alpha\alpha'}\delta^{\beta\beta'}+\frac{B_{ii'}B_{jj'}}{D^2}\delta^{\alpha\alpha'}\delta^{\beta\beta'}D\right]dt &\\
&=(\delta_{ij}\delta_{i'j'} + \delta_{i'j}\delta_{ij'})\left[\frac{A_{ii'}A_{jj'}}{2}(\delta^{\alpha\beta}\delta^{\alpha'\beta'}+\delta^{\alpha'\beta}\delta^{\alpha\beta'})+\frac{B_{ii'}B_{jj'}-A_{ii'}A_{jj'}}{D}\delta^{\alpha\alpha'}\delta^{\beta\beta'}\right]dt.
\end{flalign}

\end{proof}

Now we can compute the $\bs{M}$ matrix by blocks using the fluctuation-dissipation theorem in \myeqref{eq:FDT},
\begin{equation}
\bs{M}_{ij}dt=\frac{ d\tilde{\bs{x}}_id\tilde{\bs{x}}_j^T}{2k_B}=\frac{1}{2k_B}\begin{bmatrix}
\bs{0} & \bs{0} & \bs{0}\\
\bs{0} &  d\tilde{\bs{v}}_i d\tilde{\bs{v}}_j^T &  d\tilde{\bs{v}}_i d\tilde{S}_j\\
\bs{0} &  d\tilde{S}_i d\tilde{\bs{v}}_j^T &  d\tilde{S}_i d\tilde{S}_j
\end{bmatrix}.
\end{equation}

Note that the complete $\bs{M}$ matrix is symmetric, but the individual blocks $\bs{M}_{ij}$ are not. We first need to apply the ansatz of $d\tilde{\bs{x}}$ in \myeqref{eq:dx_tilde}. For convenience, we will compute the individual components of the vectorial quantities, such as $\bs{v}^\alpha$ or $\bs{e}^\alpha$, by using the Einstein summation convention (only for the greek superscripts). We will need the following summation rules for the Kronecker deltas:
\begin{itemize}
\item Index contraction: $\delta^{\alpha\beta}\bs{v}^\alpha = \sum_\alpha\delta^{\alpha\beta}\bs{v}^\alpha = \bs{v}^\beta$
\item Dot product: $\bs{v}^\alpha\delta^{\alpha\beta}\bs{e}^\beta=\sum_{\alpha,\beta}\bs{v}^\alpha\delta^{\alpha\beta}\bs{e}^\beta=\sum_{\alpha}\bs{v}^\alpha\bs{e}^\alpha=\bs{v}\cdot\bs{e}$
\end{itemize}

We will also need to apply the symmetry of the coefficients $A_{ij}=A_{ji}$, $B_{ij}=B_{ji}$ and $C_{ij}=C_{ji}$, the skew-symmetry of $\bs{e}_{ij}=-\bs{e}_{ji}$ and the unitary modulus of $\bs{e}_{ij}\cdot\bs{e}_{ij}=1$. Now, we can compute the individual components by using \mylemmaref{lemma:identities}. First, we will contract the Greek superscripts and then the Roman subscripts.

\begin{flalign}
m^2\frac{ d\tilde{\bs{v}}_i^\alpha d\tilde{\bs{v}}_j^\beta}{2k_B}&=\sum_{i',j'}\left[A_{ii'}d\bs{\overline{W}}_{ii'}^{\alpha\alpha'}+B_{ii'}\frac{\delta^{\alpha\alpha'}}{D}\tr(d\bs{W}_{ii'})\right]\bs{e}_{ii'}^{\alpha'}\left[A_{jj'}d\bs{\overline{W}}_{jj'}^{\beta\beta'}+B_{jj'}\frac{\delta^{\beta\beta'}}{D}\tr(d\bs{W}_{jj'})\right]\bs{e}_{jj'}^{\beta'} &\\
&=\sum_{i',j'}(\delta_{ij}\delta_{i'j'} + \delta_{i'j}\delta_{ij'})\left[\frac{A_{ii'}A_{jj'}}{2}(\delta^{\alpha\beta}\delta^{\alpha'\beta'}+\delta^{\alpha'\beta}\delta^{\alpha\beta'})+\frac{B_{ii'}B_{jj'}-A_{ii'}A_{jj'}}{D}\delta^{\alpha\alpha'}\delta^{\beta\beta'}\right]\bs{e}_{ii'}^{\alpha'}\bs{e}_{jj'}^{\beta'}dt&\\
&=\sum_{i',j'}(\delta_{ij}\delta_{i'j'} + \delta_{i'j}\delta_{ij'})\left[\frac{A_{ii'}A_{jj'}}{2}(\delta^{\alpha\beta}\bs{e}_{ii'}\cdot\bs{e}_{jj'}+\bs{e}_{ii'}^{\beta}\bs{e}_{jj'}^{\alpha})+\frac{B_{ii'}B_{jj'}-A_{ii'}A_{jj'}}{D}\bs{e}_{ii'}^{\alpha}\bs{e}_{jj'}^{\beta}\right]dt&\\
&=\delta_{ij}\sum_{k\in\mathcal{N}_i}\left[\frac{A_{ik}A_{jk}}{2}(\delta^{\alpha\beta}\bs{e}_{ik}\cdot\bs{e}_{jk}+\bs{e}_{ik}^{\beta}\bs{e}_{jk}^{\alpha})+\frac{B_{ik}B_{jk}-A_{ik}A_{jk}}{D}\bs{e}_{ik}^{\alpha}\bs{e}_{jk}^{\beta}\right]dt &\\
&\quad+\left[\frac{A_{ij}A_{ji}}{2}(\delta^{\alpha\beta}\bs{e}_{ij}\cdot\bs{e}_{ji}+\bs{e}_{ij}^{\beta}\bs{e}_{ji}^{\alpha})+\frac{B_{ij}B_{ji}-A_{ij}A_{ji}}{D}\bs{e}_{ij}^{\alpha}\bs{e}_{ji}^{\beta}\right]dt&\\
&=\delta_{ij}\sum_{k\in\mathcal{N}_i}\left[\frac{A_{ik}^2}{2}(\delta^{\alpha\beta}+\bs{e}_{ik}^{\alpha}\bs{e}_{ik}^{\beta})+\frac{B_{ik}^2-A_{ik}^2}{D}\bs{e}_{ik}^{\alpha}\bs{e}_{ik}^{\beta}\right]dt 
-\left[\frac{A_{ij}^2}{2}(\delta^{\alpha\beta}+\bs{e}_{ij}^{\alpha}\bs{e}_{ij}^{\beta})+\frac{B_{ij}^2-A_{ij}^2}{D}\bs{e}_{ij}^{\alpha}\bs{e}_{ij}^{\beta}\right]dt.
\end{flalign}

We can use very similar proofs for the remaining terms. Remember that the Wiener processes $dV_{ij}$ are independent of $d\bs{W}_{ij}$  by \myeqref{eq:Wiener}, so the $C_{ij}$ terms only appear in the entropic term of the matrix. Also, we will need to use the skew-symmetry of $\bs{v}_{ij}=-\bs{v}_{ji}$.

\begin{flalign}
mT_j\frac{ d\tilde{\bs{v}}_i^\alpha d\tilde{S}_j}{2k_B} &=
\sum_{i',j'}\left[A_{ii'}d\bs{\overline{W}}_{ii'}^{\alpha\alpha'}+B_{ii'}\frac{\delta^{\alpha\alpha'}}{D}\tr(d\bs{W}_{ii'})\right]\bs{e}_{ii'}^{\alpha'}\Bigg(-\frac{1}{2}\left[A_{jj'}d\bs{\overline{W}}_{jj'}^{\beta\beta'}+B_{jj'}\frac{\delta^{\beta\beta'}}{D}\tr(d\bs{W}_{jj'})\right]\bs{e}_{jj'}^{\beta'}\bs{v}_{jj'}^{\beta} &\\
&\quad+C_{jj'}dV_{jj'}\Bigg) &\\
&=-\sum_{i',j'}(\delta_{ij}\delta_{i'j'} + \delta_{i'j}\delta_{ij'})\left[\frac{A_{ii'}A_{jj'}}{2}(\delta^{\alpha\beta}\delta^{\alpha'\beta'}+\delta^{\alpha'\beta}\delta^{\alpha\beta'})+\frac{B_{ii'}B_{jj'}-A_{ii'}A_{jj'}}{D}\delta^{\alpha\alpha'}\delta^{\beta\beta'}\right]\bs{e}_{ii'}^{\alpha'}\bs{e}_{jj'}^{\beta'}\frac{\bs{v}_{jj'}^{\beta}}{2}dt&\\
&=-\sum_{i',j'}(\delta_{ij}\delta_{i'j'} + \delta_{i'j}\delta_{ij'})\Bigg[\frac{A_{ii'}A_{jj'}}{2}\left(\frac{\bs{v}_{jj'}^\alpha}{2}\bs{e}_{ii'}\cdot\bs{e}_{jj'}+\bs{e}_{ii'}\cdot\frac{\bs{v}_{jj'}}{2}\bs{e}_{jj'}^\alpha\right)+\frac{B_{ii'}B_{jj'}-A_{ii'}A_{jj'}}{D}\bs{e}_{ii'}^{\alpha}\bs{e}_{jj'}\cdot\frac{\bs{v}_{jj'}}{2}\Bigg]dt&\\
&=-\delta_{ij}\sum_{k\in\mathcal{N}_i}\left[\frac{A_{ik}A_{jk}}{2}\left(\frac{\bs{v}_{jk}^\alpha}{2}\bs{e}_{ik}\cdot\bs{e}_{jk}+\bs{e}_{ik}\cdot\frac{\bs{v}_{jk}}{2}\bs{e}_{jk}^\alpha\right)+\frac{B_{ik}B_{jk}-A_{ik}A_{jk}}{D}\bs{e}_{ik}^{\alpha}\bs{e}_{jk}\cdot\frac{\bs{v}_{jk}}{2}\right]dt &\\
&\quad-\left[\frac{A_{ij}A_{ji}}{2}\left(\frac{\bs{v}_{ji}^\alpha}{2}\bs{e}_{ij}\cdot\bs{e}_{ji}+\bs{e}_{ij}\cdot\frac{\bs{v}_{ji}}{2}\bs{e}_{ji}^\alpha\right)+\frac{B_{ij}B_{ji}-A_{ij}A_{ji}}{D}\bs{e}_{ij}^{\alpha}\bs{e}_{ji}\cdot\frac{\bs{v}_{ji}}{2}\right]dt &\\
&=-\delta_{ij}\sum_{k\in\mathcal{N}_i}\left[\frac{A_{ik}^2}{2}\left(\frac{\bs{v}_{ik}^\alpha}{2}+\bs{e}_{ik}\cdot\frac{\bs{v}_{ik}}{2}\bs{e}_{ik}^\alpha\right)+\frac{B_{ik}^2-A_{ik}^2}{D}\bs{e}_{ik}^{\alpha}\bs{e}_{ik}\cdot\frac{\bs{v}_{ik}}{2}\right]dt &\\
&\quad-\left[\frac{A_{ij}^2}{2}\left(\frac{\bs{v}_{ij}^\alpha}{2}+\bs{e}_{ij}\cdot\frac{\bs{v}_{ij}}{2}\bs{e}_{ij}^\alpha\right)+\frac{B_{ij}^2-A_{ij}^2}{D}\bs{e}_{ij}^{\alpha}\bs{e}_{ij}\cdot\frac{\bs{v}_{ij}}{2}\right]dt. &\\
\end{flalign}

\begin{flalign}
mT_i\frac{ d\tilde{S}_i d\tilde{\bs{v}}_j^\alpha}{2k_B}&=
\sum_{i',j'}\left(-\frac{1}{2}\left[A_{ii'}d\bs{\overline{W}}_{ii'}^{\beta\beta'}+B_{ii'}\frac{\delta^{\beta\beta'}}{D}\tr(d\bs{W}_{ii'})\right]\bs{e}_{ii'}^{\beta'}\bs{v}_{ii'}^{\beta}+C_{ii'}dV_{ii'}\right)\left[A_{jj'}d\bs{\overline{W}}_{jj'}^{\alpha\alpha'}+B_{jj'}\frac{\delta^{\alpha\alpha'}}{D}\tr(d\bs{W}_{jj'})\right]\bs{e}_{jj'}^{\alpha'} &\\
&=-\sum_{i',j'}(\delta_{ij}\delta_{i'j'} + \delta_{ij'}\delta_{i'j})\left[\frac{A_{ii'}A_{jj'}}{2}(\delta^{\alpha\beta}\delta^{\alpha'\beta'}+\delta^{\alpha'\beta}\delta^{\alpha\beta'})+\frac{B_{ii'}B_{jj'}-A_{ii'}A_{jj'}}{D}\delta^{\alpha\alpha'}\delta^{\beta\beta'}\right]\bs{e}_{ii'}^{\beta'}\frac{\bs{v}_{ii'}^{\beta}}{2}\bs{e}_{jj'}^{\alpha'}dt &\\
&=-\sum_{i',j'}(\delta_{ij}\delta_{i'j'} + \delta_{ij''}\delta_{i'j})\left[\frac{A_{ii'}A_{jj'}}{2}\left(\frac{\bs{v}_{ii'}^\alpha}{2}\bs{e}_{ii'}\cdot\bs{e}_{jj'}+\bs{e}_{jj'}\cdot\frac{\bs{v}_{ii'}}{2}\bs{e}_{ii'}^\alpha\right)+\frac{B_{ii'}B_{jj'}-A_{ii'}A_{jj'}}{D}\bs{e}_{jj'}^{\alpha}\bs{e}_{ii'}\cdot\frac{\bs{v}_{ii'}}{2}\right]dt &\\
&=-\delta_{ij}\sum_{k\in\mathcal{N}_i}\left[\frac{A_{ik}A_{jk}}{2}\left(\frac{\bs{v}_{ik}^\alpha}{2}\bs{e}_{ik}\cdot\bs{e}_{jk}+\bs{e}_{jk}\cdot\frac{\bs{v}_{ik}}{2}\bs{e}_{ik}^\alpha\right)+\frac{B_{ik}B_{jk}-A_{ik}A_{jk}}{D}\bs{e}_{jk}^{\alpha}\bs{e}_{ik}\cdot\frac{\bs{v}_{ik}}{2}\right]dt &\\
&\quad-\left[\frac{A_{ij}A_{ji}}{2}\left(\frac{\bs{v}_{ij}^\alpha}{2}\bs{e}_{ij}\cdot\bs{e}_{ji}+\bs{e}_{ji}\cdot\frac{\bs{v}_{ij}}{2}\bs{e}_{ij}^\alpha\right)+\frac{B_{ij}B_{ji}-A_{ij}A_{ji}}{D}\bs{e}_{ji}^{\alpha}\bs{e}_{ij}\cdot\frac{\bs{v}_{ij}}{2}\right]dt&\\
&=-\delta_{ij}\sum_{k\in\mathcal{N}_i}\left[\frac{A_{ik}^2}{2}\left(\frac{\bs{v}_{ik}^\alpha}{2}+\bs{e}_{ik}\cdot\frac{\bs{v}_{ik}}{2}\bs{e}_{ik}^\alpha\right)+\frac{B_{ik}^2-A_{ik}^2}{D}\bs{e}_{ik}^{\alpha}\bs{e}_{ik}\cdot\frac{\bs{v}_{ik}}{2}\right]dt &\\
&\quad+\left[\frac{A_{ij}^2}{2}\left(\frac{\bs{v}_{ij}^\alpha}{2}+\bs{e}_{ij}\cdot\frac{\bs{v}_{ij}}{2}\bs{e}_{ik}^\alpha\right)+\frac{B_{ij}^2-A_{ij}^2}{D}\bs{e}_{ij}^{\alpha}\bs{e}_{ij}\cdot\frac{\bs{v}_{ij}}{2}\right]dt. &\\
\end{flalign}

\begin{flalign}
T_iT_j\frac{ d\tilde{S}_i d\tilde{S}_j}{2k_B} &=
\sum_{i',j'}\left(-\frac{1}{2}\left[A_{ii'}d\bs{\overline{W}}_{ii'}^{\alpha\alpha'}+B_{ii'}\frac{\delta^{\alpha\alpha'}}{D}\tr(d\bs{W}_{ii'})\right]\bs{e}_{ii'}^{\alpha'}\bs{v}_{ii'}^{\alpha}+C_{ii'}dV_{ii'}\right)\Bigg(-\frac{1}{2}\Bigg[A_{jj'}d\bs{\overline{W}}_{jj'}^{\beta\beta'} &\\
&\quad+B_{jj'}\frac{\delta^{\beta\beta'}}{D}\tr(d\bs{W}_{jk})\Bigg]\bs{e}_{jj'}^{\beta'}\bs{v}_{jj'}^{\beta}+C_{jj'}dV_{jj'}\Bigg) &\\
&=\sum_{i,j'}(\delta_{ij}\delta_{i'j'} + \delta_{ij'}\delta_{i'j})\left[\frac{A_{ii'}A_{jj'}}{2}(\delta^{\alpha\beta}\delta^{\alpha'\beta'}+\delta^{\alpha'\beta}\delta^{\alpha\beta'})+\frac{B_{ii'}B_{jj'}-A_{ii'}A_{jj'}}{D}\delta^{\alpha\alpha'}\delta^{\beta\beta'}\right]\bs{e}_{ii'}^{\alpha'}\frac{\bs{v}_{ii'}^{\alpha}}{2}\bs{e}_{jj'}^{\beta'}\frac{\bs{v}_{jj'}^{\beta}}{2}dt &\\
&\quad+\sum_{i',j'}(\delta_{ij}\delta_{i'j'} - \delta_{ij'}\delta_{i'j})C_{ii'}C_{jj'}dt &\\
&=\sum_{i,j'}(\delta_{ij}\delta_{i'j'} + \delta_{ij'}\delta_{i'j})\Bigg(\frac{A_{ii'}A_{jj'}}{2}\left[\left(\frac{\bs{v}_{ii'}}{2}\cdot\frac{\bs{v}_{jj'}}{2}\right)\bs{e}_{ii'}\cdot\bs{e}_{jj'}+\left(\bs{e}_{ii'}\cdot\frac{\bs{v}_{jj'}}{2}\right)\left(\bs{e}_{jj'}\cdot\frac{\bs{v}_{ii'}}{2}\right)\right]&\\
&\quad+\frac{B_{ii'}B_{jj'}-A_{ii'}A_{jj'}}{D}\left(\bs{e}_{ii'}\cdot\frac{\bs{v}_{ii'}}{2}\right)\left(\bs{e}_{jj'}\cdot\frac{\bs{v}_{jj'}}{2}\right)\Bigg)dt +\sum_{i',j'}(\delta_{ij}\delta_{i'j'} - \delta_{ij'}\delta_{i'j})C_{ii'}C_{jj'}dt &\\
&=\delta_{ij}\sum_{k\in\mathcal{N}_i}\Bigg(\frac{A_{ik}A_{jk}}{2}\left[\left(\frac{\bs{v}_{ik}}{2}\cdot\frac{\bs{v}_{jk}}{2}\right)\bs{e}_{ik}\cdot\bs{e}_{jk}+\left(\bs{e}_{ik}\cdot\frac{\bs{v}_{jk}}{2}\right)\left(\bs{e}_{jk}\cdot\frac{\bs{v}_{ik}}{2}\right)\right]&\\
&\quad+\frac{B_{ik}B_{jk}-A_{ik}A_{jk}}{D}\left(\bs{e}_{ik}\cdot\frac{\bs{v}_{ik}}{2}\right)\left(\bs{e}_{jk}\cdot\frac{\bs{v}_{jk}}{2}\right)\Bigg)dt &\\
&\quad+\frac{A_{ij}A_{ji}}{2}\left[\left(\frac{\bs{v}_{ij}}{2}\cdot\frac{\bs{v}_{ji}}{2}\right)\bs{e}_{ij}\cdot\bs{e}_{ji}+\left(\bs{e}_{ij}\cdot\frac{\bs{v}_{ji}}{2}\right)\left(\bs{e}_{ji}\cdot\frac{\bs{v}_{ij}}{2}\right)\right]+\frac{B_{ij}B_{ji}-A_{ij}A_{ji}}{D}\left(\bs{e}_{ij}\cdot\frac{\bs{v}_{ij}}{2}\right)\left(\bs{e}_{ji}\cdot\frac{\bs{v}_{ji}}{2}\right)dt&\\
&\quad+\delta_{ij}\sum_{k\in\mathcal{N}_i} C_{ik}C_{jk}dt-C_{ij}C_{ji}dt &\\
&=\delta_{ij}\sum_{k\in\mathcal{N}_i}\left(\frac{A_{ik}^2}{2}\left[\left(\frac{\bs{v}_{ik}}{2}\right)^2+\left(\bs{e}_{ik}\cdot\frac{\bs{v}_{ik}}{2}\right)^2\right]+\frac{B_{ik}^2-A_{ik}^2}{D}\left(\bs{e}_{ik}\cdot\frac{\bs{v}_{ik}}{2}\right)^2+C_{ik}^2\right)dt &\\
&\quad+\left(\frac{A_{ij}^2}{2}\left[\left(\frac{\bs{v}_{ij}}{2}\right)^2+\left(\bs{e}_{ij}\cdot\frac{\bs{v}_{ij}}{2}\right)^2\right]+\frac{B_{ij}^2-A_{ij}^2}{D}\left(\bs{e}_{ij}\cdot\frac{\bs{v}_{ij}}{2}\right)^2-C_{ij}^2\right)dt. &\\
\end{flalign}

First, lets compute the $\bs{M}$ contributions to the dynamics:

\begin{flalign}
m\sum_{j\in\mathcal{N}_i}\frac{ d\tilde{\bs{v}}_i d\tilde{S}_j}{2k_B}&=-\sum_{j\in\mathcal{N}_i}\frac{1}{T_j}\delta_{ij}\sum_{k\in\mathcal{N}_i}\left[\frac{A_{ik}^2}{2}\left(\frac{\bs{v}_{ik}}{2}+\bs{e}_{ik}\cdot\frac{\bs{v}_{ik}}{2}\bs{e}_{ik}\right)+\frac{B_{ik}^2-A_{ik}^2}{D}\bs{e}_{ik}\bs{e}_{ik}\cdot\frac{\bs{v}_{ik}}{2}\right]dt &\\
&\quad-\sum_{j\in\mathcal{N}_i}\frac{1}{T_j}\left[\frac{A_{ij}^2}{2}\left(\frac{\bs{v}_{ij}}{2}+\bs{e}_{ij}\cdot\frac{\bs{v}_{ij}}{2}\bs{e}_{ij}\right)+\frac{B_{ij}^2-A_{ij}^2}{D}\bs{e}_{ij}\bs{e}_{ij}\cdot\frac{\bs{v}_{ij}}{2}\right]dt &\\
&=-\frac{1}{2T_i}\sum_{j\in\mathcal{N}_i}\left[\frac{A_{ij}^2}{2}\left(\bs{v}_{ij}+\bs{e}_{ij}\cdot\bs{v}_{ij}\bs{e}_{ij}\right)+\frac{B_{ij}^2-A_{ij}^2}{D}\bs{e}_{ij}\bs{e}_{ij}\cdot\bs{v}_{ij}\right]dt &\\
&\quad-\frac{1}{2}\sum_{j\in\mathcal{N}_i}\frac{1}{T_j}\left[\frac{A_{ij}^2}{2}\left(\bs{v}_{ij}+\bs{e}_{ij}\cdot\bs{v}_{ij}\bs{e}_{ij}\right)+\frac{B_{ij}^2-A_{ij}^2}{D}\bs{e}_{ij}\bs{e}_{ij}\cdot\bs{v}_{ij}\right]dt &\\
&=-\frac{1}{2}\sum_{j\in\mathcal{N}_i}\left(\frac{1}{T_i}+\frac{1}{T_j}\right)\left[\frac{A_{ij}^2}{2}\left(\bs{v}_{ij}+\bs{e}_{ij}\cdot\bs{v}_{ij}\bs{e}_{ij}\right)+\frac{B_{ij}^2-A_{ij}^2}{D}\bs{e}_{ij}\bs{e}_{ij}\cdot\bs{v}_{ij}\right]dt &\\
&=-\frac{1}{2}\sum_{j\in\mathcal{N}_i}\left(\frac{1}{T_i}+\frac{1}{T_j}\right)\left[\frac{A_{ij}^2}{2}\bs{v}_{ij}+\left(\frac{A_{ij}^2}{2}+\frac{B_{ij}^2-A_{ij}^2}{D}\right)\bs{e}_{ij}\cdot\bs{v}_{ij}\bs{e}_{ij}\right]dt. &\\
\end{flalign}

\begin{flalign}
T_i\sum_{j\in\mathcal{N}_i}\frac{ d\tilde{S}_i d\tilde{S}_j}{2k_B}&=\sum_{j\in\mathcal{N}_i}\frac{1}{T_j}\delta_{ij}\sum_{k\in\mathcal{N}_i}\left(\frac{A_{ik}^2}{2}\left[\left(\frac{\bs{v}_{ik}}{2}\right)^2+\left(\bs{e}_{ik}\cdot\frac{\bs{v}_{ik}}{2}\right)^2\right]+\frac{B_{ik}^2-A_{ik}^2}{D}\left(\bs{e}_{ik}\cdot\frac{\bs{v}_{ik}}{2}\right)^2+C_{ik}^2\right)dt &\\
&\quad+\sum_{j\in\mathcal{N}_i}\frac{1}{T_j}\left(\frac{A_{ij}^2}{2}\left[\left(\frac{\bs{v}_{ij}}{2}\right)^2+\left(\bs{e}_{ij}\cdot\frac{\bs{v}_{ij}}{2}\right)^2\right]+\frac{B_{ij}^2-A_{ij}^2}{D}\left(\bs{e}_{ij}\cdot\frac{\bs{v}_{ij}}{2}\right)^2-C_{ij}^2\right)dt &\\
&=\frac{1}{T_i}\sum_{j\in\mathcal{N}_i}\left(\frac{A_{ij}^2}{2}\left[\left(\frac{\bs{v}_{ij}}{2}\right)^2+\left(\bs{e}_{ij}\cdot\frac{\bs{v}_{ij}}{2}\right)^2\right]+\frac{B_{ij}^2-A_{ij}^2}{D}\left(\bs{e}_{ij}\cdot\frac{\bs{v}_{ij}}{2}\right)^2+C_{ij}^2\right)dt &\\
&\quad+\sum_{j\in\mathcal{N}_i}\frac{1}{T_j}\left(\frac{A_{ij}^2}{2}\left[\left(\frac{\bs{v}_{ij}}{2}\right)^2+\left(\bs{e}_{ij}\cdot\frac{\bs{v}_{ij}}{2}\right)^2\right]+\frac{B_{ij}^2-A_{ij}^2}{D}\left(\bs{e}_{ij}\cdot\frac{\bs{v}_{ij}}{2}\right)^2-C_{ij}^2\right)dt &\\
&=\frac{1}{4}\sum_{j\in\mathcal{N}_i}\left(\frac{1}{T_i}+\frac{1}{T_j}\right)\left(\frac{A_{ij}^2}{2}\left[\bs{v}_{ij}^2+(\bs{v}_{ij}\cdot\bs{e}_{ij})^2\right]+\frac{B_{ij}^2-A_{ij}^2}{D}(\bs{v}_{ij}\cdot\bs{e}_{ij})^2\right)dt+\sum_{j\in\mathcal{N}_i}\left(\frac{1}{T_i}-\frac{1}{T_j}\right)C_{ij}^2dt &\\
&=\frac{1}{4}\sum_{j\in\mathcal{N}_i}\left(\frac{1}{T_i}+\frac{1}{T_j}\right)\left[\frac{A_{ij}^2}{2}\bs{v}_{ij}^2+\left(\frac{A_{ij}^2}{2}+\frac{B_{ij}^2-A_{ij}^2}{D}\right)(\bs{v}_{ij}\cdot\bs{e}_{ij})^2\right]dt+\sum_{j\in\mathcal{N}_i}\left(\frac{1}{T_i}-\frac{1}{T_j}\right)C_{ij}^2dt. &\\
\end{flalign}

Last, we derive the divergence terms $\bs{\nabla}\cdot\bs{M}$. Note that we need to define the heat capacity at constant volume $C_i$ as the following expression:
\begin{equation}
\dpar{T_i}{S_i}=\frac{T_i}{C_i}.
\end{equation} 

We also need to establish the following divergence identities
\begin{equation}
\begin{split}
\dpar{}{\bs{v}_i}\cdot\bs{v}_{ij}&=\dpar{}{\bs{v}_i}\cdot(\bs{v}_i-\bs{v}_j)=\dpar{}{\bs{v}_i}\cdot\bs{v}_i=D,\\
\dpar{}{\bs{v}_j}\cdot\bs{v}_{ij}&=\dpar{}{\bs{v}_j}\cdot(\bs{v}_i-\bs{v}_j)=-\dpar{}{\bs{v}_j}\cdot\bs{v}_j=-D,\\
\dpar{}{\bs{v}_j}\cdot(\bs{e}_{ij}\cdot\bs{v}_{ij}\bs{e}_{ij})&=\dpar{}{\bs{v}_j^\beta}\bs{e}_{ij}^\alpha\bs{v}_{ij}^\alpha\bs{e}_{ij}^\beta=-\dpar{}{\bs{v}_j^\beta}\bs{e}_{ij}^\alpha\bs{v}_{ji}^\alpha\bs{e}_{ij}^\beta=-\bs{e}_{ij}^\beta\bs{e}_{ij}^\beta=-1,
\end{split}
\end{equation}
and usual derivatives
\begin{equation}
\dpar{}{S_i}\left(\frac{1}{T_i}\right)=-\frac{1}{T_iC_i}, \quad \dpar{}{S_i}\left(\frac{1}{T_i^2}\right)=-\frac{2}{T_i^2C_i}, \quad \dpar{A_{ij}}{S_i}=\dpar{A_{ij}}{T_i}\dpar{T_i}{S_i}=\dpar{A_{ij}}{T_i}\frac{T_i}{C_i}.
\end{equation} 

Now we can compute the remaining terms:

\begin{flalign}
m^2\sum_{j\in\mathcal{N}_i}\dpar{}{\bs{v}_j}\cdot\frac{ d\tilde{\bs{v}}_i d\tilde{\bs{v}}_j^T}{2k_B}&=\sum_{j\in\mathcal{N}_i}\dpar{}{\bs{v}_j}\cdot\delta_{ij}\sum_{k\in\mathcal{N}_i}\left[\frac{A_{ik}^2}{2}\left(\bs{I} + \bs{e}_{ik}\bs{e}_{ik}^T\right)+\frac{B_{ik}^2-A_{ik}^2}{D}\bs{e}_{ik}\bs{e}_{ik}^T\right]dt &\\
&\quad-\sum_{j\in\mathcal{N}_i}\dpar{}{\bs{v}_j}\cdot\left[\frac{A_{ij}^2}{2}\left(\bs{I} + \bs{e}_{ij}\bs{e}_{ij}^T\right)+\frac{B_{ij}^2-A_{ij}^2}{D}\bs{e}_{ij}\bs{e}_{ij}^T\right]dt=\bs{0}. &\\
\end{flalign}

\begin{flalign}
m\sum_{j\in\mathcal{N}_i}\dpar{}{S_j}\cdot\frac{ d\tilde{\bs{v}}_i d\tilde{S}_j}{2k_B}&=-\sum_{j\in\mathcal{N}_i}\dpar{}{S_j}\cdot\frac{1}{T_j}\delta_{ij}\sum_{k\in\mathcal{N}_i}\left[\frac{A_{ik}^2}{2}\left(\frac{\bs{v}_{ik}}{2}+\bs{e}_{ik}\cdot\frac{\bs{v}_{ik}}{2}\bs{e}_{ik}\right)+\frac{B_{ik}^2-A_{ik}^2}{D}\bs{e}_{ik}\bs{e}_{ik}\cdot\frac{\bs{v}_{ik}}{2}\right]dt &\\
&\quad-\sum_{j\in\mathcal{N}_i}\dpar{}{S_j}\cdot\frac{1}{T_j}\left[\frac{A_{ij}^2}{2}\left(\frac{\bs{v}_{ij}}{2}+\bs{e}_{ij}\cdot\frac{\bs{v}_{ij}}{2}\bs{e}_{ij}\right)+\frac{B_{ij}^2-A_{ij}^2}{D}\bs{e}_{ij}\bs{e}_{ij}\cdot\frac{\bs{v}_{ij}}{2}\right]dt &\\
&=-\frac{1}{2}\dpar{}{S_i}\cdot\frac{1}{T_i}\sum_{j\in\mathcal{N}_i}\left[\frac{A_{ij}^2}{2}\left(\bs{v}_{ij}+\bs{e}_{ij}\cdot\bs{v}_{ij}\bs{e}_{ij}\right)+\frac{B_{ij}^2-A_{ij}^2}{D}\bs{e}_{ij}\bs{e}_{ij}\cdot\bs{v}_{ij}\right]dt &\\
&\quad-\frac{1}{2}\sum_{j\in\mathcal{N}_i}\dpar{}{S_j}\cdot\frac{1}{T_j}\left[\frac{A_{ij}^2}{2}\left(\bs{v}_{ij}+\bs{e}_{ij}\cdot\bs{v}_{ij}\bs{e}_{ij}\right)+\frac{B_{ij}^2-A_{ij}^2}{D}\bs{e}_{ij}\bs{e}_{ij}\cdot\bs{v}_{ij}\right]dt &\\
&=\frac{1}{2}\sum_{j\in\mathcal{N}_i}\left(\frac{1}{T_iC_i}+\frac{1}{T_jC_j}\right)\left[\frac{A_{ij}^2}{2}\bs{v}_{ij}+\left(\frac{A_{ij}^2}{2}+\frac{B_{ij}^2-A_{ij}^2}{D}\right)\bs{e}_{ij}\cdot\bs{v}_{ij}\bs{e}_{ij}\right]dt &\\
&\quad -\frac{1}{2}\sum_{j\in\mathcal{N}_i}\frac{1}{T_i}\left[A_{ij}\dpar{A_{ij}}{S_i}\bs{v}_{ij}+\left(A_{ij}\dpar{A_{ij}}{S_i}+\frac{2}{D}B_{ij}\dpar{B_{ij}}{S_i} - \frac{2}{D}A_{ij}\dpar{A_{ij}}{S_i}\right)\bs{e}_{ij}\cdot\bs{v}_{ij}\bs{e}_{ij}\right]dt &\\
&\quad -\frac{1}{2}\sum_{j\in\mathcal{N}_i}\frac{1}{T_j}\left[A_{ij}\dpar{A_{ij}}{S_j}\bs{v}_{ij}+\left(A_{ij}\dpar{A_{ij}}{S_j}+\frac{2}{D}B_{ij}\dpar{B_{ij}}{S_j} - \frac{2}{D}A_{ij}\dpar{A_{ij}}{S_j}\right)\bs{e}_{ij}\cdot\bs{v}_{ij}\bs{e}_{ij}\right]dt &\\
&=\frac{1}{2}\sum_{j\in\mathcal{N}_i}\left(\frac{1}{T_iC_i}+\frac{1}{T_jC_j}\right)\left[\frac{A_{ij}^2}{2}\bs{v}_{ij}+\left(\frac{A_{ij}^2}{2}+\frac{B_{ij}^2-A_{ij}^2}{D}\right)\bs{e}_{ij}\cdot\bs{v}_{ij}\bs{e}_{ij}\right]dt &\\
&\quad -\frac{1}{2}\sum_{j\in\mathcal{N}_i}\frac{1}{C_i}\left[A_{ij}\dpar{A_{ij}}{T_i}\bs{v}_{ij}+\left(A_{ij}\dpar{A_{ij}}{T_i}+\frac{2}{D}B_{ij}\dpar{B_{ij}}{T_i} - \frac{2}{D}A_{ij}\dpar{A_{ij}}{T_i}\right)\bs{e}_{ij}\cdot\bs{v}_{ij}\bs{e}_{ij}\right]dt &\\
&\quad -\frac{1}{2}\sum_{j\in\mathcal{N}_i}\frac{1}{C_j}\left[A_{ij}\dpar{A_{ij}}{T_j}\bs{v}_{ij}+\left(A_{ij}\dpar{A_{ij}}{T_j}+\frac{2}{D}B_{ij}\dpar{B_{ij}}{T_j} - \frac{2}{D}A_{ij}\dpar{A_{ij}}{T_j}\right)\bs{e}_{ij}\cdot\bs{v}_{ij}\bs{e}_{ij}\right]dt &\\
\end{flalign}

\begin{flalign}
T_i\sum_{j\in\mathcal{N}_i}\dpar{}{\bs{v}_j}\cdot\frac{ d\tilde{S}_i d\tilde{\bs{v}}_j^T}{2k_B}&=-\sum_{j\in\mathcal{N}_i}\dpar{}{\bs{v}_j}\cdot\frac{1}{m}\delta_{ij}\sum_{k\in\mathcal{N}_i}\left[\frac{A_{ik}^2}{2}\left(\frac{\bs{v}_{ik}}{2}+\bs{e}_{ik}\cdot\frac{\bs{v}_{ik}}{2}\bs{e}_{ik}\right)+\frac{B_{ik}^2-A_{ik}^2}{D}\bs{e}_{ik}\bs{e}_{ik}\cdot\frac{\bs{v}_{ik}}{2}\right]dt &\\
&\quad+\sum_{j\in\mathcal{N}_i}\dpar{}{\bs{v}_j}\cdot\frac{1}{m}\left[\frac{A_{ij}^2}{2}\left(\frac{\bs{v}_{ij}}{2}+\bs{e}_{ij}\cdot\frac{\bs{v}_{ij}}{2}\bs{e}_{ij}\right)+\frac{B_{ij}^2-A_{ij}^2}{D}\bs{e}_{ij}\bs{e}_{ij}\cdot\frac{\bs{v}_{ij}}{2}\right]dt &\\
&=-\frac{1}{2m}\dpar{}{\bs{v}_i}\cdot\sum_{j\in\mathcal{N}_i}\left[\frac{A_{ij}^2}{2}\left(\bs{v}_{ij}+\bs{e}_{ij}\cdot\bs{v}_{ij}\bs{e}_{ij}\right)+\frac{B_{ij}^2-A_{ij}^2}{D}\bs{e}_{ij}\bs{e}_{ij}\cdot\bs{v}_{ij}\right]dt &\\
&\quad+\frac{1}{2m}\sum_{j\in\mathcal{N}_i}\dpar{}{\bs{v}_j}\cdot\left[\frac{A_{ij}^2}{2}\left(\bs{v}_{ij}+\bs{e}_{ij}\cdot\bs{v}_{ij}\bs{e}_{ij}\right)+\frac{B_{ij}^2-A_{ij}^2}{D}\bs{e}_{ij}\bs{e}_{ij}\cdot\bs{v}_{ij}\right]dt &\\
&=-\frac{1}{m}\sum_{j\in\mathcal{N}_i}\left[(D+1)\frac{A_{ij}^2}{2}+\frac{B_{ij}^2-A_{ij}^2}{D}\right]dt.&\\
\end{flalign}

\begin{flalign}
\sum_{j\in\mathcal{N}_i}\dpar{}{S_j}\cdot\frac{ d\tilde{S}_i d\tilde{S}_j}{2k_B}&=\sum_{j\in\mathcal{N}_i}\dpar{}{S_j}\cdot\frac{1}{T_iT_j}\delta_{ij}\sum_{k\in\mathcal{N}_i}\left(\frac{A_{ik}^2}{2}\left[\left(\frac{\bs{v}_{ik}}{2}\right)^2+\left(\bs{e}_{ik}\cdot\frac{\bs{v}_{ik}}{2}\right)^2\right]+\frac{B_{ik}^2-A_{ik}^2}{D}\left(\bs{e}_{ik}\cdot\frac{\bs{v}_{ik}}{2}\right)^2+C_{ik}^2\right)dt &\\
&\quad+\sum_{j\in\mathcal{N}_i}\dpar{}{S_j}\cdot\frac{1}{T_iT_j}\left(\frac{A_{ij}^2}{2}\left[\left(\frac{\bs{v}_{ij}}{2}\right)^2+\left(\bs{e}_{ij}\cdot\frac{\bs{v}_{ij}}{2}\right)^2\right]+\frac{B_{ij}^2-A_{ij}^2}{D}\left(\bs{e}_{ij}\cdot\frac{\bs{v}_{ij}}{2}\right)^2-C_{ij}^2\right)dt &\\
&=\dpar{}{S_i}\cdot\frac{1}{T_i^2}\sum_{j\in\mathcal{N}_i}\left(\frac{A_{ij}^2}{2}\left[\left(\frac{\bs{v}_{ij}}{2}\right)^2+\left(\bs{e}_{ij}\cdot\frac{\bs{v}_{ij}}{2}\right)^2\right]+\frac{B_{ij}^2-A_{ij}^2}{D}\left(\bs{e}_{ij}\cdot\frac{\bs{v}_{ij}}{2}\right)^2+C_{ij}^2\right)dt &\\
&\quad+\sum_{j\in\mathcal{N}_i}\dpar{}{S_j}\cdot\frac{1}{T_iT_j}\left(\frac{A_{ij}^2}{2}\left[\left(\frac{\bs{v}_{ij}}{2}\right)^2+\left(\bs{e}_{ij}\cdot\frac{\bs{v}_{ij}}{2}\right)^2\right]+\frac{B_{ij}^2-A_{ij}^2}{D}\left(\bs{e}_{ij}\cdot\frac{\bs{v}_{ij}}{2}\right)^2-C_{ij}^2\right)dt &\\
&=-\frac{1}{4T_i}\sum_{j\in\mathcal{N}_i}\left(\frac{2}{T_iC_i}+\frac{1}{T_jC_j}\right)\left[\frac{A_{ij}^2}{2}\bs{v}_{ij}^2+\left(\frac{A_{ij}^2}{2}+\frac{B_{ij}^2-A_{ij}^2}{D}\right)(\bs{v}_{ij}\cdot\bs{e}_{ij})^2\right]dt &\\
&\quad-\frac{1}{T_i}\sum_{j\in\mathcal{N}_i}\left(\frac{2}{T_iC_i}-\frac{1}{T_jC_j}\right)C_{ij}^2dt &\\
&\quad+\frac{1}{4T_i}\sum_{j\in\mathcal{N}_i}\frac{1}{T_i}\left[A_{ij}\dpar{A_{ij}}{S_i}\bs{v}_{ij}^2+\left(A_{ij}\dpar{A_{ij}}{S_i}+\frac{2}{D}B_{ij}\dpar{B_{ij}}{S_i} - \frac{2}{D}A_{ij}\dpar{A_{ij}}{S_i}\right)(\bs{v}_{ij}\cdot\bs{e}_{ij})^2\right]dt &\\
&\quad+\frac{1}{4T_i}\sum_{j\in\mathcal{N}_i}\frac{1}{T_j}\left[A_{ij}\dpar{A_{ij}}{S_j}\bs{v}_{ij}^2+\left(A_{ij}\dpar{A_{ij}}{S_j}+\frac{2}{D}B_{ij}\dpar{B_{ij}}{S_j} - \frac{2}{D}A_{ij}\dpar{A_{ij}}{S_j}\right)(\bs{v}_{ij}\cdot\bs{e}_{ij})^2\right]dt &\\
&\quad+\frac{1}{T_i}\sum_{j\in\mathcal{N}_i}\left(\frac{1}{T_i}2C_{ij}\dpar{C_{ij}}{S_i} - \frac{1}{T_j}2C_{ij}\dpar{C_{ij}}{S_j}\right)dt &\\
&=-\frac{1}{4T_i}\sum_{j\in\mathcal{N}_i}\left(\frac{2}{T_iC_i}+\frac{1}{T_jC_j}\right)\left[\frac{A_{ij}^2}{2}\bs{v}_{ij}^2+\left(\frac{A_{ij}^2}{2}+\frac{B_{ij}^2-A_{ij}^2}{D}\right)(\bs{v}_{ij}\cdot\bs{e}_{ij})^2\right]dt &\\
&\quad-\frac{1}{T_i}\sum_{j\in\mathcal{N}_i}\left(\frac{2}{T_iC_i}-\frac{1}{T_jC_j}\right)C_{ij}^2dt &\\
&\quad+\frac{1}{4T_i}\sum_{j\in\mathcal{N}_i}\frac{1}{C_i}\left[A_{ij}\dpar{A_{ij}}{T_i}\bs{v}_{ij}^2+\left(A_{ij}\dpar{A_{ij}}{T_i}+\frac{2}{D}B_{ij}\dpar{B_{ij}}{T_i} - \frac{2}{D}A_{ij}\dpar{A_{ij}}{T_i}\right)(\bs{v}_{ij}\cdot\bs{e}_{ij})^2\right]dt &\\
&\quad+\frac{1}{4T_i}\sum_{j\in\mathcal{N}_i}\frac{1}{C_j}\left[A_{ij}\dpar{A_{ij}}{T_j}\bs{v}_{ij}^2+\left(A_{ij}\dpar{A_{ij}}{T_j}+\frac{2}{D}B_{ij}\dpar{B_{ij}}{T_j} - \frac{2}{D}A_{ij}\dpar{A_{ij}}{T_j}\right)(\bs{v}_{ij}\cdot\bs{e}_{ij})^2\right]dt &\\
&\quad+\frac{2}{T_i}\sum_{j\in\mathcal{N}_i}\left(\frac{1}{C_i}C_{ij}\dpar{C_{ij}}{T_i} - \frac{1}{C_j}C_{ij}\dpar{C_{ij}}{T_j}\right)dt. &\\
\end{flalign}

\subsection{Complete equations}\label{subsec:dynamics}

With the parameterization specified we finally can summarize the dynamic equations of the method.

\begin{flalign}
d\bs{r}_i&=\bs{v}_idt, \label{eq:r_eq}&\\
md\bs{v}_i&=\sum_{j\in\mathcal{N}_i}\left[\frac{P_j}{d_j^2} \dpar{W_{ji}}{\bs{r}_{ji}}-\frac{P_i}{d_i^2} \dpar{W_{ij}}{\bs{r}_{ij}}\right]dt + \sum_j\left[\bs{\tau}_j\dpar{\overline{\bs{W}}_{ji}}{\bs{u}_{ji}} - \bs{\tau}_i\dpar{\overline{\bs{W}}_{ij}}{\bs{u}_{ij}}\right] &\\
&\quad-\frac{1}{2}\sum_{j\in\mathcal{N}_i}\left(\frac{1}{T_i}+\frac{1}{T_j}\right)\left[\frac{A_{ij}^2}{2}\bs{v}_{ij}+\left(\frac{A_{ij}^2}{2}+\frac{B_{ij}^2-A_{ij}^2}{D}\right)\bs{e}_{ij}\cdot\bs{v}_{ij}\bs{e}_{ij}\right]dt &\\
&\quad+\frac{k_B}{2}\sum_{j\in\mathcal{N}_i}\left(\frac{1}{T_iC_i}+\frac{1}{T_jC_j}\right)\left[\frac{A_{ij}^2}{2}\bs{v}_{ij}+\left(\frac{A_{ij}^2}{2}+\frac{B_{ij}^2-A_{ij}^2}{D}\right)\bs{e}_{ij}\cdot\bs{v}_{ij}\bs{e}_{ij}\right]dt \label{eq:v_eq}&\\
&\quad-\frac{k_B}{2}\sum_{j\in\mathcal{N}_i}\frac{1}{C_i}\left[A_{ij}\dpar{A_{ij}}{T_i}\bs{v}_{ij}+\left(A_{ij}\dpar{A_{ij}}{T_i}+\frac{2}{D}B_{ij}\dpar{B_{ij}}{T_i} - \frac{2}{D}A_{ij}\dpar{A_{ij}}{T_i}\right)\bs{e}_{ij}\cdot\bs{v}_{ij}\bs{e}_{ij}\right]dt&\\
&\quad-\frac{k_B}{2}\sum_{j\in\mathcal{N}_i}\frac{1}{C_j}\left[A_{ij}\dpar{A_{ij}}{T_j}\bs{v}_{ij}+\left(A_{ij}\dpar{A_{ij}}{T_j}+\frac{2}{D}B_{ij}\dpar{B_{ij}}{T_j} - \frac{2}{D}A_{ij}\dpar{A_{ij}}{T_j}\right)\bs{e}_{ij}\cdot\bs{v}_{ij}\bs{e}_{ij}\right]dt&\\
&\quad+\sqrt{2k_B}\sum_{j\in\mathcal{N}_i}\left[A_{ij}d\bs{\overline{W}}_{ij}+B_{ij}\frac{1}{D}\tr(d\bs{W}_{ij})\right]\bs{e}_{ij}, &\\
T_idS_i&=\frac{1}{4}\sum_{j\in\mathcal{N}_i}\left(\frac{1}{T_i}+\frac{1}{T_j}\right)\left[\frac{A_{ij}^2}{2}\bs{v}_{ij}^2+\left(\frac{A_{ij}^2}{2}+\frac{B_{ij}^2-A_{ij}^2}{D}\right)(\bs{v}_{ij}\cdot\bs{e}_{ij})^2\right]dt+\sum_{j\in\mathcal{N}_i}\left(\frac{1}{T_i}-\frac{1}{T_j}\right)C_{ij}^2dt &\\
&\quad-\frac{k_B}{m}\sum_{j\in\mathcal{N}_i}\left[(D+1)\frac{A_{ij}^2}{2}+\frac{B_{ij}^2-A_{ij}^2}{D}\right]dt &\\
&\quad-\frac{k_B}{4}\sum_{j\in\mathcal{N}_i}\left(\frac{2}{T_iC_i}+\frac{1}{T_jC_j}\right)\left[\frac{A_{ij}^2}{2}\bs{v}_{ij}^2+\left(\frac{A_{ij}^2}{2}+\frac{B_{ij}^2-A_{ij}^2}{D}\right)(\bs{v}_{ij}\cdot\bs{e}_{ij})^2\right]dt-k_B\sum_{j\in\mathcal{N}_i}\left(\frac{2}{T_iC_i}-\frac{1}{T_jC_j}\right)C_{ij}^2dt &\\
&\quad+\frac{k_B}{4}\sum_{j\in\mathcal{N}_i}\frac{1}{C_i}\left[A_{ij}\dpar{A_{ij}}{T_i}\bs{v}_{ij}^2+\left(A_{ij}\dpar{A_{ij}}{T_i}+\frac{2}{D}B_{ij}\dpar{B_{ij}}{T_i} - \frac{2}{D}A_{ij}\dpar{A_{ij}}{T_i}\right)(\bs{v}_{ij}\cdot\bs{e}_{ij})^2\right]dt \label{eq:S_eq}&\\
&\quad+\frac{k_B}{4}\sum_{j\in\mathcal{N}_i}\frac{1}{C_j}\left[A_{ij}\dpar{A_{ij}}{T_j}\bs{v}_{ij}^2+\left(A_{ij}\dpar{A_{ij}}{T_j}+\frac{2}{D}B_{ij}\dpar{B_{ij}}{T_j} - \frac{2}{D}A_{ij}\dpar{A_{ij}}{T_j}\right)(\bs{v}_{ij}\cdot\bs{e}_{ij})^2\right]dt &\\
&\quad+2k_B\sum_{j\in\mathcal{N}_i}\left(\frac{1}{C_i}C_{ij}\dpar{C_{ij}}{T_i} - \frac{1}{C_j}C_{ij}\dpar{C_{ij}}{T_j}\right)dt &\\
&\quad-\frac{\sqrt{2k_B}}{2}\sum_{j\in\mathcal{N}_i}\left[A_{ij}d\bs{\overline{W}}_{ij}+B_{ij}\frac{1}{D}\tr(d\bs{W}_{ij})\right]:\bs{e}_{ij}\bs{v}_{ij}+\sqrt{2k_B}\sum_{j\in\mathcal{N}_i} C_{ij}dV_{ij}.
\end{flalign}

In the next section we prove that the proposed methodology fulfills the required thermodynamics and conservation laws.

\section{Conservation laws}\label{app:proofs}

\begin{proposition}\label{prop:momentum}
The metriplectic equation in \myeqref{eq:GENERIC_discrete_part} satisfies momentum conservation $d\bs{P}=\bs{0}$.
\end{proposition}
\begin{proof}

This is a direct consequence of the skew-symmetry of the particle pair forces over index swapping. The total deterministic momentum over all particles is $\bs{P}=\sum_im\bs{v}_i$.  By using Itô's lemma, we have that:
\begin{equation}
d\bs{P}=\dpar{\bs{P}}{\bs{x}}\cdot d\bs{x}+\frac{1}{2}\frac{\partial^2 \bs{P}}{\partial\bs{x}^2}:d\tilde{\bs{x}}d\tilde{\bs{x}}^T,
\end{equation}
which is equivalent, by using the fluctuation dissipation theorem, to 
\begin{equation}
d\bs{P}=\dpar{\bs{P}}{\bs{x}}\cdot d\bs{x}+k_B\frac{\partial^2 \bs{P}}{\partial\bs{x}^2}:\bs{M}=\dpar{\bs{P}}{\bs{x}}\cdot\left[\left(\bs{L}\dpar{E}{\bs{x}} + \bs{M}\dpar{S}{\bs{x}} + k_B\dpar{}{\bs{x}}\cdot\bs{M}\right)dt + d\tilde{\bs{x}}\right]+k_B\frac{\partial^2 \bs{P}}{\partial\bs{x}^2}:\bs{M}.
\end{equation}

We now show each summand is zero. The first term, by expanding the gradient of the internal energy from \myeqref{eq:gradU}:

\begin{equation}
\dpar{\bs{P}}{\bs{x}}\cdot\bs{L}\dpar{E}{\bs{x}}=\sum_i
\left[\bs{0},m\bs{1},0 \right]\frac{1}{m}
\begin{bmatrix}
m\bs{v}_i \\
-\dpar{U}{\bs{r}_i} \\
0 
\end{bmatrix} = \sum_i -\dpar{U}{\bs{r}_i} = \sum_i \sum_{j\in\mathcal{N}_i}\left[\frac{P_j}{d_j^2} \dpar{W_{ji}}{\bs{r}_{ji}}-\frac{P_i}{d_i^2} \dpar{W_{ij}}{\bs{r}_{ij}}+\bs{\tau}_j:\dpar{\overline{\bs{W}}_{ji}}{\bs{u}_{ji}} - \bs{\tau}_i:\dpar{\overline{\bs{W}}_{ij}}{\bs{u}_{ij}}\right].
\end{equation}

As the particle interactions are defined within a fixed neighbourhood $\mathcal{N}_i$ for each particle, we can split the set of interactions in three subsets: $\{(i,j)\}=\{(i',j')\}\cup\{(j',i')\}\cup\{(i',i')\}$. This means that every particle interaction $(i',j')$ has a mirror interaction with swapped indices $(j',i')$ for $i'\neq j'$, as well as self interactions of a particle with itself $(i',i')$ when $i'=j'$. The position vectors are $\bs{r}_{i'j'}=-\bs{r}_{j'i'}$ and $\bs{r}_{i'i'}=\bs{0}$ respectively. By applying the index swapping to the derivative and the condition of a spherical symmetric kernel $\bs{\nabla}W(\bs{0})=\bs{0}$, we get the following split sum:
\begin{equation}
\begin{split}
\dpar{\bs{P}}{\bs{x}}\cdot\bs{L}\dpar{E}{\bs{x}} &= \sum_{i',j'}\left[ \frac{P_{j'}}{d_{j'}^2} \dpar{W_{j'i'}}{\bs{r}_{j'i'}} - \frac{P_{i'}}{d_{i'}^2} \dpar{W_{i'j'}}{\bs{r}_{i'j'}}\right] + \sum_{j',i'}\left[ \frac{P_{i'}}{d_{i'}^2} \dpar{W_{i'j'}}{\bs{r}_{i'j'}}-\frac{P_{j'}}{d_{j'}^2} \dpar{W_{j'i'}}{\bs{r}_{j'i'}}\right] + \sum_{i',i'}2\frac{P_{i'}}{d_{i'}^2} \dpar{W_{i'i'}}{\bs{r}_{i'i'}} \\
&\quad + \sum_{i'j'}\left[\bs{\tau}_{j'}:\dpar{\overline{\bs{W}}_{j'i'}}{\bs{u}_{j'i'}} - \bs{\tau}_{i'}:\dpar{\overline{\bs{W}}_{i'j'}}{\bs{u}_{i'j'}}\right] + \sum_{j'i'}\left[\bs{\tau}_{i'}:\dpar{\overline{\bs{W}}_{i'j'}}{\bs{u}_{i'j'}} - \bs{\tau}_{j'}:\dpar{\overline{\bs{W}}_{j'i'}}{\bs{u}_{j'i'}}\right] + \bs{0} = \bs{0}.
\end{split}
\end{equation}

We can gather the divergence and second order term into a single divergence term
\begin{equation}
\dpar{\bs{P}}{\bs{x}}\cdot k_B\dpar{}{\bs{x}}\cdot\bs{M}+k_B\frac{\partial^2 \bs{P}}{\partial\bs{x}^2}:\bs{M}=k_B\dpar{}{\bs{x}}\cdot\left(\bs{M}\dpar{\bs{P}}{\bs{x}}\right)=0.
\end{equation}

By decomposing $\bs{M}$ on its dyadic $d\bs{x}d\bs{x}^T=2k_B\bs{M}dt$ we get that the sufficient condition is that $\dpar{\bs{P}}{\bs{x}}\cdot d\tilde{\bs{x}}=\bs{0}$. By expanding and using the noise ansatz,
\begin{equation}
\dpar{\bs{P}}{\bs{x}}\cdot d\tilde{\bs{x}}=\sum_i
\left[\bs{0},m\bs{1},0 \right]
\begin{bmatrix}
\bs{0} \\
d\tilde{\bs{v}}_i \\
d\tilde{S}_i
\end{bmatrix}=\sum_i m d\tilde{\bs{v}}_i=\sum_i\sqrt{2k_B}\sum_{j\in\mathcal{N}_i}\left[A_{ij}d\bs{\overline{W}}_{ij}+B_{ij}\frac{1}{D}\tr(d\bs{W}_{ij})\right]\bs{e}_{ij}.
\end{equation}

We obtain the desired cancellation by a similar reasoning as before: pairwise interactions cancel out due to the symmetries of $A_{ij}$, $B_{ij}$ and $d\bs{W}_{ij}$ and skew-symmetry of the interdistance unit vectors $\bs{e}_{ij}=-\bs{e}_{ij}$ and $\bs{e}_{ii}=\bs{0}$. Thus, the fluctuating velocities do not contribute to the total momentum.
\begin{flalign}
\dpar{\bs{P}}{\bs{x}}\cdot d\tilde{\bs{x}}&=\sqrt{2k_B}\sum_{i',j'}\left[A_{i'j'}d\bs{\overline{W}}_{i'j'}+B_{i'j'}\frac{1}{D}\tr(d\bs{W}_{i'j'})\right]\bs{e}_{i'j'}+\sqrt{2k_B}\sum_{j',i'}\left[A_{j'i'}d\bs{\overline{W}}_{j'i'}+B_{j'i'}\frac{1}{D}\tr(d\bs{W}_{j'i'})\right]\bs{e}_{j'i'} &\\
&\quad +\sqrt{2k_B}\sum_{i',i'}\left[A_{i'i'}d\bs{\overline{W}}_{i'i'}+B_{i'i'}\frac{1}{D}\tr(d\bs{W}_{i'i'})\right]\bs{e}_{i'i'} &\\
&=\sqrt{2k_B}\sum_{i',j'}\left[A_{i'j'}d\bs{\overline{W}}_{i'j'}+B_{i'j'}\frac{1}{D}\tr(d\bs{W}_{i'j'})\right]\bs{e}_{i'j'}-\sqrt{2k_B}\sum_{i',j'}\left[A_{i'j'}d\bs{\overline{W}}_{i'j'}+B_{i'j'}\frac{1}{D}\tr(d\bs{W}_{i'j'})\right]\bs{e}_{i'j'}=\bs{0}. &\\
\end{flalign}

The second and fourth terms are $\bs{M}\dpar{\bs{P}}{\bs{x}}$ and $\dpar{\bs{P}}{\bs{x}}\cdot d\tilde{\bs{x}}$ respectively, already proven to be zero.

\end{proof}

\begin{proposition}\label{prop:energy_stochastic}
\myeqref{eq:GENERIC_discrete_part} satisfies energy conservation $dE=0$ if $\bs{M}\bs{\nabla} E=\bs{0}$ and $d\tilde{\bs{x}}\cdot\bs{\nabla}E=0$.
\end{proposition}
\begin{proof}
By Itô's lemma, we have that
\begin{equation}
dE=\dpar{E}{\bs{x}}\cdot d\bs{x}+\frac{1}{2}\frac{\partial^2 E}{\partial\bs{x}^2}:d\tilde{\bs{x}}d\tilde{\bs{x}}^T,
\end{equation}
which is equivalent, by using the fluctuation dissipation theorem, to 
\begin{equation}
dE=\dpar{E}{\bs{x}}\cdot d\bs{x}+k_B\frac{\partial^2 E}{\partial\bs{x}^2}:\bs{M}=\dpar{E}{\bs{x}}\cdot\left[\left(\bs{L}\dpar{E}{\bs{x}} + \bs{M}\dpar{S}{\bs{x}} + k_B\dpar{}{\bs{x}}\cdot\bs{M}\right)dt + d\tilde{\bs{x}}\right]+k_B\frac{\partial^2 E}{\partial\bs{x}^2}:\bs{M}.
\end{equation}
We now show each summand is zero. 

Due to the skew-symmetry of $\bs{L}$ we have that the first summand is
\begin{equation}
\dpar{E}{\bs{x}}\cdot\bs{L}\dpar{E}{\bs{x}}=0.
\end{equation}

For the second, following the symmetric positive semi-definiteness and degeneracy condition for $\bs{M}$
\begin{equation}
\dpar{E}{\bs{x}}\cdot\bs{M}\dpar{S}{\bs{x}}=\dpar{S}{\bs{x}}\cdot\bs{M}\dpar{E}{\bs{x}}=0.
\end{equation}

We can gather the divergence and second order term into a single divergence which can be cancelled via the degeneracy condition
\begin{equation}
\dpar{E}{\bs{x}}\cdot k_B\dpar{}{\bs{x}}\cdot\bs{M}+k_B\frac{\partial^2 E}{\partial\bs{x}^2}:\bs{M}=k_B\dpar{}{\bs{x}}\cdot\left(\bs{M}\dpar{E}{\bs{x}}\right)=0.
\end{equation}
The final stochastic differential term is zero by assumption:
\begin{equation}
\dpar{E}{\bs{x}}\cdot d\tilde{\bs{x}} = d\tilde{\bs{x}}\cdot \dpar{E}{\bs{x}} = 0.
\end{equation}

\end{proof}

\begin{proposition}\label{prop:degeneracy}
The metriplectic equation in \myeqref{eq:GENERIC_discrete_part} satisfies the degeneracy conditions $\bs{L}\nabla S=\bs{M}\nabla E = \bs{0}$ and $d\tilde{\bs{x}}\cdot\bs{\nabla}E=0$, thus it satisfies energy conservation.
\end{proposition}
\begin{proof}

The first degeneracy condition is satisfied trivially by the definition of the entropy as a state variable and the zero block structure of the Poisson matrix,
\begin{equation}
\bs{L}\dpar{S}{\bs{x}}=
\sum_{j\in\mathcal{N}_i}\frac{1}{m}
\begin{bmatrix}
\bs{0} & -\bs{I}\delta_{ij} & \bs{0}\\
\bs{I}\delta_{ij} & \bs{0} & \bs{0}\\
\bs{0} & \bs{0} & \bs{0}
\end{bmatrix}
\begin{bmatrix}
\bs{0} \\
\bs{0} \\
1
\end{bmatrix}=\bs{0}.
\end{equation}

For the second degeneracy, we need to use the factorization of the friction matrix into the dyadic of the noise differentials $d\bs{x}d\bs{x}^T=2k_B\bs{M}dt$ in the following way,
\begin{equation}
d\tilde{\bs{x}}d\tilde{\bs{x}}^T\dpar{E}{\bs{x}}=2k_B\bs{M}dt\dpar{E}{\bs{x}}.
\end{equation}

Now we can prove that the left term is equal to zero, which also holds for the right term. Considering that $2k_Bdt\neq 0$, it implies that the second degeneracy is zero. By applying the definition of $\bs{\nabla}E$ and expanding the relative velocity $\bs{v}_{ij}=\bs{v}_i-\bs{v}_j$ we get
\begin{flalign}
d\tilde{\bs{x}}\cdot\dpar{E}{\bs{x}} &= \sum_i
\left[\bs{0},d\tilde{\bs{v}}_i,d\tilde{S}_i \right]
\begin{bmatrix}
-\dpar{U}{\bs{r}_i} \\
m\bs{v}_i \\
T_i
\end{bmatrix} 
= \sum_i (m\bs{v}_i\cdot d\tilde{\bs{v}}_i + T_i d\tilde{S}_i)
=\sum_i \bs{v}_i\cdot\sqrt{2k_B}\sum_{j\in\mathcal{N}_i}[A_{ij}d\bs{\overline{W}}_{ij}+B_{ij}\frac{1}{D}\tr(d\bs{W}_{ij})]\bs{e}_{ij} &\\
&\quad +\sum_i-\frac{\sqrt{2k_B}}{2}\sum_{j\in\mathcal{N}_i}\left[A_{ij}d\bs{\overline{W}}_{ij}+B_{ij}\frac{1}{D}\tr(d\bs{W}_{ij})\right]:\bs{e}_{ij}\bs{v}_{ij} +\sum_i\sqrt{2k_B}\sum_{j\in\mathcal{N}_i} C_{ij}dV_{ij} &\\
&=\sqrt{2k_B}\sum_{i,j}\left[A_{ij}d\bs{\overline{W}}_{ij}+B_{ij}\frac{1}{D}\tr(d\bs{W}_{ij})\right]\bs{e}_{ij}\cdot\bs{v}_i-\frac{\sqrt{2k_B}}{2}\sum_{i,j}\left[A_{ij}d\bs{\overline{W}}_{ij}+B_{ij}\frac{1}{D}\tr(d\bs{W}_{ij})\right]\bs{e}_{ij}\cdot\bs{v}_i &\\
&\quad +\frac{\sqrt{2k_B}}{2}\sum_{i,j}\left[A_{ij}d\bs{\overline{W}}_{ij}+B_{ij}\frac{1}{D}\tr(d\bs{W}_{ij})\right]\bs{e}_{ij}\cdot\bs{v}_j &\\
&=\sqrt{2k_B}\sum_{i,j}\left[A_{ij}d\bs{\overline{W}}_{ij}+B_{ij}\frac{1}{D}\tr(d\bs{W}_{ij})\right]\bs{e}_{ij}\cdot\bs{v}_i-\sqrt{2k_B}\sum_{i,j}\left[A_{ij}d\bs{\overline{W}}_{ij}+B_{ij}\frac{1}{D}\tr(d\bs{W}_{ij})\right]\bs{e}_{ij}\cdot\bs{v}_i=0 .
\end{flalign}

We have used the same index swapping cancellations between pairwise particles as in \mypropref{prop:momentum}. The independent term vanishes due to the skew-symmetry of $dV_{ij}$ and symmetry of $C_{ij}$, which makes a cancellation for each pair of particles. Similarly, we have performed an index swap by exploiting the symmetries of $A_{ij}$, $B_{ij}$, $C_{ij}$ and $d\bs{W}_{ij}$.

\end{proof}

\begin{proposition}\label{prop:3tensor}
The degeneracies $\bs{M}\bs{\nabla}E=\bs{0}$ and $d\tilde{\bs{x}}\cdot\bs{\nabla}E=0$ are satisfied if $Q_{ji}=C_{ijk}\dpar{E}{\bs{x}_k}$ where $C_{ijk}$ is a 3-tensor with the following skew-symmetry $C_{ijk}=-C_{ikj}$.
\end{proposition}
\begin{proof}
The noise vector is described by $d\tilde{\bs{x}} = \bs{Q} d\bs{W}$ where $d\bs{W}$ denotes a vector of independent and identically distributed (IID) increments of a Wiener process. From the degeneracy $d\tilde{\bs{x}} \cdot \bs{\nabla}E = d\bs{W}^T\bs{Q}^T\bs{\nabla}E=\bs{0}$ we obtain the constraint that $\bs{Q}^T\bs{\nabla}E=\bs{0}$ implies $d\tilde{\bs{x}} \cdot \bs{\nabla}E=0$ for any sample $d\bs{W}$. The other degeneracy condition follows immediately $\bs{M}\bs{\nabla}E=\bs{Q}\bs{Q}^T\bs{\nabla}E=\bs{0}$.

Now we proof the condition $\bs{Q}^T\bs{\nabla}E=\bs{0}$. By using the 3-tensor skew-symmetric parametrization we get:
\begin{equation}
Q_{ji}\dpar{E}{x_j}=C_{ijk}\dpar{E}{x_k}\dpar{E}{x_j}=C_{ikj}\dpar{E}{x_j}\dpar{E}{x_k}=-C_{ijk}\dpar{E}{x_k}\dpar{E}{x_j}=-Q_{ji}\dpar{E}{x_j}.
\end{equation}

By rearranging the first and last terms of the equality, we have
\begin{equation}
Q_{ji}\dpar{E}{x_j}+Q_{ji}\dpar{E}{x_j}=2Q_{ji}\dpar{E}{x_j}=0\rightarrow Q_{ji}\dpar{E}{x_j}=0.
\end{equation}

\end{proof}

\begin{remark}
The proposed parametrization can be constructed in several ways, designing specific parameters $\theta$ in increasing level of complexity and respecting the required skew-symmetry:
\begin{itemize}
    \item Nodal parameters + edge skew-symmetry: $C_{ijk}=\theta_i(\theta'_{jk}-\theta'_{kj})$.
    \item Edge parameters + edge skew-symmetry: $C_{ijk}=\theta_{ij}(\theta'_{jk}-\theta'_{kj})$.
    \item General tensor skew-symmetry: $C_{ijk}=\theta_{ijk} - \theta_{ikj}$.
\end{itemize}
\end{remark}

\begin{remark}
This parametrization can also be extended in case there are more than one entropic variables: $S_1,...,S_N$ by using a higher rank tensor with similar symmetries $C_{ijk_1...k_a...k_N}=-C_{ik_ak_1...j...k_N}$ for all extra indices and defining $Q_{ji}=C_{ijk_1...k_N}\dpar{E}{\bs{x}_{k_1}} ... \dpar{E}{\bs{x}_{k_N}}$.
\end{remark}

\section{Derivation of log-likelihood} \label{app:NLL}

Let $\mathcal{D}=\{(\bs{x}_{\text{rev}}^t,\bs{x}_{\text{rev}}^{t+1})\}_{t=0}^{N_\text{train}}$ be a dataset consisting of position-velocity pairs. To train the system, we construct a joint probability distribution $p(\bs{x}_{\text{rev}}^{t+1},\bs{x}_{\text{irr}}^{t+1}|\bs{x}_{\text{rev}}^{t},\bs{x}_{\text{irr}}^{t})$ corresponding to a numerical integration of \myeqref{eq:GENERIC_discrete_part} and train the architectures with maximum likelihood. Recall that for a general drift-diffusion SDE
\begin{equation}
    d\bs{x} = f(\bs{x},t) dt + g(\bs{x},t) d\bs{W},
\end{equation}
application of the Euler-Maruyama integrator provides a Gaussian increment
\begin{equation}\label{eq:eulermaruyama}
\bs{x}^{t+1}|\bs{x}^t \sim \mathcal{N}(\bs{\mu}^{t+1} = \bs{x}^t + f(\bs{x}^t,t)\Delta t,\bs{\Sigma}^{t+1} = g(\bs{x}^t,t)^2\Delta t).
\end{equation}
Under a Markovian assumption, we can thus define a joint distribution for the complete time series and derive the negative log-likelihood.
\begin{equation}
    \text{NLL}(x^1,...,x^T) = - \sum_{t=1}^T \log p(x^t|x^{t-1}).
\end{equation}
To obtain robust long-term rollouts, we specialize this candidate log-likelihood function to \myeqref{eq:GENERIC_discrete_part} and modify the transition probabilities slightly to develop a semi-implicit Euler scheme. We denote all trainable parameters via
$$\Theta=\{\theta_\mathcal{V},\theta_U,\theta_A,\theta_B,\theta_C,\theta_S,k_B,m\}$$ and define the maximum log-likelihood objective
\begin{equation}
\Theta^*=\arg\min_{\Theta} \text{NLL}(x^1,...,x^T).
\end{equation}
To modify \myeqref{eq:eulermaruyama}, we identify a transition probability $p(x^t|x^{t-1})$ consistent with the update rule.
\begin{equation}
\begin{split}
\bs{v}^{t+1}&=\bs{v}^t+\frac{d\bs{v}^t}{dt}\Delta t + \Delta \tilde{\bs{v}}^t, \\
\bs{r}^{t+1}&=\bs{r}^t+\bs{v}^{t+1}\Delta t, \\
S^{t+1}&=S^t+\frac{dS^t}{dt}\Delta t + \Delta \tilde{S}^t.
\end{split}
\end{equation}

For small $\Delta t$, this is provided by the multivariate normal distribution $\bs{x}^{t+1}|\bs{x}^t\sim N(\bs{\mu}^{t+1},\bs{\Sigma}^{t+1})$ with mean and covariance
\begin{equation}
\bs{\mu}^{t+1}=
\begin{bmatrix}
\bs{r}^t+\bs{v}^t\Delta t+\frac{d\bs{v}^t}{dt}\Delta t^2 \\
\bs{v}^t+\frac{d\bs{v}^t}{dt}\Delta t \\
S^t+\frac{dS^t}{dt}\Delta t
\end{bmatrix},\quad \bs{\Sigma}^{t+1}=\begin{bmatrix}
\Delta\tilde{\bs{v}} \Delta\tilde{\bs{v}}^T\Delta t^2 & \Delta\tilde{\bs{v}} \Delta\tilde{\bs{v}}^T\Delta t &  \Delta\tilde{\bs{v}} \Delta\tilde{S}\Delta t \\
\Delta\tilde{\bs{v}} \Delta\tilde{\bs{v}}^T\Delta t & \Delta\tilde{\bs{v}} \Delta\tilde{\bs{v}}^T &  \Delta\tilde{\bs{v}} \Delta\tilde{S}\\
\Delta\tilde{S} \Delta\tilde{\bs{v}}^T\Delta t & \Delta\tilde{S} \Delta\tilde{\bs{v}}^T &  \Delta\tilde{S} \Delta\tilde{S}
\end{bmatrix}.
\end{equation}

Now we can define the loss function to be the standard negative log-likelihood of the observations as
\begin{equation}
\mathcal{L}=\sum_{t=0}^{N_\text{train}}\Big[\frac{1}{2}\ln|\bs{\Sigma}^{t+1}|+\frac{1}{2}(\bs{x}^{t+1}-\bs{\mu}^{t+1})^T\left(\bs{\Sigma}^{t+1}\right)^{-1}(\bs{x}^{t+1}-\bs{\mu}^{t+1})\Big].
\end{equation}

There are two problems with this approach. First, the computation of $\bs{\Sigma}$ involves the explicit construction of the global $\bs{M}$ matrix which has a space complexity scaling of $O(N^2)$, so it is very expensive in terms of memory usage and computation time for the determinant and inverse operations. We can overcome this problem by using only the marginal covariances $\bs{\Sigma}_{ii}$ of each particle, which scales linearly with the number of particles $O(N)$. Second, the covariance matrix $\bs{\Sigma}$ is singular because the position has a deterministic relationship with the velocity up to a linear scaling of $\Delta t$. We overcome this by also marginalizing over the dependent position variable, which solves the rank deficiency.

Now we can split the problem as a particle-wise problem with the following statistical properties:
\begin{equation}\label{eq:mu_sigma}
\bs{\mu}^{t+1}_i=
\begin{bmatrix}
\bs{v}_i^t+\frac{d\bs{v}_i^t}{dt}\Delta t \\
S_i^t+\frac{dS_i^t}{dt}\Delta t
\end{bmatrix},\quad \bs{\Sigma}^{t+1}_{ii}=\begin{bmatrix}
\Delta\tilde{\bs{v}}_i \Delta\tilde{\bs{v}}_i^T &  \Delta\tilde{\bs{v}}_i \Delta\tilde{S}_i\\
\Delta\tilde{S}_i \Delta\tilde{\bs{v}}^T_i &  \Delta\tilde{S}_i\Delta\tilde{S}_i
\end{bmatrix}.
\end{equation}
Thus the final loss function is defined as
\begin{equation}\label{eq:loss}
\mathcal{L}=\frac{1}{N_\text{train}}\sum_{t=0}^{N_\text{train}}\frac{1}{N}\sum_{i=0}^{N}\Big[\frac{1}{2}\ln|\bs{\Sigma}_{ii}^{t+1}|+\frac{1}{2}(\bs{x}_i^{t+1}-\bs{\mu}_i^{t+1})^T\left(\bs{\Sigma}_{ii}^{t+1}\right)^{-1}(\bs{x}_i^{t+1}-\bs{\mu}_i^{t+1})\Big].
\end{equation}

\subsection{Computation of marginal covariances} \label{app:cov}

For the marginal covariances, we need to get the $i=j$ components of the complete covariance matrix $\bs{M}$, whose components were already computed in the previous section. By the following properties $\delta_{ii}=1$, $\bs{e}_{ii}=\bs{v}_{ii}=0$ and the fact the the summation over $j\in\mathcal{N}_i$ only contains nodes different than $i$, we obtain the following contributions:

\begin{flalign}
m^2\frac{ d\tilde{\bs{v}}_i d\tilde{\bs{v}}_i^T}{2k_B}&=\delta_{ii}\sum_{k\in\mathcal{N}_i}\left[\frac{A_{ik}^2}{2}\left(\bs{I} + \bs{e}_{ik}\bs{e}_{ik}^T\right)+\frac{B_{ik}^2-A_{ik}^2}{D}\bs{e}_{ik}\bs{e}_{ik}^T\right]dt-\left[\frac{A_{ii}^2}{2}\left(\bs{I} + \bs{e}_{ii}\bs{e}_{ii}^T\right)+\frac{B_{ii}^2-A_{ii}^2}{D}\bs{e}_{ii}\bs{e}_{ii}^T\right]dt &\\
&=\sum_{j\in\mathcal{N}_i}\left[\frac{A_{ij}^2}{2}\bs{I}+\left(\frac{A_{ij}^2}{2}+\frac{B_{ij}^2-A_{ij}^2}{D}\right)\bs{e}_{ij}\bs{e}_{ij}^T\right]dt. &\\
\end{flalign}

\begin{flalign}
mT_i\frac{ d\tilde{\bs{v}}_i d\tilde{S}_i}{2k_B}&=-\delta_{ii}\sum_{k\in\mathcal{N}_i}\left[\frac{A_{ik}^2}{2}\left(\frac{\bs{v}_{ik}}{2}+\bs{e}_{ik}\cdot\frac{\bs{v}_{ik}}{2}\bs{e}_{ik}\right)+\frac{B_{ik}^2-A_{ik}^2}{D}\bs{e}_{ik}\bs{e}_{ik}\cdot\frac{\bs{v}_{ik}}{2}\right]dt &\\
&\quad-\left[\frac{A_{ii}^2}{2}\left(\frac{\bs{v}_{ii}}{2}+\bs{e}_{ii}\cdot\frac{\bs{v}_{ii}}{2}\bs{e}_{ii}\right)+\frac{B_{ii}^2-A_{ii}^2}{D}\bs{e}_{ii}\bs{e}_{ii}\cdot\frac{\bs{v}_{ii}}{2}\right]dt &\\
&\quad=-\frac{1}{2}\sum_{j\in\mathcal{N}_i}\left[\frac{A_{ij}^2}{2}\bs{v}_{ij}+\left(\frac{A_{ij}^2}{2}+\frac{B_{ij}^2-A_{ij}^2}{D}\right)\bs{e}_{ij}\cdot\bs{v}_{ij}\bs{e}_{ij}\right]dt.&\\
\end{flalign}

\begin{flalign}
mT_i\frac{ d\tilde{S}_i d\tilde{\bs{v}}_i^T}{2k_B}&=-\delta_{ii}\sum_{k\in\mathcal{N}_i}\left[\frac{A_{ik}^2}{2}\left(\frac{\bs{v}_{ik}^T}{2}+\bs{e}_{ik}\cdot\frac{\bs{v}_{ik}}{2}\bs{e}_{ik}^T\right)+\frac{B_{ik}^2-A_{ik}^2}{D}\bs{e}_{ik}^T\bs{e}_{ik}\cdot\frac{\bs{v}_{ik}}{2}\right]dt &\\
&\quad+\left[\frac{A_{ii}^2}{2}\left(\frac{\bs{v}_{ii}^T}{2}+\bs{e}_{ii}\cdot\frac{\bs{v}_{ii}}{2}\bs{e}_{ii}^T\right)+\frac{B_{ii}^2-A_{ii}^2}{D}\bs{e}_{ii}^T\bs{e}_{ii}\cdot\frac{\bs{v}_{ii}}{2}\right]dt &\\
&\quad=-\frac{1}{2}\sum_{j\in\mathcal{N}_i}\left[\frac{A_{ij}^2}{2}\bs{v}_{ij}^T+\left(\frac{A_{ij}^2}{2}+\frac{B_{ij}^2-A_{ij}^2}{D}\right)\bs{e}_{ij}\cdot\bs{v}_{ij}\bs{e}_{ij}^T\right]dt.&\\
\end{flalign}

\begin{flalign}
T_i^2\frac{ d\tilde{S}_i d\tilde{S}_i}{2k_B}&=\delta_{ii}\sum_{k\in\mathcal{N}_i}\left(\frac{A_{ik}^2}{2}\left[\left(\frac{\bs{v}_{ik}}{2}\right)^2+\left(\bs{e}_{ik}\cdot\frac{\bs{v}_{ik}}{2}\right)^2\right]+\frac{B_{ik}^2-A_{ik}^2}{D}\left(\bs{e}_{ik}\cdot\frac{\bs{v}_{ik}}{2}\right)^2+C_{ik}^2\right)dt &\\
&\quad+\left(\frac{A_{ii}^2}{2}\left[\left(\frac{\bs{v}_{ii}}{2}\right)^2+\left(\bs{e}_{ii}\cdot\frac{\bs{v}_{ii}}{2}\right)^2\right]+\frac{B_{ii}^2-A_{ii}^2}{D}\left(\bs{e}_{ii}\cdot\frac{\bs{v}_{ii}}{2}\right)^2-C_{ii}^2\right)dt &\\
&\quad=\sum_{j\in\mathcal{N}_i}\left(\frac{A_{ij}^2}{2}\left[\left(\frac{\bs{v}_{ij}}{2}\right)^2+\left(\bs{e}_{ij}\cdot\frac{\bs{v}_{ij}}{2}\right)^2\right]+\frac{B_{ij}^2-A_{ij}^2}{D}\left(\bs{e}_{ij}\cdot\frac{\bs{v}_{ij}}{2}\right)^2+C_{ij}^2\right)dt&\\
&\quad=\frac{1}{4}\sum_{j\in\mathcal{N}_i}\left[\frac{A_{ij}^2}{2}\bs{v}_{ij}^2+\left(\frac{A_{ij}^2}{2}+\frac{B_{ij}^2-A_{ij}^2}{D}\right)\left(\bs{e}_{ij}\cdot\bs{v}_{ij}\right)^2\right]dt+\sum_{j\in\mathcal{N}_i}C_{ij}^2dt.&\\
\end{flalign}

Note that the resulting marginal covariance matrices are symmetric, as expected. Furthermore, as the complete covariance matrix $\bs{M}$ is positive semidefinite by construction, all its leading minors are also positive semidefinite matrices. Thus, all the marginal covariance matrices are valid.

\section{Constrained monotonic neural networks}\label{app:cmnn}

The purpose of Constrained Monotonic Neural Networks \cite{runje2023constrained} is to construct a universal approximator with a precondition of (partial) monotonicity and/or convexity of a multivariate, multidimensional function. The basic building block is a monotone constrained fully connected layer:
\begin{equation}
\bs{h}=\rho^{\bs{s}}(|\bs{W}^T|_{\bs{t}}\bs{x}+\bs{b})
\end{equation}
where $\bs{x}\in\mathbb{R}^n$ is the input vector, $\bs{h}\in\mathbb{R}^m$ is the output vector, $\bs{W}\in\mathbb{R}^{n\times m}$ and $\bs{b}\in\mathbb{R}^m$ are the weights and biases, $\rho^{\bs{s}}$ is the combined activation function, and $\bs{t}$ is the monotonicity indicator vector. This is a usual MLP architecture whose weights and activation functions are specifically chosen to ensure monotonicity and convexity of the output results. 

In particular, we need to model the internal energy $U_i=U(S_i,\mathcal{V}_i)$ scalar function which is increasing in entropy $S_i$ and decreasing in volume $\mathcal{V}_i$, both scalar quantities. Thus, the monotonicity indicator vector is $\bs{t}=[1,-1]$ and the weight matrix is modified by the operation $|\bs{W}^T|_{\bs{t}}$ which assigns the following sign change:
\begin{equation}
w_{ji}^{\prime} = \begin{cases}
|w_{ji}| & \text{if } t_i = 1 \\
-|w_{ji}| & \text{if } t_i = -1 \\
w_{ji} & \text{otherwise}
\end{cases}.
\end{equation}
For the subsequent hidden layers, monotonic dense units with the indicator vector $\bs{t}$ always being set to 1 are used in order to preserve monotonicity. Last, we enforce convexity with respect to both inputs by using a convex activation function $\rho^{\bs{s}}$. In our experiments, we have used a softmax activation. It is noteworthy that even with the imposed constraints on weights and activation functions, the network retains its universal approximation property. The complete proof is provided in \cite{runje2023constrained}.

\newpage
\section{Training and testing algorithms} \label{app:algo}

\begin{algorithm}
\begin{algorithmic}
\State \textbf{Input:} State variables: $(\bs{r}^t,\bs{v}^t)$
\For{each particle $i$} 
\State Compute neighbour particles $j\in\mathcal{N}_i$
\State Estimate entropy $S_i$:
\State \hskip 1.0em $S_i=\text{NN}(\bs{r}_{ij}, \bs{v}_{ij})$ 
\State Compute particle volume $\mathcal{V}_i$ and traceless strain tensor $\bar{\bs{\varepsilon}}_i$: 
\State \hskip 1.0em $d_i=\text{NN}(\bs{r}_{ij})$ 
\State \hskip 1.0em $\bar{\bs{\varepsilon}}_i=\text{NN}(\bs{r}_{ij},\bs{r}_{ij}^0)$ 
\State Compute internal energy $U_i$: 
\State \hskip 1.0em $U_i=\text{NN}(S_{i},\mathcal{V}_i,\bar{\bs{\varepsilon}}_i)$ 
\State \hskip 1.0em $P_i=-\dpar{U_i}{\mathcal{V}_i}$, $T_i=\dpar{U_i}{S_i}$, $\bs{\tau}_i=\dpar{U_i}{\bar{\bs{\varepsilon}}_i}$ 
\State Compute prediction statistics: 
\State \hskip 1.0em $A_{ij},B_{ij},C_{ij}=\text{NN}(\bs{r}_{ij},T_i, T_j)$ 
\State \hskip 1.0em Compute $\bs{\Sigma}_{ii}^{t+1}$ and $\bs{\mu}_i^{t+1}$ 
\EndFor
\State Compute negative log-likelihood loss $\mathcal{L}$ 
\end{algorithmic}
\caption{Training algorithm for the proposed method.}\label{alg:train}
\end{algorithm}

\begin{algorithm}
\begin{algorithmic}
\State \textbf{Input:} Initial conditions: $(\bs{r}^0,\bs{v}^0)$
\State Estimate initial entropy $S^0=\text{NN}(\bs{r}_{ij}^0, \bs{v}_{ij}^0)$ 
\For{each snapshot $t$} 
\For{each particle $i$} 
\State Compute neighbour particles $j\in\mathcal{N}_i$
\State Compute particle volume $\mathcal{V}_i$ and traceless strain tensor $\bar{\bs{\varepsilon}}_i$: 
\State \hskip 1.0em $d_i=\text{NN}(\bs{r}_{ij})$ 
\State \hskip 1.0em $\bar{\bs{\varepsilon}}_i=\text{NN}(\bs{r}_{ij},\bs{r}_{ij}^0)$ 
\State Compute internal energy $U_i$: 
\State \hskip 1.0em $U_i=\text{NN}(S_{i},\mathcal{V}_i,\bar{\bs{\varepsilon}}_i)$ 
\State \hskip 1.0em $P_i=-\dpar{U_i}{\mathcal{V}_i}$, $T_i=\dpar{U_i}{S_i}$, $\bs{\tau}_i=\dpar{U_i}{\bar{\bs{\varepsilon}}_i}$ 
\State Compute stochastic increments $d\tilde{\bs{x}}_i$: 
\State \hskip 1.0em $A_{ij},B_{ij},C_{ij}=\text{NN}(\bs{r}_{ij},T_i,T_j)$ 
\State \hskip 1.0em Random numbers $d\bs{W}_{ij}$ and $dV_{ij}$ 
\State \hskip 1.0em Compute $d\tilde{\bs{x}}_i=(\bs{0},d\tilde{\bs{v}}_i,d\tilde{S}_i)$ 
\EndFor
\State Solve dynamics $\bs{x}^{t+1}$ via GENERIC 
\State Update state vector: $\bs{x}^t\leftarrow\bs{x}^{t+1}$
\EndFor
\end{algorithmic}
\caption{Testing algorithm for the proposed method.}\label{alg:test}
\end{algorithm}

\section{Alternative architectures} \label{app:gnns}

\subsection{GNS}

Graph Network-based Simulator (GNS) \cite{sanchez2020learning} is a deep learning architecture specifically designed for capturing the dynamical behaviour of complex fluids. For GNS, we use the original deterministic acceleration prediction based on positions and velocities. We use the following encoder-processor-decoder message-passing structure:

\begin{itemize}
\item Encoder: Following the original paper, the nodal features are chosen to be the velocities of the particles $\bs{v}_i$, and the edge features the relative positions $\bs{r}_{ij}$ and distances $|\bs{r}_{ij}|$. Those are encoded in node $\bs{z}_i$ and edge $\bs{z}_{ij}$ latent vectors by using two networks:

\begin{equation}
\bs{z}_{ij}=\text{MLP}_e\left(\frac{\bs{r}_{ij}}{h}, \frac{|\bs{r}_{ij}|}{h}\right),\quad \bs{z}_i=\text{MLP}_v(\bs{v}_i).
\end{equation}

\item Processor: The latent nodes and edges are then processed with a message passing layer with mean aggregation function. This scheme is repeated $M$ times with residual connections in order for the information to transverse the complete graph. 
\begin{equation}
\bs{z}^{l+1}_{ij}=\bs{z}_{ij}^{l}+\text{MLP}^l_m(\bs{z}^{l}_i,\bs{z}^{l}_j,\bs{z}^{l}_{ij}),\quad \bs{z}^{l+1}_i=\bs{z}^{l}_i+\text{MLP}_v^l\left(\bs{z}^{l}_i,\frac{1}{|\mathcal{N}_i|}\sum_{j\in\mathcal{N}_i}\bs{z}^{l+1}_{ij}\right).
\end{equation}
\item Decoder: The decoder extracts the dynamic information from the latent nodes as an acceleration prediction:
\begin{equation}
\bs{a}_i=\text{MLP}_d(\bs{z}^{M}_i)dt.
\end{equation}
\end{itemize}

The dynamic equations are then the following:

\begin{equation}
\begin{split}
d\bs{r}_i&=\bs{v}_idt\\
d\bs{v}_i&=\text{GNS}(\bs{r}_i,\bs{v}_i)dt=\bs{a}_idt
\end{split}
\end{equation}

\subsection{GNS-SDE}

The original GNS paper does not account for stochastic effects. We have adapted the methodology to account for stochastic effects in a second architecture which we refer to as GNS-SDE \cite{dridi2021learning}. The idea is to have two independent GNS which acount for the drift and diffusion terms of a generic stochastic differential equation as follows:
\begin{equation}
\begin{split}
d\bs{r}_i&=\bs{v}_idt,\\
d\bs{v}_i&=\text{GNS}_{\text{drift}}(\bs{r}_i,\bs{v}_i) dt + \text{GNS}_{\text{diff}}(\bs{r}_i,\bs{v}_i) d\bs{w}
\end{split}
\end{equation}
where $\text{GNS}_{\text{drift}}(\bs{x})$ and $\text{GNS}_{\text{diff}}(\bs{x})$ are the drift and diffusion coefficients respectively and $d\bs{w}$ is a Wiener process with zero mean and unit variance.

\subsection{DPD}

As a traditional coarse-graining baseline, we use the standard Dissipative Particle Dynamics (DPD) force decomposition in conservative, dissipative and random forces between particles as:

\begin{equation}
\begin{split}
d\bs{r}_i&=\bs{v}_idt,\\
md\bs{v}_i&=\sum_{j\in\mathcal{N}_i}(\bs{F}_{ij}^C dt+\bs{F}_{ij}^D dt+\bs{F}_{ij}^R d\bs{w})
\end{split}
\end{equation}

where each force is computed using:

\begin{equation}
\begin{split}
\bs{F}_{ij}^C&=\alpha w^C(|\bs{r}_{ij}|)\bs{e}_{ij},\\
\bs{F}_{ij}^D&=-\gamma w^D(|\bs{r}_{ij}|)(\bs{e}_{ij}\cdot\bs{v}_{ij})\bs{e}_{ij},\\
\bs{F}_{ij}^R&=\sigma w^R(|\bs{r}_{ij}|)\bs{e}_{ij}.
\end{split}
\end{equation}

The conservative and random pairwise weight function is chosen to be the usual linear function of the normalized distance
\begin{equation}
w^C(|\bs{r}_{ij}|)=w^R(|\bs{r}_{ij}|)=1-\frac{|\bs{r}_{ij}|}{h}.
\end{equation}

The dissipative forces are fully determined by the fluctuation-dissipation theorem:
\begin{equation}
\begin{split}
\sigma^2&=2\gamma k_BT,\\
w^D(|\bs{r}_{ij}|)&=w^R(|\bs{r}_{ij}|)^2.
 \end{split}
\end{equation}

To sum up, this model has 4 trainable parameters: the force amplitudes $\alpha$ and $\sigma$, the particle mass $m$ and the thermal energy $k_BT$.

\section{Hyperparameters}\label{app:hyperparams}

\subsection{Datasets}

The datasets were generated using the following hyperparameters:

\begin{itemize}
\item Number of CG particles ($N$): Discretization of the domain in coarse-grained particles. 
\item Length of the domain ($L$): Length of the cubic box in 2D and 3D. For molecular dynamics examples, the periodic boundary conditions are applied.
\item Cutoff radius ($h$): The radius which determines the neighbourhood from which the interactions are active, following the minimum image convention for periodic boundary conditions.
\item Time step ($\Delta t$): Time step of the coearse-grained learnt simulation. In molecular dynamics examples, several orders of magnitude bigger than the original time scale. In solid mechanics example, it is determined by the CFL condition and the measured data avilability.

\end{itemize}

\begin{table}[h]
  \caption{Dataset hyperparameters.}
  \label{tab:dataset_hyperparameters}
  \centering  
\begin{tabular}{lcccc}
\toprule
& $N$ & $L$& $h$ & $\Delta t$ \\
\midrule
 \textsc{Ideal gas}       & 500 & 1.00 (3D periodic) & 0.20 & 5.0e-4 \\
 \textsc{Star polymer 11} & 1000 & 30.18 (3D periodic) & 4.75 & 2.5e-2 \\
 \textsc{Star polymer 51} & 2000 & 63.41 (3D periodic) & 7.61 & 4.0e-2 \\
 \textsc{Viscoelastic}    & 900 & 2 (2D) & 0.13 & 1.0e-2 \\
 \textsc{Colloids}          & 1135 & 2 (2D) & 0.11 & 5.0e-4 \\
\bottomrule
\end{tabular} 
\end{table}

\subsubsection{\textsc{Ideal gas}}

The first example is the simulation of $N=500$ coarse-grained particles in a 3D periodic box of length $L=1$. The fluid has a shear and bulk viscosity of $\eta=\zeta=0.1$ and thermal conductivity of $\kappa=0.1$. The equation of state is assumed to be an monoatomic ideal gas,
\begin{equation}
U(S,\mathcal{V})=\frac{3h_p^2 N_p^2}{4\pi m}\left(\frac{N_p}{\mathcal{V}}\right)^{\frac{2}{3}}\exp\left(\frac{2S}{3k_BN_p}-\frac{5}{3}\right),
\end{equation}
where $h_p$ is the Planck constant and $N_p$ the number of particles per fluid particle. We select units in which $N_ph_p=1$, $N_pk_B=10$ and $m=1$. The dataset was generated directly in the coarse-grained domain using the methodology described in \cite{espanol2003smoothed}.

\subsubsection{\textsc{Star polymer 11} and \textsc{Star polymer 51}}

The second example is a polymer melt in a 3D periodic box. Each polymer is composed of a core atom and 10 arms bonded by a FENE potential, whereas the interaction between polymers is modelled by a Lennard-Jones potential. Two different internal configurations are considered: 1 and 5 beads per arm, which results in less and higher coarse-graining levels. The fully resolved datasets are generated using LAMMPS \cite{plimpton1995fast} and then coarse-grained by computing the center of mass position and velocity of each polymer. 

\subsubsection{\textsc{Viscoelastic}}

The 2D solid square domain of length $L=2$ with a imposed shear displacement of $\Delta L_x=1$ in the X axis. The material is chosen to be a Mooney-Rivlin model with density $\rho=1000$ and  parameters $C_{10}=1.5\cdot 10^6$, $C_{01}=5\cdot 10^5$ and $D_1=10^{-7}$. A time-dependant viscoelastic behaviour was included as a two-term Prony series with shear relaxation modulus ratio and relaxation time of $G_1=0.5$, $\tau_1=0.2$ and $G_2=0.49$, $\tau_2=0.3$ respectively. The system is discretized in $N=900$ coarse graining particles. The reference solution was computed using the finite element solver Abaqus. 

\subsubsection{\textsc{Colloids}}

The last example is a real measured dataset of a jammed 2D monolayer of repulsive microspheres under cyclic shear forcing. The system consists of sulfate latex particles adsorbed at a water–decane interface, forming a soft jammed material due to electrostatic dipole-dipole repulsion. The material is sheared using an interfacial stress rheometer (ISR), where a magnetized needle applies oscillatory shear stress within a confined channel \cite{keim2014mechanical,galloway2020scaling}. High-resolution optical microscopy is used to capture the positions of each particle, processed with a computer vision particle tracking algorithm, and the velocities are obtained with finite differences. The system is discretized in $N=1135$ coarse graining particles. 

\subsection{Training}

The hyperparameters used to train the models are the following:
\begin{itemize}
\item Number of training snapshots ($N_{\text{train}}$): Time horizon considered for the training from the initial conditions. They are shuffled in $75\%$ train and $25\%$ validation split.
\item Number of extrapolation snapshots ($N_{\text{extrap}}$): Sequential rollout of the model from the initial conditions. In molecular dynamics, the model is extrapolated to $N_{\text{extrap}}=25N_{\text{train}}$ to reach the steady-state statistics. 
\item Learning rate ($l_r$): Initial learning rate for training. 
\item Number of epochs ($N_{\text{epoch}}$): Number of total training epochs.
\item  Number of hidden layers ($N_h$): Number of hidden layers of the MLPs. All the MLPs used are 2 hidden layers with SiLU activation functions \cite{ramachandran2017searching}.
\item \commentquercus{Number of parameters ($N_{\text{params}}$): Parameter count of the trained model.}.
\end{itemize}

\begin{table}[h]
  \caption{Training hyperparameters.}
  \label{tab:train_hyperparameters}
  \centering  
\begin{tabular}{lccccccc}
\toprule
 & $N_{\text{train}}$ & $N_{\text{extrap}}$ & $l_r$ & $N_{\text{epoch}}$ & $N_h$ & $N_{\text{params}}$ \\
\midrule
 \textsc{Ideal gas}       & 300 & 7.5e3 & 1.0e-2 & 5.0e3 & 50 & 16560 \\
 \textsc{Star polymer 11} & 400 & 1.0e4 & 1.0e-3 & 5.0e3 & 100 & 63110 \\
 \textsc{Star polymer 51} & 400 & 1.0e4 & 1.0e-3 & 1.0e4 & 50 & 16560 \\
 \textsc{Viscoelastic}    & 800 & 1000 & 1.0e-3 & 2.0e3 & 100 & 84517 \\
 \textsc{Colloids}          & 2000 & 3500 & 1.0e-3 & 2.0e3 & 100 & 84517 \\
\bottomrule
\end{tabular} 
\end{table}

\section{Additional figures}

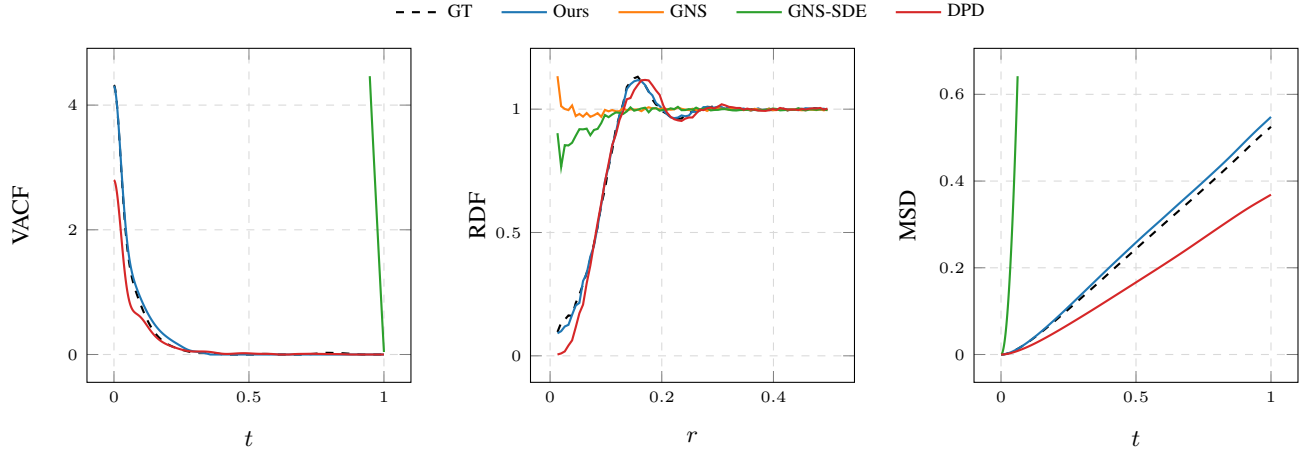
\begin{figure}[h]
\centering
\begin{tikzpicture}
\pgfplotsset{width=\textwidth, height=6cm}

\begin{axis}[
    xlabel=$t$,
    width=0.33\textwidth,
    ylabel=VACF,
    grid=major, 
    grid style={dashed,gray!30}, 
    tick label style={font=\tiny}, 
    font=\small,
    ylabel style={align=center, yshift=-1em}, 
    restrict y to domain*=0:4.5,
    legend style={font=\tiny}, 
]
\addplot [color=black, thick, dashed] table[x=t_vec, y=VACF_gt] {plots/taylor_green2/taylor_green.txt};
\addplot [color=tabblue, thick] table[x=t_vec, y=VACF_net] {plots/taylor_green2/taylor_green.txt};
\addplot [color=taborange, thick] table[x=t_vec, y expr={ ( \thisrow{VACF_net} <= 0 || \thisrow{VACF_net} >= 4.5 ) ? nan : \thisrow{VACF_net} },] {plots/taylor_green2/taylor_green_gns.txt};
\addplot [color=tabgreen, thick] table[x=t_vec, y expr={ ( \thisrow{VACF_net} <= 0 || \thisrow{VACF_net} >= 4.5 ) ? nan : \thisrow{VACF_net} },] {plots/taylor_green2/taylor_green_gns_sde.txt};
\addplot [color=tabred, thick] table[x=t_vec, y=VACF_net] {plots/taylor_green2/taylor_green_dpd.txt};

\end{axis}

\begin{axis}[xshift=0.33\textwidth,
    width=0.33\textwidth,
    xlabel=$r$,
    ylabel=RDF,
    grid=major, 
    grid style={dashed,gray!30}, 
    tick label style={font=\tiny}, 
    font=\small,
    ylabel style={align=center, yshift=-1.5em}, 
    legend columns=5,              
    transpose legend=false,
    legend style={
    at={(0.5,1.05)},  
    anchor=south,
    legend cell align=left,
    font=\scriptsize,
    draw=none,                 
    /tikz/every even column/.append style={column sep=0.3cm}
    }
]
\addplot [color=black, thick, dashed] table[x=r_RDF, y=RDF_gt] {plots/taylor_green2/taylor_green.txt};
\addplot [color=tabblue, thick] table[x=r_RDF, y=RDF_net] {plots/taylor_green2/taylor_green.txt};
\addplot [color=taborange, thick] table[x=r_RDF, y=RDF_net] {plots/taylor_green2/taylor_green_gns.txt};
\addplot [color=tabgreen, thick] table[x=r_RDF, y=RDF_net] {plots/taylor_green2/taylor_green_gns_sde.txt};
\addplot [color=tabred, thick] table[x=r_RDF, y=RDF_net] {plots/taylor_green2/taylor_green_dpd.txt};

\addlegendimage{black, thick, dashed}
\addlegendentry{GT}
\addlegendimage{tabblue, thick}
\addlegendentry{Ours}
\addlegendimage{taborange, thick}
\addlegendentry{GNS}
\addlegendimage{tabgreen, thick}
\addlegendentry{GNS-SDE}
\addlegendimage{tabred, thick}
\addlegendentry{DPD}

\end{axis}

\begin{axis}[xshift=0.66\textwidth,
    width=0.33\textwidth,
    xlabel=$t$,
    ylabel=MSD,
    grid=major, 
    grid style={dashed,gray!30}, 
    tick label style={font=\tiny}, 
    font=\small,
    ylabel style={align=center, yshift=-1em}, 
    restrict y to domain*=0:0.65,
]
\addplot [color=black, thick, dashed] table[x=t_vec, y=MSD_gt] {plots/taylor_green2/taylor_green.txt};
\addplot [color=tabblue, thick] table[x=t_vec, y=MSD_net] {plots/taylor_green2/taylor_green.txt};
\addplot [color=taborange, thick] table[x=t_vec, y expr={ ( \thisrow{MSD_net} <= 0 || \thisrow{MSD_net} >= 0.65 ) ? nan : \thisrow{MSD_net} },] {plots/taylor_green2/taylor_green_gns.txt};
\addplot [color=tabgreen, thick] table[x=t_vec, y expr={ ( \thisrow{MSD_net} <= 0 || \thisrow{MSD_net} >= 0.65 ) ? nan : \thisrow{MSD_net} },] {plots/taylor_green2/taylor_green_gns_sde.txt};
\addplot [color=tabred, thick] table[x=t_vec, y expr={ ( \thisrow{MSD_net} <= 0 || \thisrow{MSD_net} >= 0.65 ) ? nan : \thisrow{MSD_net} },] {plots/taylor_green2/taylor_green_dpd.txt};
\end{axis}
        
\end{tikzpicture}
\caption{\commentquercus{Correlation metrics for the \textsc{Taylor-Green} dataset trained with forcing conditions.}}
\label{fig:taylor_green_app}
\end{figure}

\begin{table}[h] 
  \caption{L2 relative error ($\downarrow$) of correlation metrics for the $\textsc{Ideal Gas}$ interpolation and extrapolation experiments and $\textsc{Star Polymer}$ examples.}
  \label{tab:results_app}
  \centering  
\begin{tabular}{llccc}
\toprule
\multicolumn{2}{c}{} & VACF & RDF & MSD \\
\midrule
\textsc{Taylor-Green} & GNS      & 1.09e4 & 3.27e-1 & 2.24e4 \\
                      & GNS-SDE  & 8.23e4 & 3.40e-1 & 1.70e5 \\
                      & DPD   & 3.25e-1 & 4.41e-2 & 2.66e-1 \\
                      & Ours     & \textbf{5.18e-2} & \textbf{9.82e-3} & \textbf{7.73e-2} \\       
\textsc{Self-diffusion} & GNS    & 1.02e4 & 3.15e-1 & 1.94e4 \\
                      & GNS-SDE  & 6.46e4 & 3.27e-1 & 1.23e5 \\
                      & DPD   & 4.52e-1 & 4.44e-2 & 4.65e-1 \\
                      & Ours     & \textbf{4.19e-2} & \textbf{1.59e-2} & \textbf{4.54e-2} \\ 
\textsc{Shear flow}   & GNS      & 7.59e5 & 3.43e-1 & 3.45e-1 \\
                      & GNS-SDE  & 1.47e5 & 2.94e-1 & 3.74e-1 \\
                      & DPD   & 2.18e-1 & 3.68e-2 & 4.64e-2 \\
                      & Ours     & \textbf{7.94e-2} & \textbf{1.12e-2} & \textbf{3.76e-2} \\    
\midrule           
\textsc{Star polymer 11} & GNS   & 1.25e6 & 4.29e-1 & 2.09e6 \\
                      & GNS-SDE  & 3.87e5 & 4.27e-1 & 6.45e5 \\
                      & DPD   & 6.66e-1 & 2.26e-1 & 9.51e-1 \\
                      & Ours     & \textbf{9.89e-2} & \textbf{1.99e-2} & \textbf{4.73e-2} \\ 
\textsc{Star polymer 51}   & GNS & 1.27e1 & 2.13e-1 & 2.12e1 \\
                      & GNS-SDE  & 7.97e6 & 2.23e-1 & 1.61e7 \\
                      & DPD   & 4.00e0 & 9.64e-2 & 7.63e0 \\
                      & Ours     & \textbf{1.41e-1} & \textbf{5.28e-2} & \textbf{5.32e-2} \\  
\bottomrule
\end{tabular} 

\end{table}

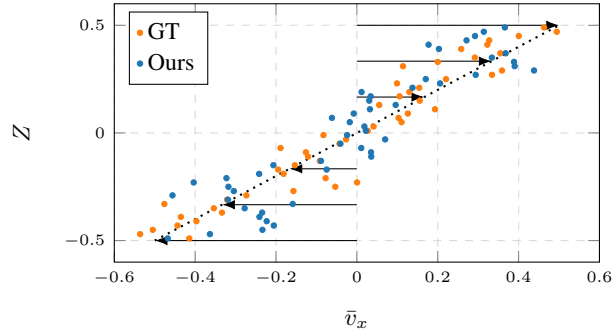
\begin{figure}
\centering
\usetikzlibrary{arrows.meta}
\begin{tikzpicture}
\begin{axis}[
    height=5cm,
    width=8cm,
    xlabel=$\bar{v}_x$,
    ylabel=$Z$,
    xmin=-0.6, xmax=0.6,
    ymin=-0.6, ymax=0.6,
    grid=major,
    grid style={dashed,gray!30}, 
    font=\small, 
    tick label style={font=\tiny},
    legend pos=north west,
    legend cell align={left},
]

\pgfplotstableread{plots/shear_flow/shear_flow.txt}\loadedtable

\addplot[only marks, mark=*, taborange, mark size=1pt] table [x=v_mean_gt, y=z_bin_centers] {\loadedtable};
\addlegendentry{GT}

\addplot[only marks, mark=*, tabblue, mark size=1pt] table [x=v_mean_net, y=z_bin_centers] {\loadedtable};
\addlegendentry{Ours}

\addplot[black, dotted, thick] coordinates {(-0.5,-0.5) (0.5,0.5)};

\addplot[-{Latex}] coordinates {(0, 0.5) (0.5, 0.5)};
\addplot[-{Latex}] coordinates {(0, -0.5) (-0.5, -0.5)};

\addplot[-{Latex}] coordinates {(0, 0.1666) (0.1666, 0.1666)};
\addplot[-{Latex}] coordinates {(0, -0.1666) (-0.1666, -0.1666)};

\addplot[-{Latex}] coordinates {(0, 0.3333) (0.3333, 0.3333)};
\addplot[-{Latex}] coordinates {(0, -0.3333) (-0.3333, -0.3333)};

\end{axis}
\end{tikzpicture}
\caption{Velocity profile of the \textsc{Shear flow} test dataset averaged over the X axis. The forced boundary condition is represented as a dotted line.}
\label{fig:shear_flow_app}
\end{figure}

\begin{figure}[h]
\centering
\input{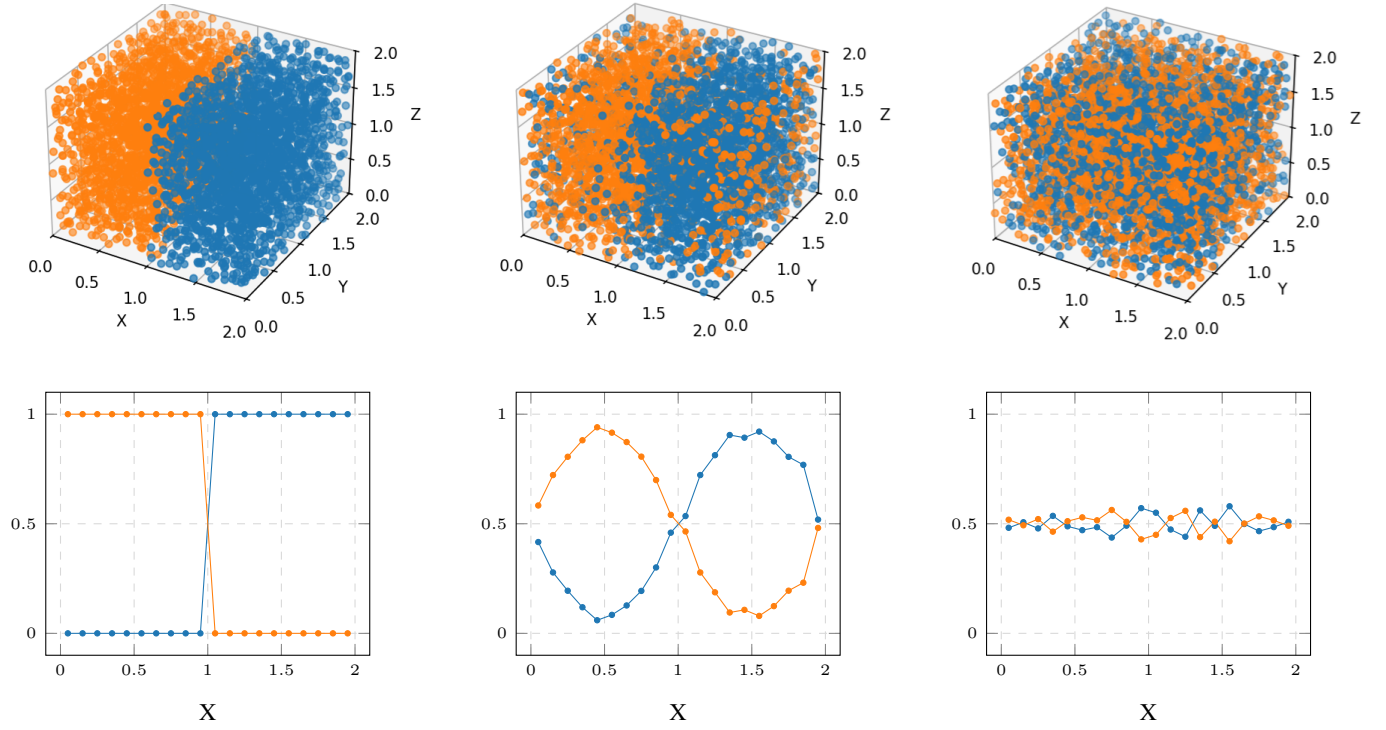}
\caption{Visualization of the neural network inference simulation of the \textsc{Self diffusion} case with a bigger domain $L=2$ at same density. The upper panels show several snapshots of the mixing of particles due to the thermal fluctuations. The lower panels represent the corresponding spatial concentration of red and blue particle profiles averaged over the X axis.}
\label{fig:self_diffusion_app}
\end{figure}

\begin{figure}[h]
\centering
\begin{tikzpicture}
\pgfplotsset{width=\textwidth, height=6cm}

\begin{axis}[
    xlabel=$t$,
    width=0.33\textwidth,
    ylabel=VACF,
    grid=major, 
    grid style={dashed,gray!30}, 
    tick label style={font=\tiny}, 
    font=\small,
    ylabel style={align=center, yshift=-1em}, 
    restrict y to domain*=0:4.5,
    legend style={font=\tiny}, 
]
\addplot [color=black, thick, dashed] table[x=t_vec, y=VACF_gt] {plots/ideal_gas/taylor_green.txt};
\addplot [color=tabblue, thick] table[x=t_vec, y=VACF_net] {plots/ideal_gas/taylor_green.txt};
\end{axis}

\begin{axis}[xshift=0.33\textwidth,
    width=0.33\textwidth,
    xlabel=$r$,
    ylabel=RDF,
    grid=major, 
    grid style={dashed,gray!30}, 
    tick label style={font=\tiny}, 
    font=\small,
    ylabel style={align=center, yshift=-1em}, 
    legend style={at={(0.95,0.05)},anchor=south east} 
]
\addplot [color=black, thick, dashed] table[x=r_RDF, y=RDF_gt] {plots/ideal_gas/taylor_green.txt};
\addplot [color=tabblue, thick] table[x=r_RDF, y=RDF_net] {plots/ideal_gas/taylor_green.txt};
\legend{GT, Ours}
\end{axis}

\begin{axis}[xshift=0.66\textwidth,
    width=0.33\textwidth,
    xlabel=$t$,
    ylabel=MSD,
    grid=major, 
    grid style={dashed,gray!30}, 
    tick label style={font=\tiny}, 
    font=\small,
    ylabel style={align=center, yshift=-1em}, 
    restrict y to domain*=0:0.65,
]
\addplot [color=black, thick, dashed] table[x=t_vec, y=MSD_gt] {plots/ideal_gas/taylor_green.txt};
\addplot [color=tabblue, thick] table[x=t_vec, y=MSD_net] {plots/ideal_gas/taylor_green.txt};
\end{axis}
        
\end{tikzpicture}
\caption{\commentquercus{Correlation metrics for the \textsc{Taylor-Green} dataset trained with equilibrium conditions.}}
\label{fig:ideal_gas_app}
\end{figure}
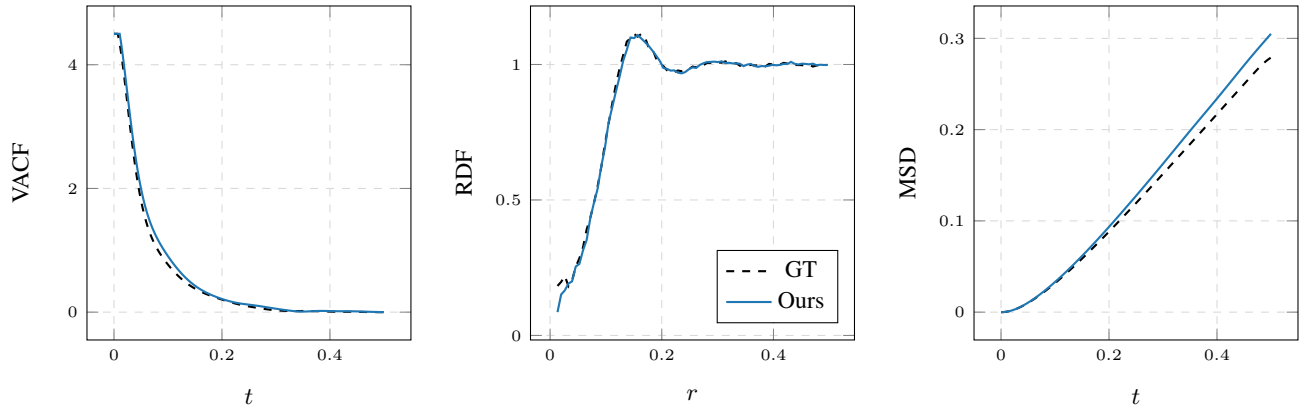

\begin{table}[h]
  \caption{Time performance (mean $\pm$ std) in ms/step ($\downarrow$) for the evaluated methods and the fully resolved dynamics (full-order).}
  \label{tab:time}
  \centering  
\begin{tabular}{lccc}
\toprule
 & \textsc{Ideal gas} & \textsc{Star polymer 11} & \textsc{Star polymer 51} \\
\midrule
GNS     & 1.32 $\pm$ 0.04  & 1.93 $\pm$ 0.03  & 2.51 $\pm$ 0.02 \\
GNS-SDE & 2.57 $\pm$ 0.02  & 3.75 $\pm$ 0.04 & 4.99 $\pm$ 0.04 \\
DPD     & \textbf{0.17 $\bs{\pm}$ 0.00}  & \textbf{0.18 $\bs{\pm}$ 0.00} & \textbf{0.17 $\bs{\pm}$ 0.00} \\
Ours    & 5.98 $\pm$ 0.05  & 6.12 $\pm$ 0.06 &  6.00 $\pm$ 0.08 \\
Full-order & - & 21.87 & 74.40 \\
\bottomrule
\end{tabular} 
\end{table}

\begin{figure}
\centering
\input{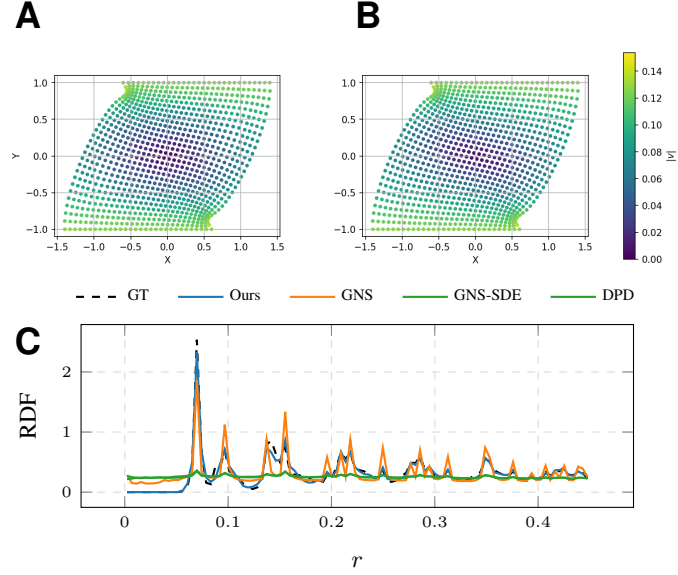}
\caption{\commentquercus{Results for the \textsc{Viscoelastic} dataset. Intermediate snapshot with color encoding velocity magnitude for (A) ground truth and (B) model prediction. (C) Radial Distribution Function, indicating the correct prediction of structural statistics.}}
\label{fig:results_viscoelastic_app}
\end{figure}

\begin{table}[h] 
  \caption{L2 relative error ($\downarrow$) of correlation metrics for the $\textsc{Viscoelastic}$ and $\textsc{Colloids}$ examples.}
  \label{tab:results_app}
  \centering  
\begin{tabular}{llcc}
\toprule
\multicolumn{2}{c}{} & RDF \\
\midrule
\textsc{Viscoelastic} & GNS    & 3.08e-1\\
                      & GNS-SDE  & 5.07e-1\\
                      & DPD   & 5.06e-1\\
                      & Ours no $U^{\text{dev}}$ & 1.84e-1 \\
                      & Ours    & \textbf{9.17e-2} \\       
\textsc{Colloids} & GNS      & 4.33e-1\\
                      & GNS-SDE  & 4.90e-1\\
                      & DPD   & 6.89e-1\\
                      & Ours no $U^{\text{dev}}$ & 4.49e-1 \\
                      & Ours    & \textbf{9.22e-2} \\ 
\bottomrule
\end{tabular} 

\end{table}

\begin{figure}
\centerline{\includegraphics[width=\linewidth]{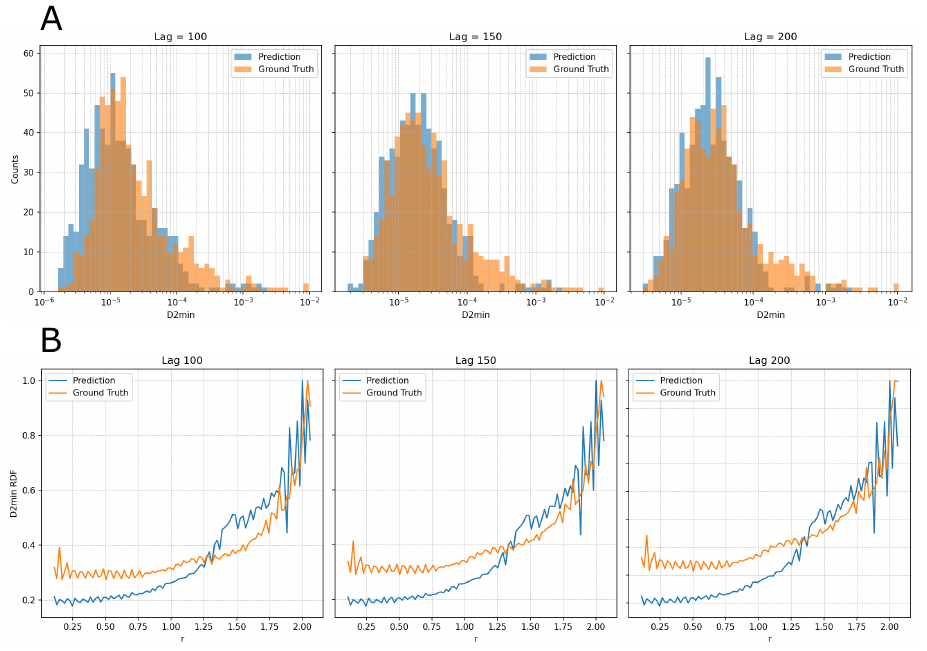}}
\caption{Dynamical properties for the \textsc{Colloids} dataset with lag times 100, 150 and 200. (A) Histograms of non-affine displacements $D^2_{\text{min}}$ for all the domain particles of the prediction and the reference solution. (B) Normalized radial distribution function of the local non-affine displacements.}
\label{fig:needle_app}
\end{figure}

\end{document}